\documentclass[12pt]{article}

\usepackage{times}
\usepackage{fullpage}
\usepackage{parskip}
\usepackage{authblk}

\usepackage{algorithm}
\usepackage{algorithmic}

\usepackage{amsmath, amsfonts, mathtools}
\usepackage{amsthm, amssymb}

\usepackage{natbib}

\usepackage{graphicx}
\usepackage{epstopdf} %i had to disable this because arxiv doesn't support it for some reason

\usepackage{enumitem}

\usepackage{setspace}

\newcommand{\fullname}{Kronecker-factored Approximate Curvature}
\newcommand{\acronym}{K-FAC}

\title{Optimizing Neural Networks with \fullname{}}

\author{James Martens\thanks{jmartens@cs.toronto.edu} }
\author{Roger Grosse\thanks{rgrosse@cs.toronto.edu}}
\affil{Department of Computer Science, University of Toronto}
\date{}

\newcommand{\nonlin}{\phi}

\newcommand{\expected}{\operatorname{E}}
\newcommand{\ovec}{\operatorname{vec}}
\DeclareMathOperator{\trace}{tr}

\newcommand{\cov}{\operatorname{Cov}}

\newcommand{\deriv}{\mathcal{D}}
\newcommand{\derivfrac}[2]{\frac{\mathrm{d}#1}{\mathrm{d}#2}}

\newcommand{\Real}{\mathbb{R}}
\DeclareMathOperator{\KL}{KL}
\newcommand{\diag}{\operatorname{diag}}

\newcommand{\Normal}{\mathcal{N}}

\newcommand{\bigO}{\mathcal{O}}

\newcommand{\B}{B}

\makeatletter
\def\thm@space@setup{%
  \thm@preskip=\parskip \thm@postskip=0pt
}
\makeatother

\newtheorem{theorem}{Theorem}

\newtheorem{lemma}[theorem]{Lemma}

\newtheorem{corollary}[theorem]{Corollary}

\usepackage{hyperref}
\usepackage{url}

\begin{document}

\maketitle

\begin{abstract} 
We propose an efficient method for approximating natural gradient descent in neural networks which we call \fullname{} (\acronym{}). \acronym{} is based on an efficiently invertible approximation of a neural network's Fisher information matrix which is neither diagonal nor low-rank, and in some cases is completely non-sparse.  It is derived by approximating various large blocks of the Fisher (corresponding to entire layers) as being the Kronecker product of two much smaller matrices.  While only several times more expensive to compute than the plain stochastic gradient, the updates produced by \acronym{} make \emph{much} more progress optimizing the objective, which results in an algorithm that can be much faster than stochastic gradient descent with momentum in practice.  And unlike some previously proposed approximate natural-gradient/Newton methods which use high-quality non-diagonal curvature matrices (such as Hessian-free optimization), \acronym{} works very well in highly stochastic optimization regimes.  This is because the cost of storing and inverting \acronym{}'s approximation to the curvature matrix does not depend on the amount of data used to estimate it, which is a feature typically associated only with diagonal or low-rank approximations to the curvature matrix.
\end{abstract}

\section{Introduction}

The problem of training neural networks is one of the most important and highly investigated ones in machine learning.  Despite work on layer-wise pretraining schemes, and various sophisticated optimization methods which try to approximate Newton-Raphson updates or natural gradient updates, stochastic gradient descent (SGD), possibly augmented with momentum, remains the method of choice for large-scale neural network training \citep{momentum_ilya}.

From the work on Hessian-free optimization (HF) \citep{HF} and related methods \citep[e.g.][]{KSD} we know that updates computed using local curvature information can make much more progress per iteration than the scaled gradient.  The reason that HF sees fewer practical applications than SGD are twofold.  Firstly, its updates are much more expensive to compute, as they involve running linear conjugate gradient (CG) for potentially hundreds of iterations, each of which requires a matrix-vector product with the curvature matrix (which are as expensive to compute as the stochastic gradient on the current mini-batch).  Secondly, HF's estimate of the curvature matrix must remain fixed while CG iterates, and thus the method is able to go through much less data than SGD can in a comparable amount of time, making it less well suited to stochastic optimizations.

As discussed in \citet{HF_chapter} and \citet{momentum_ilya}, CG has the potential to be much faster at local optimization than gradient descent, when applied to quadratic objective functions.  Thus, insofar as the objective can be locally approximated by a quadratic, each step of CG could potentially be doing a lot more work than each iteration of SGD, which would result in HF being much faster overall than SGD.  However, there are examples of quadratic functions \citep[e.g.][]{cg_hardcase}, characterized by curvature matrices with highly spread-out eigenvalue distributions, where CG will have no distinct advantage over well-tuned gradient descent with momentum.  Thus, insofar as the quadratic functions being optimized by CG within HF are of this character, HF shouldn't in principle be faster than well-tuned SGD with momentum.  The extent to which neural network objective functions give rise to such quadratics is unclear, although \citet{momentum_ilya} provides some preliminary evidence that they do. %, which suggests just applying gradient descent to the original objective. %And while HF employs a diagonal precondition which speeds up CG (and might redistribute eigenvalues more favorably), diagonal preconditioning methods are also available for gradient descent \citep[e.g.][]{diag_lecun, pesky}.

CG falls victim to this worst-case analysis because it is a first-order method.  This motivates us to consider methods which don't rely on first-order methods like CG as their primary engines of optimization.   One such class of methods which have been widely studied are those which work by directly inverting a diagonal, block-diagonal, or low-rank approximation to the curvature matrix \citep[e.g.][]{diag_lecun, pesky, ADADELTA, TONGA, Ollivier}.  In fact, a diagonal approximation of the Fisher information matrix is used within HF as a preconditioner for CG.  However, these methods provide only a limited performance improvement in practice, especially compared to SGD with momentum \citep[see for example][]{oLBFGS, ADADELTA}, and many practitioners tend to forgo them in favor of SGD or SGD with momentum.

We know that the curvature associated with neural network objective functions is highly non-diagonal, and that updates which properly respect and account for this non-diagonal curvature, such as those generated by HF, can make much more progress minimizing the objective than the plain gradient or updates computed from diagonal approximations of the curvature (usually $\sim\!\!10^2$ HF updates are required to adequately minimize most objectives, compared to the $\sim\!\!\!\!10^4-10^5$ required by methods that use diagonal approximations). %This follows from the simple fact that methods like HF, which use a full-rank non-sparse estimate of the curvature matrix, has a much higher rate of per-update progress than methods based on diagonal approximations.  
Thus, if we had an efficient and direct way to compute the inverse of a high-quality non-diagonal approximation to the curvature matrix (i.e. without relying on first-order methods like CG) this could potentially yield an optimization method whose updates would be large and powerful like HF's, while being (almost) as cheap to compute as the stochastic gradient.

In this work we develop such a method, which we call \fullname{} (\acronym{}).  We show that our method can be much faster in practice than even highly tuned implementations of SGD with momentum on certain standard neural network optimization benchmarks.  

The main ingredient in \acronym{} is a sophisticated approximation to the Fisher information matrix, which despite being neither diagonal nor low-rank, nor even block-diagonal with small blocks, can be inverted very efficiently, and can be estimated in an online fashion using arbitrarily large subsets of the training data (without increasing the cost of inversion).  

This approximation is built in two stages.  In the first, the rows and columns of the Fisher are divided into groups, each of which corresponds to \emph{all the weights in a given layer}, and this gives rise to a block-partitioning of the matrix (where the blocks are \emph{much} larger than those used by \citet{TONGA} or \citet{Ollivier}).  These blocks are then approximated as Kronecker products between much smaller matrices, which we show is equivalent to making certain approximating assumptions regarding the statistics of the network's gradients.  

In the second stage, this matrix is further approximated as having an \emph{inverse} which is either block-diagonal or block-tridiagonal.  We justify this approximation through a careful examination of the relationships between inverse covariances, tree-structured graphical models, and linear regression.  Notably, this justification doesn't apply to the Fisher itself, and our experiments confirm that while the inverse Fisher does indeed possess this structure (approximately), the Fisher itself does not.

%With our approximate Fisher inverse we can compute an approximate natural gradient.      Since even the exact natural gradient cannot be used directly as an update (since it is only a ``direction"), to m
% To turn this into a practical optimization algorithm requires the application of various ``damping" techniques, which have been previsely used in practical 2nd-order optimization methods like HF.

The rest of this paper is organized as follows.  \textbf{Section \ref{sec:background_and_notation}} gives basic background and notation for neural networks and the natural gradient.  \textbf{Section \ref{sec:kron_approx}} describes our initial Kronecker product approximation to the Fisher.  \textbf{Section \ref{sec:structured_precision_approx}} describes our further block-diagonal and block-tridiagonal approximations of the inverse Fisher, and how these can be used to derive an efficient inversion algorithm.  \textbf{Section \ref{sec:estimating_A_and_G}} describes how we compute online estimates of the quantities required by our inverse Fisher approximation  over a large ``window" of previously processed mini-batches (which makes \acronym{} very different from methods like HF or KSD, which base their estimates of the curvature on a single mini-batch).  \textbf{Section \ref{sec:damping}} describes how we use our approximate Fisher to obtain a practical and robust optimization algorithm which requires very little manual tuning, through the careful application of various theoretically well-founded ``damping" techniques that are standard in the optimization literature.  Note that damping techniques compensate both for the local quadratic approximation being implicitly made to the objective, and for our further approximation of the Fisher, and are non-optional for essentially any 2nd-order method like \acronym{} to work properly, as is well established by both theory and practice within the optimization literature \citep{nocedal_book}.  \textbf{Section \ref{sec:momentum}} describes a simple and effective way of adding a type of ``momentum" to \acronym{}, which we have found works very well in practice.  \textbf{Section \ref{sec:efficiency}} describes the computational costs associated with \acronym{}, and various ways to reduce them to the point where each update is at most only several times more expensive to compute than the stochastic gradient.  \textbf{Section \ref{sec:pseudocode}} gives complete high-level pseudocode for \acronym{}.  \textbf{Section \ref{sec:invariance}} characterizes a broad class of network transformations and reparameterizations to which \acronym{} is essentially invariant.  \textbf{Section \ref{sec:related}} considers some related prior methods for neural network optimization.  Proofs of formal results are located in the appendix.

%[[\textbf{insert more as they come.  }]]

%[[\textbf{Be sure to modify the above paragraph if any sections get cut in the ICML version}]]

\section{Background and notation}
\label{sec:background_and_notation}

\subsection{Neural Networks}
\label{sec:neural_networks}

In this section we will define the basic notation for feed-forward neural networks which we will use throughout this paper.  Note that this presentation closely follows the one from \citet{ng_martens}.

A neural network transforms its input $a_0 = x$ to an output $f(x,\theta) = a_\ell$ through a series of $\ell$ layers, each of which consists of a bank of units/neurons.  The units each receive as input a weighted sum of the outputs of units from the previous layer and compute their output via a nonlinear ``activation" function.   We denote by $s_i$ the vector of these weighted sums for the $i$-th layer, and by $a_i$ the vector of unit outputs (aka ``activities"). The precise computation performed at each layer $i \in \{1,\ldots,\ell\}$ is given as follows:
\begin{align*}
s_i &= W_i \bar{a}_{i-1} \\
a_i &= \nonlin_i(s_i)
\end{align*}
where $\nonlin_i$ is an element-wise nonlinear function, %(each coordinate of which is sometimes refered to as an activation function")
$W_i$ is a weight matrix, and $\bar{a}_i$ is defined as the vector formed by appending to $a_i$ an additional homogeneous coordinate with value 1.  Note that we do not include explicit bias parameters here as these are captured implicitly through our use of homogeneous coordinates.  In particular, the last column of each weight matrix $W_i$ corresponds to what is usually thought of as the ``bias vector".  Figure \ref{fig:nnet} illustrates our definition for $\ell = 2$.

We will define $\theta = [ \ovec(W_1)^\top \ovec(W_2)^\top \dots \ovec(W_\ell)^\top ]^\top$, which is the vector consisting of all of the network's parameters concatenated together, where $\ovec$ is the operator which vectorizes matrices by stacking their columns together.

We let $L(y,z)$ denote the loss function which measures the disagreement between a prediction $z$ made by the network, and a target $y$.  The training objective function $h(\theta)$ is the average (or expectation) of losses $L(y,f(x,\theta))$ with respect to a training distribution $\hat{Q}_{x,y}$ over input-target pairs $(x, y)$.  $h(\theta)$ is a proxy for the objective which we actually care about but don't have access to, which is the expectation of the loss taken with respect to the true data distribution $Q_{x,y}$.

%As is required in order to use the natural gradient formalism we 
We will assume that the loss is given by the negative log probability associated with a simple predictive distribution $R_{y|z}$ for $y$ parameterized by $z$, i.e. that we have
\begin{align*}
L(y,z) = -\log r( y | z )
\end{align*}
where $r$ is $R_{y|z}$'s density function.  This is the case for both the standard least-squares and cross-entropy objective functions, where the predictive distributions are multivariate normal and multinomial, respectively.

We will let $P_{y|x}(\theta) = R_{y|f(x,\theta)}$ denote the conditional distribution defined by the neural network, as parameterized by $\theta$, and $p(y|x,\theta) = r(y|f(x,\theta))$ its density function.  Note that minimizing the objective function $h(\theta)$ can be seen as maximum likelihood learning of the model $P_{y|x}(\theta)$.

For convenience we will define the following additional notation:
\begin{align*}
\deriv v = \derivfrac{L( y, f(x,\theta) ) }{v} = -\derivfrac{ \log p(y|x,\theta) }{v} \quad \quad \mbox{and} \quad \quad g_i = \deriv s_i
\end{align*}

Algorithm \ref{alg:nnet_gradient} shows how to compute the gradient $\deriv \theta$ of the loss function of a neural network using standard backpropagation.

\begin{algorithm}
\begin{algorithmic}
\STATE {\bf input:} $a_0 = x$; $\theta$ mapped to $(W_1,W_2,\ldots,W_\ell)$.
\STATE
\STATE /* Forward pass */
\FORALL{$i$ {\bf from} 1 {\bf to} $\ell$}
\STATE $s_i \gets W_i \bar{a}_{i-1}$\\
\STATE $a_i \gets \nonlin_i(s_i)$
\ENDFOR
\STATE
\STATE /* Loss derivative computation */
\STATE $\displaystyle \deriv a_\ell \gets \left. \frac{\partial L(y, z)}{\partial z} \right |_{z = a_\ell}$
\STATE
\STATE
\STATE /* Backwards pass */
\FORALL {$i$ {\bf from} $\ell$ {\bf downto} 1}
\STATE $g_i \gets \deriv a_i \odot \nonlin'_i(s_i)$
\STATE $\deriv W_i \gets g_i \bar{a}_{i-1}^\top$
\STATE $\deriv a_{i-1} \gets W_i^\top g_i$
\ENDFOR
\STATE
%\STATE {\bf output:} $\deriv \theta$ as mapped from $(\deriv W_1, \deriv W_2, \ldots, \deriv W_\ell)$.
\STATE {\bf output:} $\deriv \theta = [ \ovec(\deriv W_1)^\top \ovec(\deriv W_2)^\top \dots \ovec(\deriv W_\ell)^\top ]^\top$
\end{algorithmic}
\caption{An algorithm for computing the gradient of the loss $L(y,f(x,\theta))$ for a given $(x,y)$. Note that we are assuming here for simplicity that the $\nonlin_i$ are defined as coordinate-wise functions. \label{alg:nnet_gradient}}
\end{algorithm}

\begin{figure}[H]
\begin{centering}
\includegraphics[width=.4\columnwidth]{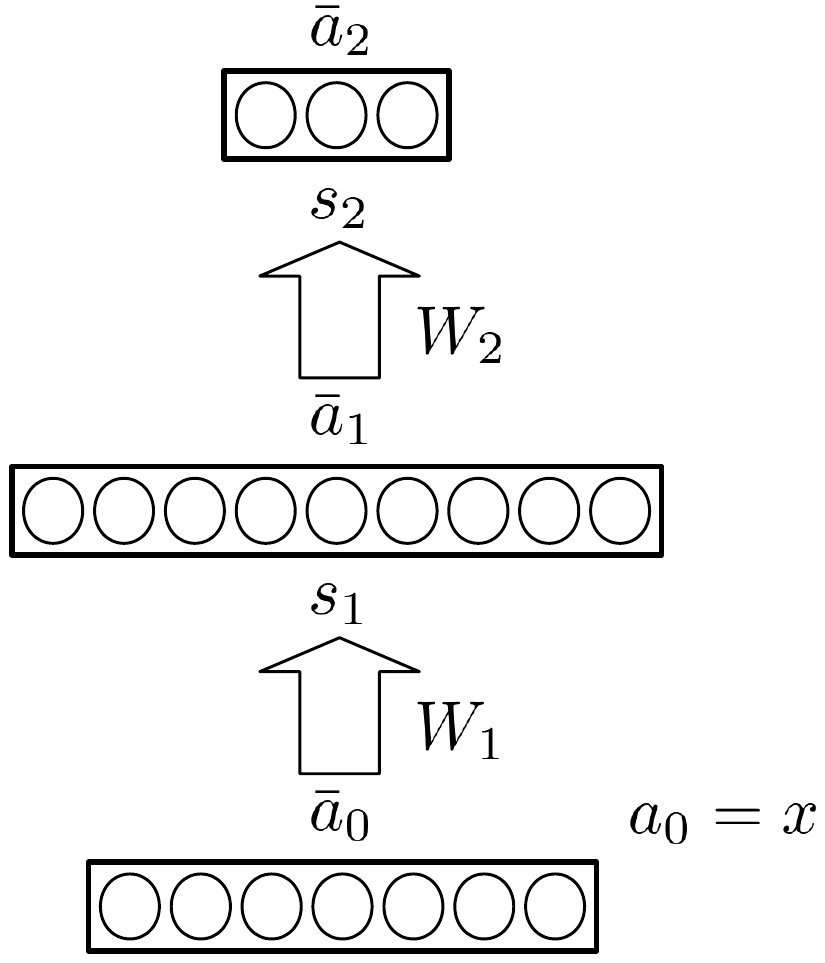}
\caption{\small A depiction of a standard feed-forward neural network for $\ell = 2$. \label{fig:nnet} }
\end{centering}
\end{figure}
%\newpage

\subsection{The Natural Gradient}
\label{sec:natural_gradient}

Because our network defines a conditional model $P_{y|x}(\theta)$, it has an associated Fisher information matrix (which we will simply call ``the Fisher") which is given by
\begin{align*}
F = \expected\left [ \derivfrac{ \log p(y|x,\theta) }{\theta} \derivfrac{ \log p(y|x,\theta) }{\theta}^\top \right] = \expected[ \deriv \theta \deriv \theta^\top ]
\end{align*}
Here, the expectation is taken with respect to the data distribution $Q_x$ over inputs $x$, and the model's predictive distribution $P_{y|x}(\theta)$ over $y$.  Since we usually don't have access to $Q_x$, and the above expectation would likely be intractable even if we did, we will instead compute $F$ using the training distribution $\hat{Q}_x$ over inputs $x$.

The well-known natural gradient \citep{natural_efficient} is defined as $F^{-1} \nabla h(\theta)$.  Motivated from the perspective of information geometry \citep{information_geom}, the natural gradient defines the direction in parameter space which gives the largest change in the objective per unit of change in the model, as measured by the KL-divergence.  This is to be contrasted with the standard gradient, which can be defined as the direction in parameter space which gives the largest change in the objective per unit of change in the parameters, as measured by the standard Euclidean metric.  %This interpretation of the natural gradient follows from the fact that the Fisher information matrix corresponds to the quadratic term in the Taylor series expansion of the KL-divergence between the $P_{y|x}$ at the current parameter setting and a (proposed) new one.   

The natural gradient also has links to several classical ideas from optimization.  It can be shown \citep{ng_martens,Razvan} that the Fisher is equivalent to the Generalized Gauss-Newton matrix (GGN) \citep{schraudolph,HF_chapter} in certain important cases, which is a well-known positive semi-definite approximation to the Hessian of the objective function.  In particular, \citep{ng_martens} showed that when the GGN is defined so that the network is linearized up to the loss function, and the loss function corresponds to the negative log probability of observations under an exponential family model $R_{y|z}$ with $z$ representing the \emph{natural parameters}, then the Fisher corresponds exactly to the GGN.\footnote{Note that the condition that $z$ represents the natural parameters might require one to formally include the nonlinear transformation usually performed by the final nonlinearity $\nonlin_{\ell}$ of the network (such as the logistic-sigmoid transform before a cross-entropy error) as part of the loss function $L$ instead.  Equivalently, one could linearize the network only up to the input $s_\ell$ to $\nonlin_{\ell}$ when computing the GGN (see \citet{HF_chapter}).}

The GGN has served as the curvature matrix of choice in HF and related methods, and so in light of its equivalence to the Fisher, these 2nd-order methods can be seen as approximate natural gradient methods.  And perhaps more importantly from a practical perspective, natural gradient-based optimization methods can conversely be viewed as 2nd-order optimization methods, which as pointed out by \citet{ng_martens}), brings to bare the vast wisdom that has accumulated about how to make such methods work well in both theory and practice \citep[e.g][]{nocedal_book}.  In Section \ref{sec:damping} we productively make use of these connections in order to design a robust and highly effective optimization method using our approximation to the natural gradient/Fisher (which is developed in Sections \ref{sec:kron_approx} and \ref{sec:structured_precision_approx}).

For some good recent discussion and analysis of the natural gradient, see \citet{IGO,ng_martens,Razvan}.

\section{A block-wise Kronecker-factored Fisher approximation}
\label{sec:kron_approx}

The main computational challenge associated with using the natural gradient is computing $F^{-1}$ (or its product with $\nabla h$).  For large networks, with potentially millions of parameters, computing this inverse naively is computationally impractical.  In this section we develop an initial approximation of $F$ which will be a key ingredient in deriving our efficiently computable approximation to $F^{-1}$ and the natural gradient.

Note that $\deriv \theta = [ \ovec(\deriv W_1)^\top \: \ovec(\deriv W_2)^\top \: \cdots \: \ovec(\deriv W_\ell)^\top ]^\top$ and so $F$ can be expressed as
\begin{align*}
%F &= \expected \left[ \deriv\theta \deriv\theta^\top \right] = \expected\left[ [ \ovec(\deriv W_1)^\top \: \ovec(\deriv W_2)^\top \: \cdots \: \ovec(\deriv W_\ell)^\top ]^\top [ \ovec(\deriv W_1)^\top \: \ovec(\deriv W_2)^\top \: \cdots \: \ovec(\deriv W_\ell)^\top ]  \right ] \\
F &= \expected \left[ \deriv\theta \deriv\theta^\top \right] \\
&= \expected\left[ [ \ovec(\deriv W_1)^\top \: \ovec(\deriv W_2)^\top \: \cdots \: \ovec(\deriv W_\ell)^\top ]^\top [ \ovec(\deriv W_1)^\top \: \ovec(\deriv W_2)^\top \: \cdots \: \ovec(\deriv W_\ell)^\top ]  \right ] \\
%F &= \expected \left[ \deriv\theta \deriv\theta^\top \right] = \expected\left[ [ \ovec(\deriv W_1)^\top \: \cdots \: \ovec(\deriv W_\ell)^\top ]^\top [ \ovec(\deriv W_1)^\top \: \cdots \: \ovec(\deriv W_\ell)^\top ]  \right ] \\
%&= \expected\left [ 
%\begin{bmatrix}
% \ovec(\deriv W_1) \ovec(\deriv W_1)^\top   &  \ovec(\deriv W_1) \ovec(\deriv W_2)^\top     &  \cdots     &   \ovec(\deriv W_1) \ovec(\deriv W_\ell)^\top  \\
% \ovec(\deriv W_1) \ovec(\deriv W_2)^\top      &  \ovec(\deriv W_2) \ovec(\deriv W_2)^\top     & \cdots   &   \ovec(\deriv W_2) \ovec(\deriv W_\ell)^\top  \\
%\vdots      &  \vdots     & \ddots & \vdots  \\
% \ovec(\deriv W_\ell) \ovec(\deriv W_1)^\top    &    \ovec(\deriv W_\ell) \ovec(\deriv W_2)^\top     & \cdots &  \ovec(\deriv W_\ell) \ovec(\deriv W_\ell)^\top 
%\end{bmatrix}
%\right ] \\
&=  
\begin{bmatrix}
 \expected \left[ \ovec(\deriv W_1) \ovec(\deriv W_1)^\top \right ]  &  \expected \left[ \ovec(\deriv W_1) \ovec(\deriv W_2)^\top \right ]   &  \cdots     &   \expected \left[ \ovec(\deriv W_1) \ovec(\deriv W_\ell)^\top \right ] \\
\expected \left[ \ovec(\deriv W_2) \ovec(\deriv W_1)^\top \right ]     &  \expected \left[ \ovec(\deriv W_2) \ovec(\deriv W_2)^\top \right ]    & \cdots   &   \expected \left[ \ovec(\deriv W_2) \ovec(\deriv W_\ell)^\top \right ] \\
\vdots      &  \vdots     & \ddots & \vdots  \\
\expected \left[ \ovec(\deriv W_\ell) \ovec(\deriv W_1)^\top \right ]   &   \expected \left[ \ovec(\deriv W_\ell) \ovec(\deriv W_2)^\top \right ]    & \cdots &  \expected \left[ \ovec(\deriv W_\ell) \ovec(\deriv W_\ell)^\top \right ]
\end{bmatrix}
\end{align*}

Thus, we see that $F$ can be viewed as an $\ell$ by $\ell$ block matrix, with the $(i,j)$-th block $F_{i,j}$ given by $F_{i,j} = \expected \left[ \ovec(\deriv W_i) \ovec(\deriv W_j)^\top \right ]$.

Noting that $\deriv W_i = g_i \bar{a}_{i-1}^\top$ and that $\ovec( u v^\top ) = v \otimes u$ we have $\ovec(\deriv W_i) = \ovec(g_i \bar{a}_{i-1}^\top) = \bar{a}_{i-1} \otimes g_i$, and thus we can rewrite $F_{i,j}$ as
\begin{align*}
F_{i,j} = \expected \left[ \ovec(\deriv W_i) \ovec(\deriv W_j)^\top \right ] = \expected \left[(\bar{a}_{i-1} \otimes g_i)(\bar{a}_{j-1} \otimes g_j)^\top \right ] &= \expected \left[(\bar{a}_{i-1} \otimes g_i)(\bar{a}_{j-1}^\top \otimes g_j^\top) \right ] \\
&= \expected \left[ \bar{a}_{i-1} \bar{a}_{j-1}^\top \otimes g_i g_j^\top \right ]
\end{align*}
where $A \otimes B$ denotes the Kronecker product between $A \in \Real^{m \times n}$ and $B$, and is given by
\begin{align*}
A\otimes B \equiv \begin{bmatrix} [A]_{1,1} B & \cdots & [A]_{1,n}B \\ \vdots & \ddots & \vdots \\ [A]_{m,1} B & \cdots & [A]_{m,n} B \end{bmatrix}
\end{align*}
Note that the Kronecker product satisfies many convenient properties that we will make use of in this paper, especially the identity $(A\otimes B)^{-1} = A^{-1}\otimes B^{-1}$.  See \citet{kronprod} for a good discussion of the Kronecker product.

Our initial approximation $\tilde{F}$ to $F$ will be defined by the following block-wise approximation:
\begin{align}
\label{eqn:kron_approx}
F_{i,j} = \expected \left[ \bar{a}_{i-1} \bar{a}_{j-1}^\top \otimes g_i g_j^\top \right] \approx \expected \left[ \bar{a}_{i-1} \bar{a}_{j-1}^\top \right] \otimes \expected \left[ g_i g_j^\top \right] = \bar{A}_{i-1,j-1} \otimes G_{i,j} = \tilde{F}_{i,j}
\end{align}
where $\bar{A}_{i,j} = \expected \left[ \bar{a}_i \bar{a}_j^\top \right]$ and $G_{i,j} = \expected \left[ g_i g_j^\top \right]$.

This gives
\begin{align*}
\tilde{F} = 
\begin{bmatrix}
\bar{A}_{0,0} \otimes G_{1,1}  &  \bar{A}_{0,1} \otimes G_{1,2} &  \cdots     &    \bar{A}_{0,\ell-1} \otimes G_{1,\ell} \\
\bar{A}_{1,0} \otimes G_{2,1}  &  \bar{A}_{1,1} \otimes G_{2,2} &  \cdots     &    \bar{A}_{1,\ell-1} \otimes G_{2,\ell} \\
\vdots      &  \vdots     & \ddots & \vdots  \\
\bar{A}_{\ell-1,0} \otimes G_{\ell,1}   &   \bar{A}_{\ell-1,1} \otimes G_{\ell,2}    & \cdots &  \bar{A}_{\ell-1,\ell-1} \otimes G_{\ell,\ell}
\end{bmatrix}
\end{align*}
which has the form of what is known as a Khatri-Rao product in multivariate statistics.

\begin{figure}
\begin{centering}
\includegraphics[width=.3\columnwidth]{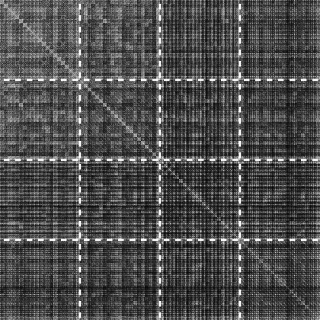}
\hspace{0.025\columnwidth}
\includegraphics[width=.3\columnwidth]{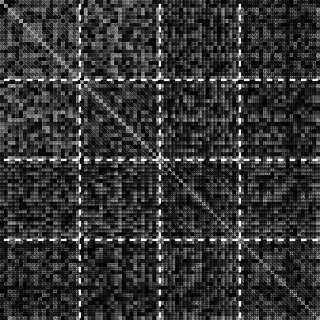}
\hspace{0.025\columnwidth}
\includegraphics[width=.3\columnwidth]{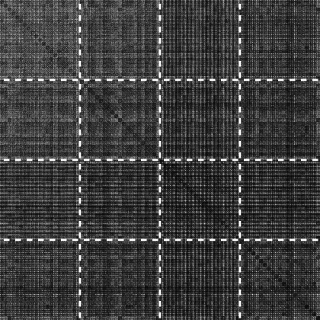}
\caption{\small A comparison of the exact Fisher $F$ and our block-wise Kronecker-factored approximation $\tilde{F}$, for the middle 4 layers of a standard deep neural network partially trained to classify a 16x16 down-scaled version of MNIST.  The network was trained with 7 iterations of \acronym{} in batch mode, achieving 5\% error (the error reached 0\% after 22 iterations) .  The network architecture is 256-20-20-20-20-20-10 and uses standard tanh units. On the \textbf{left} is the exact Fisher $F$, in the \textbf{middle} is our approximation $\tilde{F}$, and on the \textbf{right} is the difference of these.  The dashed lines delineate the blocks.  Note that for the purposes of visibility we plot the absolute values of the entries, with the white level corresponding linearly to the size of these values (up to some maximum, which is the same in each image). \label{fig:kron_approx} }
\end{centering}
\end{figure}

The expectation of a Kronecker product is, in general, not equal to the Kronecker product of expectations, and so this is indeed a major approximation to make, and one which likely won't become exact under any realistic set of assumptions, or as a limiting case in some kind of asymptotic analysis.  Nevertheless, it seems to be fairly accurate in practice, and is able to successfully capture the ``coarse structure" of the Fisher, as demonstrated in Figure \ref{fig:kron_approx} for an example network.

As we will see in later sections, this approximation leads to significant computational savings in terms of storage and inversion, which we will be able to leverage in order to design an efficient algorithm for computing an approximation to the natural gradient.  %In particular, storing blocks of $F$ only requires storing the two consistuent factors, as there is usually no need to actually evaluate the kronecker products.  Moreover, we will be able to efficiently invert these blocks using the well-known identity $(A \otimes B)^{-1} = A^{-1} \otimes B^{-1}$.

%In the following subsection we will provide some justification for this approximation by giving additional interpretations of what it is actually doing.

\subsection{Interpretations of this approximation}
\label{sec:justifying_approx}

Consider an arbitrary pair of weights $[W_i]_{k_1,k_2}$ and $[W_j]_{k_3,k_4}$ from the network, where $[\cdot]_{i,j}$ denotes the value of the $(i,j)$-th entry.  We have that the corresponding derivatives of these weights are given by $\deriv [W_i]_{k_1,k_2} = \bar{a}^{(1)} g^{(1)}$ and $\deriv [W_j]_{k_3,k_4} = \bar{a}^{(2)} g^{(2)}$, where we denote for convenience $\bar{a}^{(1)} = [\bar{a}_{i-1}]_{k_1}$, $\bar{a}^{(2)} = [\bar{a}_{j-1}]_{k_3}$, $g^{(1)} = [g_i]_{k_2}$, and $g^{(2)} = [g_j]_{k_4}$.

The approximation given by eqn.~\ref{eqn:kron_approx} is equivalent to making the following approximation for each pair of weights:
\begin{align}
\label{eqn:scalar_approx}
\expected \left[\deriv[W_i]_{k_1,k_2} \deriv[W_j]_{k_3,k_4} \right] = \expected \left[ (\bar{a}^{(1)} g^{(1)})  (\bar{a}^{(2)} g^{(2)}) \right] = \expected \left[ \bar{a}^{(1)} \bar{a}^{(2)} \: g^{(1)} g^{(2)} \right] \approx \expected \left[ \bar{a}^{(1)} \bar{a}^{(2)} \right] \expected \left[ g^{(1)} g^{(2)} \right]
\end{align}
%In other words, we are approximating the 2nd-order statistics (i.e. expectation of the product) between derivatives of these weights as the product between two quantities: the $2$nd order statistics of the derivatives $g^{(1)}$ and $g^{(2)}$ of the inputs (denoted by $s^{(1)}$ and $s^{(2)}$) to their corresponding outgoing units , and the $2$nd order statistics of the activations $c^{(1)}$ of $c^{(2)}$ of their corresponding incoming units.
And thus one way to interpret the approximation in eqn.~\ref{eqn:kron_approx} is that we are assuming statistical independence between products $\bar{a}^{(1)} \bar{a}^{(2)}$ of unit activities and products $g^{(1)} g^{(2)}$ of unit input derivatives.

Another more detailed interpretation of the approximation emerges by considering the following expression for the approximation error $\expected \left[ \bar{a}^{(1)} \bar{a}^{(2)} \: g^{(1)} g^{(2)} \right] - \expected \left[ \bar{a}^{(1)} \bar{a}^{(2)} \right] \expected \left[ g^{(1)} g^{(2)} \right]$ (which is derived in the appendix):
\begin{align}
\label{eqn:kron_error_expression}
%\expected \left[ \bar{a}^{(1)} \bar{a}^{(2)} \: g^{(1)} g^{(2)} \right] - \expected \left[ \bar{a}^{(1)} \bar{a}^{(2)} \right] \expected \left[ g^{(1)} g^{(2)} \right] = 
\kappa(\bar{a}^{(1)}, \bar{a}^{(2)}, g^{(1)}, g^{(2)}) + \expected[ \bar{a}^{(1)} ] \kappa(\bar{a}^{(2)}, g^{(1)}, g^{(2)}) + \expected[\bar{a}^{(2)}] \kappa(\bar{a}^{(1)}, g^{(1)}, g^{(2)})
\end{align}

Here $\kappa(\cdot)$ denotes the cumulant of its arguments.  Cumulants are a natural generalization of the concept of mean and variance to higher orders, and indeed 1st-order cumulants are means and 2nd-order cumulants are covariances.  Intuitively, cumulants of order $k$ measure the degree to which the interaction between variables is intrinsically of order $k$, as opposed to arising from many lower-order interactions.

A basic upper bound for the approximation error is
\begin{align}
\label{eqn:kron_error_bound}
|\kappa(\bar{a}^{(1)}, \bar{a}^{(2)}, g^{(1)}, g^{(2)})| + |\expected[ \bar{a}^{(1)} ]| |\kappa(\bar{a}^{(2)}, g^{(1)}, g^{(2)})| + |\expected[\bar{a}^{(2)}]| |\kappa(\bar{a}^{(1)}, g^{(1)}, g^{(2)})|
\end{align}
which will be small if all of the higher-order cumulants are small (i.e. those of order 3 or higher).  Note that in principle this upper bound may be loose due to possible cancellations between the terms in eqn.~\ref{eqn:kron_error_expression}.

Because higher-order cumulants are zero for variables jointly distributed according to a multivariate Gaussian, it follows that this upper bound on the approximation error will be small insofar as the joint distribution over $\bar{a}^{(1)}$, $\bar{a}^{(2)}$, $g^{(1)}$, and $g^{(2)}$ is well approximated by a multivariate Gaussian.  And while we are not aware of an argument for why this should be the case in practice, it does seem to be the case that for the example network from Figure \ref{fig:kron_approx}, the size of the error is well predicted by the size of the higher-order cumulants.  In particular, the total approximation error, summed over all pairs of weights in the middle 4 layers, is $2894.4$, and is of roughly the same size as the corresponding upper bound ($4134.6$), whose size is tied to that of the higher order cumulants (due to the impossibility of cancellations in eqn.~\ref{eqn:kron_error_bound}).

\section{Additional approximations to $\tilde{F}$ and inverse computations}
\label{sec:structured_precision_approx}

To the best of our knowledge there is no efficient general method for inverting a Khatri-Rao product like $\tilde{F}$.  Thus, we must make further approximations if we hope to obtain an efficiently computable approximation of the inverse Fisher.

In the following subsections we argue that the inverse of $\tilde{F}$ can be reasonably approximated as having one of two special structures, either of which make it efficiently computable.  The second of these will be slightly less restrictive than the first (and hence a better approximation) at the cost of some additional complexity.  We will then show how matrix-vector products with these approximate inverses can be efficiently computed, which will thus give an efficient algorithm for computing an approximation to the natural gradient.

\subsection{Structured inverses and the connection to linear regression}

%[[insert good lead-in / background discussion here ???]]

%There are well known relationships between the structure of inverse covariance matrix (aka the ``precision" or ``concentration" matrix) and the 

Suppose we are given a multivariate distribution whose associated covariance matrix is $\Sigma$.

Define the matrix $B$ so that for $i \neq j$, $[B]_{i,j}$ is the coefficient on the $j$-th variable in the optimal linear predictor of the $i$-th variable from all the other variables, and for $i = j$, $[B]_{i,j} = 0$.  Then define the matrix $D$ to be the diagonal matrix where $[D]_{i,i}$ is the variance of the error associated with such a predictor of the $i$-th variable.

\citet{pourahmadi_new} showed that $B$ and $D$ can be obtained from the inverse covariance $\Sigma^{-1}$ by the formulas
\begin{align*}
[B]_{i,j} = -\frac{[\Sigma^{-1}]_{i,j}}{[\Sigma^{-1}]_{i,i}} \quad \mbox{and} \quad [D]_{i,i} = \frac{1}{[\Sigma^{-1}]_{i,i}}
\end{align*}
from which it follows that the inverse covariance matrix can be expressed as
\begin{align*}
\Sigma^{-1} = D^{-1} (I - B)
\end{align*}

Intuitively, this result says that each row of the inverse covariance $\Sigma^{-1}$ is given by the coefficients of the optimal linear predictor of the $i$-th variable from the others, up to a scaling factor.  So if the $j$-th variable is much less ``useful" than the other variables for predicting the $i$-th variable, we can expect that the $(i,j)$-th entry of the inverse covariance will be relatively small.  

Note that ``usefulness" is a subtle property as we have informally defined it.  In particular, it is not equivalent to the degree of correlation between the $j$-th and $i$-th variables, or any such simple measure.  As a simple example, consider the case where the $j$-th variable is equal to the $k$-th variable plus independent Gaussian noise.  Since any linear predictor can achieve a lower variance simply by shifting weight from the $j$-th variable to the $k$-th variable, we have that the $j$-th variable is not useful (and its coefficient will thus be zero) in the task of predicting the $i$-th variable for any setting of $i$ other than $i = j$ or $i = k$.

Noting that the Fisher $F$ is a covariance matrix over $\deriv \theta$ w.r.t.~the model's distribution (because $\expected[ \deriv \theta ] = 0$ by Lemma \ref{lemma:udv_expectation}), we can thus apply the above analysis to the distribution over $\deriv \theta$ to gain insight into the approximate structure of $F^{-1}$, and by extension its approximation $\tilde{F}^{-1}$.

Consider the derivative $\deriv W_i$ of the loss with respect to the weights $W_i$ of layer $i$.   Intuitively, if we are trying to predict one of the entries of $\deriv W_i$ from the other entries of $\deriv \theta$, those entries also in $\deriv W_i$ will likely be the most useful in this regard.  Thus, it stands to reason that the largest entries of $\tilde{F}^{-1}$ will be those on the diagonal blocks, so that $\tilde{F}^{-1}$ will be well approximated as block-diagonal, with each block corresponding to a different $\deriv W_i$.

Beyond the other entries of $\deriv W_i$, it is the entries of $\deriv W_{i+1}$ and $\deriv W_{i-1}$ (i.e. those associated with adjacent layers) that will arguably be the most useful in predicting a given entry of $\deriv W_i$.   This is because the true process for computing the loss gradient only uses information from the layer below (during the forward pass) and from the layer above (during the backwards pass).  Thus, approximating $\tilde{F}^{-1}$ as block-tridiagonal seems like a reasonable and milder alternative than taking it to be block-diagonal.  Indeed, this approximation would be exact if the distribution over $\deriv \theta$ were given by a directed graphical model which generated each of the $\deriv W_i$'s, one layer at a time, from either $\deriv W_{i+1}$ or $\deriv W_{i-1}$. Or equivalently, if $\deriv W_i$ were distributed according to an undirected Gaussian graphical model with binary potentials only between entries in the same or adjacent layers.  Both of these models are depicted in Figure \ref{fig:graphical_model}.  

Now while in reality the $\deriv W_i$'s are generated using information from adjacent layers according to a process that is \emph{neither linear nor Gaussian}, it nonetheless stands to reason that their joint statistics might be reasonably approximated by such a model.  In fact, the idea of approximating the distribution over loss gradients with a directed graphical model forms the basis of the recent FANG method of \citet{FANG}.

\begin{figure}
\begin{centering}
%\includegraphics[width=.45\columnwidth]{covfig2.eps}
%\hspace{0.2in}
%\includegraphics[width=.45\columnwidth]{blockfig_new1.eps}
%\includegraphics[width=.3\columnwidth]{covfig3b_kron.eps}
%\includegraphics[width=.3\columnwidth]{covfig3b_kron_compressed.png}
%\includegraphics[width=.3\columnwidth]{newfig2_kronF_L1e-6.png}
\includegraphics[width=.3\columnwidth]{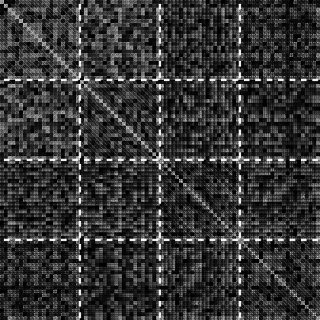}
\hspace{0.025\columnwidth}
\includegraphics[width=.3\columnwidth]{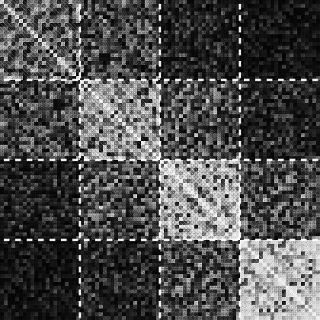}
\hspace{0.025\columnwidth}
\includegraphics[width=.3\columnwidth]{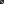}
\caption{\small A comparison of our block-wise Kronecker-factored approximation $\tilde{F}$, and its inverse, using the example neural network from Figure \ref{fig:kron_approx}.   On the \textbf{left} is $\tilde{F}$, in the \textbf{middle} is its exact inverse, and on the \textbf{right} is a 4x4 matrix containing the averages of the absolute values of the entries in each block of the inverse.  As predicted by our theory, the inverse exhibits an approximate block-tridiagonal structure, whereas $\tilde{F}$ itself does not.  Note that the corresponding plots for the exact $F$ and its inverse look similar.  The very small blocks visible on the diagonal of the inverse each correspond to the weights on the outgoing connections of a particular unit. The inverse was computed subject to the factored Tikhonov damping technique described in Sections \ref{sec:factored_tik} and \ref{sec:gamma}, using the same value of $\gamma$ that was used by \acronym{} at the iteration from which this example was taken (see Figure \ref{fig:kron_approx}).  Note that for the purposes of visibility we plot the absolute values of the entries, with the white level corresponding linearly to the size of these values (up to some maximum, which is chosen differently for the Fisher approximation and its inverse, due to the highly differing scales of these matrices).
\label{fig:inverse_approx}}
\end{centering}
\end{figure}

Figure \ref{fig:inverse_approx} examines the extent to which the inverse Fisher is well approximated as block-diagonal or block-tridiagonal for an example network.

In the following two subsections we show how both the block-diagonal and block-tridiagonal approximations to $\tilde{F}^{-1}$ give rise to computationally efficient methods for computing matrix-vector products with it.  And at the end of Section \ref{sec:structured_precision_approx} we present two figures (Figures \ref{fig:inverse_approx2} and \ref{fig:inverse_approx3}) which examine the quality of these approximations for an example network.

\subsection{Approximating $\tilde{F}^{-1}$ as block-diagonal}
\label{sec:blockdiag_approx}

Approximating $\tilde{F}^{-1}$ as block-diagonal is equivalent to approximating $\tilde{F}$ as block-diagonal.  A natural choice for such an approximation $\breve{F}$ of $\tilde{F}$, is to take the block-diagonal of $\breve{F}$ to be that of $\tilde{F}$. This gives the matrix
\begin{align*}
\breve{F} = \diag\left( \tilde{F}_{1,1}, \tilde{F}_{2,2}, \: \ldots, \:, \tilde{F}_{\ell,\ell} \right) = \diag\left( \bar{A}_{0,0} \otimes G_{1,1}, \: \bar{A}_{1,1} \otimes G_{2,2}, \: \ldots, \: \bar{A}_{\ell-1,\ell-1} \otimes G_{\ell,\ell} \right)
\end{align*}

Using the identity $(A \otimes B)^{-1} = A^{-1} \otimes B^{-1}$ we can easily compute the inverse of $\breve{F}$ as
\begin{align*}
\breve{F}^{-1} = \diag\left( \bar{A}_{0,0}^{-1} \otimes G_{1,1}^{-1}, \: \bar{A}_{1,1}^{-1} \otimes G_{2,2}^{-1}, \: \ldots, \: \bar{A}_{\ell-1,\ell-1}^{-1} \otimes G_{\ell,\ell}^{-1} \right)
\end{align*}
Thus, computing $\breve{F}^{-1}$ amounts to computing the inverses of $2\ell$ smaller matrices.

Then to compute $u = \breve{F}^{-1}v$, we can make use of the well-known identity $(A \otimes B) \ovec(X) = \ovec( B X A^\top )$ to get
\begin{align*}
U_i = G_{i,i}^{-1} V_i \bar{A}_{i-1,i-1}^{-1}
\end{align*}
where $v$ maps to $(V_1, V_2,\ldots, V_\ell)$ and $u$ maps to $(U_1, U_2, \ldots, U_\ell)$ in an analogous way to how $\theta$ maps to $(W_1, W_2, \ldots, W_\ell)$.

Note that block-diagonal approximations to the Fisher information have been proposed before in TONGA \citep{TONGA}, where each block corresponds to the weights associated with a particular unit.  In our block-diagonal approximation, the blocks correspond to all the parameters in a given layer, and are thus \emph{much} larger.  In fact, they are so large that they would be impractical to invert as general matrices.

\subsection{Approximating $\tilde{F}^{-1}$ as block-tridiagonal}
\label{sec:blocktridiag_approx}

Note that unlike in the above block-diagonal case, approximating $\tilde{F}^{-1}$ as block-tridiagonal is \emph{not} equivalent to approximating $\tilde{F}$ as block-tridiagonal.  Thus we require a more sophisticated approach to deal with such an approximation.  We develop such an approach in this subsection.

To start, we will define $\hat{F}$ to be the matrix which agrees with $\tilde{F}$ on the tridiagonal blocks, and which satisfies the property that $\hat{F}^{-1}$ is block-tridiagonal.  Note that this definition implies certain values for the off-tridiagonal blocks of $\hat{F}$ which will differ from those of $\tilde{F}$ insofar as $\tilde{F}^{-1}$ is not actually block-tridiagonal.

To establish that such a matrix $\hat{F}$ is well defined and can be inverted efficiently, we first observe that assuming that $\hat{F}^{-1}$ is block-tridiagonal is equivalent to assuming that it is the precision matrix of an undirected Gaussian graphical model (UGGM) over $\deriv\theta$ (as depicted in Figure \ref{fig:graphical_model}), whose density function is proportional to $\exp(-\deriv\theta^\top \hat{F}^{-1} \deriv\theta)$.  As this graphical model has a tree structure, there is an equivalent \emph{directed} graphical model with the same distribution and the same (undirected) graphical structure \citep[e.g.][]{bishop}, where the directionality of the edges is given by a directed acyclic graph (DAG).  Moreover, this equivalent directed model will also be linear/Gaussian, and hence a directed Gaussian Graphical model (DGGM).

Next we will show how the parameters of such a DGGM corresponding to $\hat{F}$ can be efficiently recovered from the tridiagonal blocks of $\hat{F}$, so that $\hat{F}$ is uniquely determined by these blocks (and hence well-defined).  We will assume here that the direction of the edges is from the higher layers to the lower ones.  Note that a different choice for these directions would yield a superficially different algorithm for computing the inverse of $\hat{F}$ that would nonetheless yield the same output.

For each $i$, we will denote the conditional covariance matrix of $\ovec(\deriv W_i)$ on $\ovec(\deriv W_{i+1})$ by $\Sigma_{i|i+1}$ and the linear coefficients from $\ovec(\deriv W_{i+1})$ to $\ovec(\deriv W_i)$ by the matrix $\Psi_{i,i+1}$, so that the conditional distributions defining the model are %of $\ovec(\deriv W_i)$ given $\deriv W_{i+1}$ is
\begin{align*}
\ovec(\deriv W_i) \sim \Normal\left( \Psi_{i,i+1}\!\ovec(\deriv W_{i+1}), \: \Sigma_{i|i+1} \right) \quad \quad \mbox{and} \quad \quad \ovec(\deriv W_\ell) \sim \Normal\left( \vec{0}, \: \Sigma_\ell \right)
\end{align*}

\begin{figure}
\begin{centering}
\includegraphics[width=.85\columnwidth]{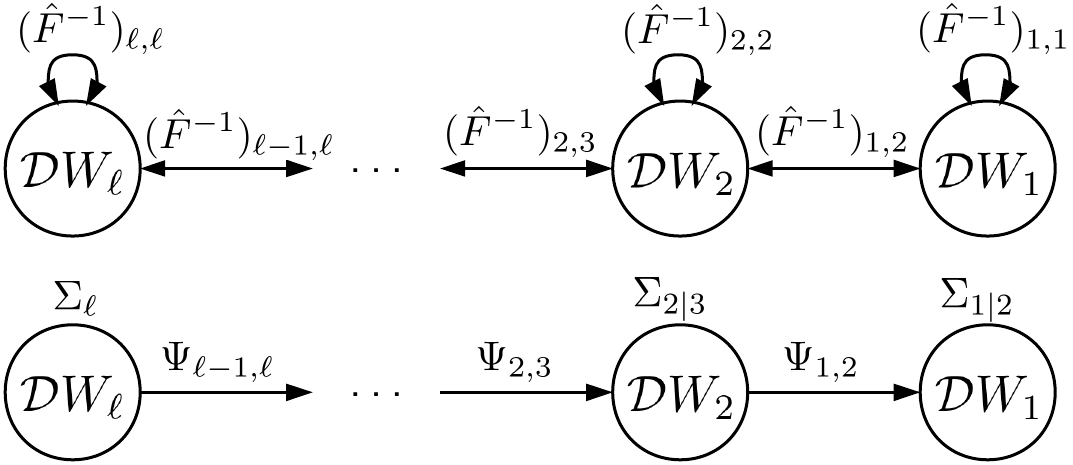}
\caption{\small A diagram depicting the UGGM corresponding to $\hat{F}^{-1}$ and its equivalent DGGM.  The UGGM's edges are labeled with the corresponding weights of the model (these are distinct from the network's weights).  Here, $(\hat{F}^{-1})_{i,j}$ denotes the $(i,j)$-th block of $\hat{F}^{-1}$. The DGGM's edges are labeled with the matrices that specify the linear mapping from the source node to the conditional mean of the destination node (whose conditional covariance is given by its label). \label{fig:graphical_model}}
\end{centering}
\end{figure}

%See Figure \ref{fig:graphical_model} for a depiction of the UGGM corresponding to $\hat{F}^{-1}$ and its equivalent DGGM.

Since $\Sigma_{\ell}$ is just the covariance of $\ovec(\deriv W_\ell)$, it is given simply by $\hat{F}_{\ell,\ell} = \tilde{F}_{\ell,\ell}$.  And for $i \leq \ell-1$, we can see that $\Psi_{i,i+1}$ is given by
\begin{align*}
\Psi_{i,i+1} = \hat{F}_{i,i+1} {\hat{F}_{i+1,i+1}}^{-1} = \tilde{F}_{i,i+1} {\tilde{F}_{i+1,i+1}}^{-1} = \left( \bar{A}_{i-1,i} \otimes G_{i,i+1} \right) \left( \bar{A}_{i,i} \otimes G_{i+1,i+1} \right)^{-1} %&= \bar{A}_{i-1,i} \bar{A}_{i,i}^{-1} \otimes G_{i,i+1} G_{i+1,i+1}^{-1} \\
&= \Psi^{\bar{A}}_{i-1,i} \otimes \Psi^G_{i,i+1}
\end{align*}
where
\begin{align*}
\Psi^{\bar{A}}_{i-1,i} = \bar{A}_{i-1,i} \bar{A}_{i,i}^{-1}  \quad \quad \mbox{and} \quad \quad  \Psi^G_{i,i+1} = G_{i,i+1} G_{i+1,i+1}^{-1}
\end{align*}
The conditional covariance $\Sigma_{i|i+1}$ is thus given by
\begin{align*}
\Sigma_{i|i+1} = \hat{F}_{i,i} - \Psi_{i,i+1} \hat{F}_{i+1,i+1} \Psi_{i,i+1}^\top &= \tilde{F}_{i,i} - \Psi_{i,i+1} \tilde{F}_{i+1,i+1} \Psi_{i,i+1}^\top \\
&= \bar{A}_{i-1,i-1} \otimes G_{i,i} - \Psi^{\bar{A}}_{i-1,i} \bar{A}_{i,i} \Psi^{\bar{A}\top}_{i-1,i} \otimes \Psi^G_{i,i+1} G_{i+1,i+1}  \Psi^{G\top}_{i,i+1}
\end{align*}

Following the work of \citet{FANG}, we use the block generalization of well-known ``Cholesky" decomposition of the precision matrix of DGGMs \citep{pourahmadi_old}, which gives
\begin{align*}
\hat{F}^{-1} = \Xi^\top \Lambda \Xi
\end{align*}
where,
\begin{align*}
\Lambda = \diag \left ( \Sigma_{1|2}^{-1}, \Sigma_{2|3}^{-1}, \: \ldots, \: \Sigma_{\ell-1|\ell}^{-1}, \Sigma_\ell^{-1} \right ) \quad \quad  \mbox{and} \quad \quad 
\Xi = 
\begin{bmatrix}
I  & -\Psi_{1,2}  &   &    \\
   &  I & -\Psi_{2,3}  &  \\
 &   &  I &  \ddots \\
 &   &    & \ddots & -\Psi_{\ell-1,\ell} \\
&   &   &     &  I 
\end{bmatrix}
\end{align*}

Thus, matrix-vector multiplication with $\hat{F}^{-1}$ amounts to performing matrix-vector multiplication by $\Xi$, followed by $\Lambda$, and then by $\Xi^\top$.

As in the block-diagonal case considered in the previous subsection, matrix-vector products with $\Xi$ (and $\Xi^\top$) can be efficiently computed using the well-known identity $(A \otimes B)^{-1} = A^{-1} \otimes B^{-1}$.  In particular, $u = \Xi^\top v$ can be computed as
\begin{align*}
U_i = V_i - \Psi^{G\top}_{i-1,i} V_{i-1} \Psi^{\bar{A}}_{i-2,i-1} \quad \quad \mbox{ and } \quad \quad U_1 = V_1
\end{align*}
and similarly $u = \Xi v$ can be computed as
\begin{align*}
U_i = V_i - \Psi^G_{i,i+1} V_{i+1} \Psi^{\bar{A}\top}_{i-1,i} \quad \quad \mbox{ and } \quad \quad U_\ell = V_\ell
\end{align*}
where the $U_i$'s and $V_i$'s are defined in terms of $u$ and $v$ as in the previous subsection.

Multiplying a vector $v$ by $\Lambda$ amounts to multiplying each $\ovec(V_i)$ by the corresponding $\Sigma_{i|i+1}^{-1}$.  This is slightly tricky because $\Sigma_{i|i+1}$ is the difference of Kronecker products, so we cannot use the straightforward identity $(A \otimes B)^{-1} = A^{-1} \otimes B^{-1}$.  Fortunately, there are efficient techniques for inverting such matrices which we discuss in detail in Appendix \ref{sec:stein_solve}.

%CUT:
%Note that the equation $u = \Lambda v$ is equivalent to $\Lambda^{-1} u = v$, which is in turn equivalent to $\Sigma_i \ovec(U_i) = \ovec(V_i)$.  Using the identity $(A \otimes B) \ovec(X) = \ovec( B X A^\top )$, and our formula for $\Sigma_i$, this can be expressed in matrix form as
%\begin{align*}
%G_{i,i} U_i Z_{i,i}  -  \Psi^G_{i,i+1} G_{i+1,i+1}  {\Psi^G_{i,i+1}}^\top U_i \Psi^Z_{i,i+1} \bar{A}_{i,i} {\Psi^\bar{A}_{i-1,i}}^\top = V_i
%\end{align*}
%This has the form $A U_i B - C U_i D = E$, which is known as a generalized Stein equation, and has been well studied in the control theory literature [[e.g. cite]].  Fortunately, there are various efficient methods available for solving such equations, some of which we will discuss in Section [[ref]] (including the one we end up using in our experiments).

\subsection{Examining the approximation quality}
\label{sec:extra_quality_figs}

Figures \ref{fig:inverse_approx2} and \ref{fig:inverse_approx3} examine the quality of the approximations $\breve{F}$ and $\hat{F}$ of $\tilde{F}$, which are derived by approximating $\tilde{F}^{-1}$ as block-diagonal and block-tridiagonal (resp.), for an example network.  

From Figure \ref{fig:inverse_approx2}, which compares $\breve{F}$ and $\hat{F}$ directly to $\tilde{F}$, we can see that while $\breve{F}$ and $\hat{F}$ exactly capture the diagonal and tridiagonal blocks (resp.) of $\tilde{F}$, as they must by definition, $\hat{F}$ ends up approximating the off-tridiagonal blocks of $\tilde{F}$ very well too.  This is likely owed to the fact that the approximating assumption used to derive $\hat{F}$, that $\tilde{F}^{-1}$ is block-tridiagonal, is a very reasonable one in practice (judging by Figure \ref{fig:inverse_approx}).

Figure \ref{fig:inverse_approx3}, which compares $\breve{F}^{-1}$ and $\hat{F}^{-1}$ to $\tilde{F}^{-1}$, paints an arguably more interesting and relevant picture, as the quality of the approximation of the natural gradient will be roughly proportional\footnote{The error in any approximation $F_0^{-1}\nabla h$ of the natural gradient $F^{-1}\nabla h$ will be roughly proportional to the error in the approximation $F_0^{-1}$ of the associated \emph{inverse} Fisher $F^{-1}$, since $\|F^{-1}\nabla h - F_0^{-1}\nabla h\| \leq \|\nabla h\| \|F^{-1} - F_0^{-1}\|$.} to the quality of approximation of the \emph{inverse} Fisher.  We can see from this figure that due to the approximate block-diagonal structure of $\tilde{F}^{-1}$, $\breve{F}^{-1}$ is actually a reasonably good approximation of $\tilde{F}^{-1}$, despite $\breve{F}$ being a rather poor approximation of $\tilde{F}$ (based on Figure \ref{fig:inverse_approx2}).  Meanwhile, we can see that by accounting for the tri-diagonal blocks, $\hat{F}^{-1}$ is indeed a significantly better approximation of $\tilde{F}^{-1}$ than $\breve{F}^{-1}$ is, even on the \emph{diagonal} blocks.

\begin{figure}%[H]
\begin{centering}
\includegraphics[width=.3\columnwidth]{newfig3_kronF_defdamp.png}
\hspace{0.025\columnwidth}
\includegraphics[width=.3\columnwidth]{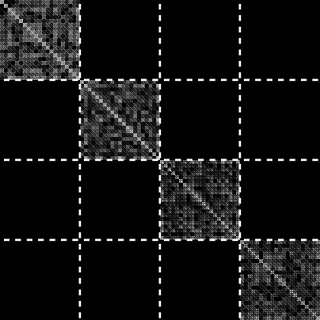}
\hspace{0.025\columnwidth}
\includegraphics[width=.3\columnwidth]{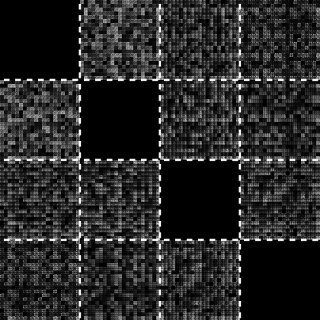}

\vspace{0.025\columnwidth}

\includegraphics[width=.3\columnwidth]{newfig3_kronF_defdamp.png}
\hspace{0.025\columnwidth}
\includegraphics[width=.3\columnwidth]{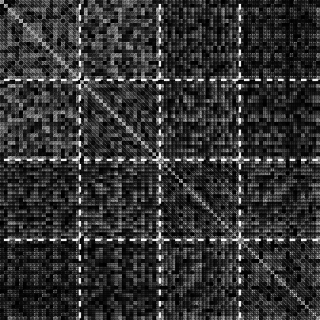}
\hspace{0.025\columnwidth}
\includegraphics[width=.3\columnwidth]{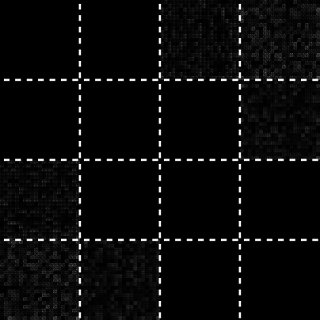}

\caption{\small A comparison of our block-wise Kronecker-factored approximation $\tilde{F}$, and its approximations $\breve{F}$ and $\hat{F}$ (which are based on approximating the inverse $\tilde{F}^{-1}$ as either block-diagonal or block-tridiagonal, respectively), using the example neural network from Figure \ref{fig:kron_approx}. On the \textbf{left} is $\tilde{F}$, in the \textbf{middle} its approximation, and on the \textbf{right} is the absolute difference of these.  The \textbf{top row} compares to $\breve{F}$ and the \textbf{bottom row} compares to $\hat{F}$.  While the diagonal blocks of the top right matrix, and the tridiagonal blocks of the bottom right matrix are exactly zero due to how $\breve{F}$ and $\hat{F}$ (resp.) are constructed, the off-tridiagonal blocks of the bottom right matrix, while being very close to zero, are actually non-zero (which is hard to see from the plot). Note that for the purposes of visibility we plot the absolute values of the entries, with the white level corresponding linearly to the size of these values (up to some maximum, which is the same in each image). \label{fig:inverse_approx2}}
\end{centering}
\end{figure}

\begin{figure}[H]
\begin{centering}
\includegraphics[width=.3\columnwidth]{newfig3_invkronF_defdamp.png}
\hspace{0.025\columnwidth}
\includegraphics[width=.3\columnwidth]{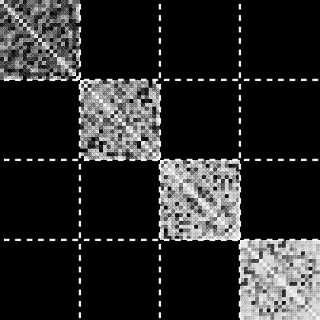}
\hspace{0.025\columnwidth}
\includegraphics[width=.3\columnwidth]{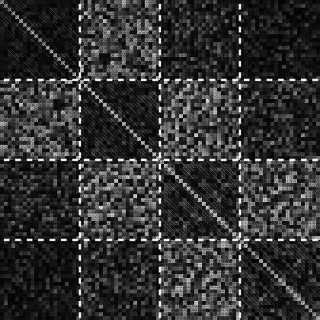}

\vspace{0.025\columnwidth}

\includegraphics[width=.3\columnwidth]{newfig3_invkronF_defdamp.png}
\hspace{0.025\columnwidth}
\includegraphics[width=.3\columnwidth]{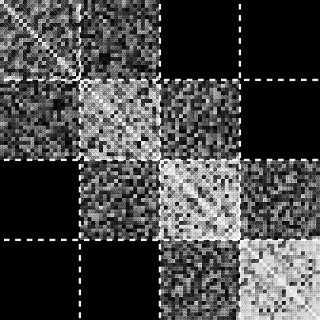}
\hspace{0.025\columnwidth}
\includegraphics[width=.3\columnwidth]{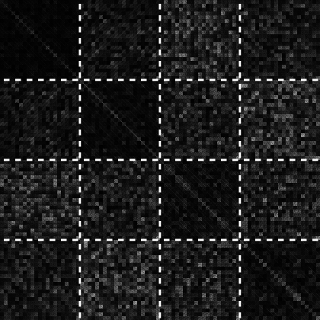}

\caption{\small A comparison of the exact inverse $\tilde{F}^{-1}$ of our block-wise Kronecker-factored approximation $\tilde{F}$, and its block-diagonal and block-tridiagonal approximations $\breve{F}^{-1}$ and $\hat{F}^{-1}$ (resp.), using the example neural network from Figure \ref{fig:kron_approx}. On the \textbf{left} is $\tilde{F}^{-1}$, in the \textbf{middle} its approximation, and on the \textbf{right} is the absolute difference of these.  The \textbf{top row} compares to $\breve{F}^{-1}$ and the \textbf{bottom row} compares to $\hat{F}^{-1}$.  The inverse was computed subject to the factored Tikhonov damping technique described in Sections \ref{sec:factored_tik} and \ref{sec:gamma}, using the same value of $\gamma$ that was used by \acronym{} at the iteration from which this example was taken (see Figure \ref{fig:kron_approx}).  Note that for the purposes of visibility we plot the absolute values of the entries, with the white level corresponding linearly to the size of these values (up to some maximum, which is the same in each image). \label{fig:inverse_approx3}}
\end{centering}
\end{figure}

\section{Estimating the required statistics}
\label{sec:estimating_A_and_G}

Recall that $\bar{A}_{i,j} = \expected \left[ \bar{a}_i \bar{a}_j^\top \right]$ and $G_{i,j} = \expected \left[ g_i g_j^\top \right]$. Both approximate Fisher inverses discussed in Section \ref{sec:structured_precision_approx} require some subset of these.  In particular, the block-diagonal approximation requires them for $i = j$, while the block-tridiagonal approximation requires them for $j \in \{i,i+1\}$ (noting that $\bar{A}_{i,j}^\top = \bar{A}_{j,i}$ and $G_{i,j}^\top = G_{j,i}$).

Since the $\bar{a}_i$'s don't depend on $y$, we can take the expectation $\expected \left[ \bar{a}_i \bar{a}_j^\top \right]$ with respect to just the training distribution $\hat{Q}_x$ over the inputs $x$.  On the other hand, the $g_i$'s do depend on $y$, and so the expectation\footnote{It is important to note this expectation should \emph{not} be taken with respect to the training/data distribution over $y$ (i.e. $\hat{Q}_{y|x}$ or $Q_{y|x}$).  Using the training/data distribution for $y$ would perhaps give an approximation to a quantity known as the ``empirical Fisher information matrix", which lacks the previously discussed equivalence to the Generalized Gauss-Newton matrix, and would not be compatible with the theoretical analysis performed in Section \ref{sec:justifying_approx} (in particular, Lemma \ref{lemma:udv_expectation} would break down).  Moreover, such a choice would not give rise to what is usually thought of as the natural gradient, and based on the findings of \citet{HF}, would likely perform worse in practice as part of an optimization algorithm.  See \citet{ng_martens} for a more detailed discussion of the empirical Fisher and reasons why it may be a poor choice for a curvature matrix compared to the standard Fisher.} $\expected \left[ g_i g_j^\top \right]$ must be taken with respect to \emph{both} $\hat{Q}_x$ and the network's predictive distribution $P_{y|x}$.

While computing matrix-vector products with the $G_{i,j}$ could be done exactly and efficiently for a given input $x$ (or small mini-batch of $x$'s) by adapting the methods of \citet{schraudolph}, there doesn't seem to be a sufficiently efficient method for computing the entire matrix itself.  Indeed, the hardness results of \citet{curvprop} suggest that this would require, for each example $x$ in the mini-batch, work that is asymptotically equivalent to matrix-matrix multiplication involving matrices the same size as $G_{i,j}$.  While a small constant number of such multiplications is arguably an acceptable cost (see Section \ref{sec:efficiency}), a number which grows with the size of the mini-batch would not be.

Instead, we will approximate the expectation over $y$ by a standard Monte-Carlo estimate obtained by sampling $y$'s from the network's predictive distribution and then rerunning the backwards phase of backpropagation (see Algorithm \ref{alg:nnet_gradient}) as if these were the training targets.  

Note that computing/estimating the required $\bar{A}_{i,j}$/$G_{i,j}$'s involves computing averages over outer products of various $\bar{a}_i$'s from network's usual forward pass, and $g_i$'s from the modified backwards pass (with targets sampled as above).  Thus we can compute/estimate these quantities on the same input data used to compute the gradient $\nabla h$, at the cost of one or more additional backwards passes, and a few additional outer-product averages.  Fortunately, this turns out to be quite inexpensive, as we have found that just one modified backwards pass is sufficient to obtain a good quality estimate in practice, and the required outer-product averages are similar to those already used to compute the gradient in the usual backpropagation algorithm.  %Implementing these computations usually just amounts to copying the backwards pass once, and inserting a few lines of code into the forward pass and the modified backwards pass to 

In the case of online/stochastic optimization we have found that the best strategy is to maintain running estimates of the required $\bar{A}_{i,j}$'s and $G_{i,j}$'s using a simple exponentially decaying averaging scheme.  In particular, we take the new running estimate to be the old one weighted by $\epsilon$, plus the estimate on the new mini-batch weighted by $1-\epsilon$, for some $0 \leq \epsilon < 1$.  In our experiments we used $\epsilon = \min\{ 1-1/k, 0.95\}$, where $k$ is the iteration number.

Note that the more naive averaging scheme where the estimates from each iteration are given equal weight would be inappropriate here.  This is because the $\bar{A}_{i,j}$'s and $G_{i,j}$'s depend on the network's parameters $\theta$, and these will slowly change over time as optimization proceeds, so that estimates computed many iterations ago will become stale.% and thus will need to be ``down-weighted".

This kind of exponentially decaying averaging scheme is commonly used in methods involving diagonal or block-diagonal approximations (with much smaller blocks than ours) to the curvature matrix \citep[e.g.][]{lecun_tricks,ng_adaptive, pesky}.  Such schemes have the desirable property that they allow the curvature estimate to depend on much more data than can be reasonably processed in a single mini-batch.  

Notably, for methods like HF which deal with the exact Fisher indirectly via matrix-vector products, such a scheme would be impossible to implement efficiently, as the exact Fisher matrix (or GGN) seemingly cannot be summarized using a compact data structure whose size is independent of the amount of data used to estimate it.  Indeed, it seems that the only representation of the exact Fisher which would be independent of the amount of data used to estimate it would be an explicit $n \times n$ matrix (which is far too big to be practical).  Because of this, HF and related methods must base their curvature estimates only on subsets of data that can be reasonably processed all at once, which limits their effectiveness in the stochastic optimization regime. %, which is a severe drawback to using such approaches in the stochastic optimization setting.

\section{Update damping}
\label{sec:damping}

\subsection{Background and motivation}

%The idealized natural gradient approach is to follow the smooth geodesic path in the Riemannian manifold (implied by the Fisher information matrix viewed as a metric tensor) from the current predictive distribution to the target one, by taking a series of infinitesimally small steps in the direction of the natural gradient (which gets recomputed at each point).  While this is clearly impractical as a real optimization method, one can take larger steps and still follow these geodesics approximately.  But in our experience, to obtain an update which satisfies the minimal requirement of not worsening the objective function value, one must make the step size so small that the resulting optimization algorithm is not practical.

The idealized natural gradient approach is to follow the smooth path in the Riemannian manifold (implied by the Fisher information matrix viewed as a metric tensor) that is generated by taking a series of infinitesimally small steps (in the original parameter space) in the direction of the natural gradient (which gets recomputed at each point). While this is clearly impractical as a real optimization method, one can take larger steps and still follow these paths approximately.  But in our experience, to obtain an update which satisfies the minimal requirement of not worsening the objective function value, it is often the case that one must make the step size so small that the resulting optimization algorithm performs poorly in practice.

The reason that the natural gradient can only be reliably followed a short distance is that it is defined merely as an optimal \emph{direction} (which trades off improvement in the objective versus change in the predictive distribution), and not a discrete \emph{update}.  

Fortunately, as observed by \citet{ng_martens}, the natural gradient can be understood using a more traditional optimization-theoretic perspective which implies how it can be used to generate updates that will be useful over larger distances.  In particular, when $R_{y|z}$ is an exponential family model with $z$ as its \emph{natural} parameters (as it will be in our experiments), \citet{ng_martens} showed that the Fisher becomes equivalent to the Generalized Gauss-Newton matrix (GGN), which is a positive semi-definite approximation of the Hessian of $h$.  Additionally, there is the well-known fact that when $L(x,f(x,\theta))$ is the negative log-likelihood function associated with a given $(x,y)$ pair (as we are assuming in this work), the Hessian $H$ of $h$ and the Fisher $F$ are closely related in the sense $H$ is the expected Hessian of $L$ under the \emph{training} distribution $\hat{Q}_{x,y}$, while $F$ is the expected Hessian of $L$ under the \emph{model's} distribution $P_{x,y}$ (defined by the density $p(x,y) = p(y|x)q(x)$).

From these observations it follows that 
\begin{align}
\label{eqn:quad_model}
M(\delta) = \frac{1}{2} \delta^\top F \delta + \nabla h(\theta)^\top \delta + h(\theta)
\end{align}
can be viewed as a convex approximation of the 2nd-order Taylor series of expansion of $h(\delta + \theta)$, whose minimizer $\delta^*$ is the (negative) natural gradient $-F^{-1} \nabla h(\theta)$.  Note that if we add an $\ell_2$ or ``weight-decay" regularization term to $h$ of the form $\displaystyle \frac{\eta}{2} \| \theta \|_2^2$, then similarly $F + \eta I$ can be viewed as an approximation of the Hessian of $h$, and replacing $F$ with $F + \eta I$ in $M(\delta)$ yields an approximation of the 2nd-order Taylor series, whose minimizer is a kind of ``regularized" (negative) natural gradient $-(F + \eta I)^{-1} \nabla h(\theta)$, which is what we end up using in practice.

From the interpretation of the natural gradient as the minimizer of $M(\delta)$, we can see that it fails to be useful as a local update only insofar as $M(\delta)$ fails to be a good local approximation to $h(\delta + \theta)$.  And so as argued by \citet{ng_martens}, it is natural to make use of the various ``damping" techniques that have been developed in the optimization literature for dealing with the breakdowns in local quadratic approximations that inevitably occur during optimization.  Notably, this breakdown usually won't occur in the final ``local convergence" stage of optimization where the function becomes well approximated as a convex quadratic within a sufficiently large neighborhood of the local optimum.  This is the phase traditionally analyzed in most theoretical results, and while it is important that an optimizer be able to converge well in this final phase, it is arguably much more important from a practical standpoint that it behaves sensibly before this phase.  

This initial ``exploration phase" \citep{darken_2phase} is where damping techniques help in ways that are not apparent from the asymptotic convergence theorems alone, which is not to say there are not strong mathematical arguments that support their use \citep[see][]{nocedal_book}.  In particular, in the exploration phase it will often still be true that $h(\theta+\delta)$ is accurately approximated by a convex quadratic \emph{locally within some region} around $\delta = 0$, and that therefor optimization can be most efficiently performed by minimizing a sequence of such convex quadratic approximations within adaptively sized local regions.

Note that well designed damping techniques, such as the ones we will employ, automatically adapt to the local properties of the function, and effectively ``turn themselves off" when the quadratic model becomes a sufficiently accurate local approximation of $h$, allowing the optimizer to achieve the desired asymptotic convergence behavior \citep{levenberg_marquardt}.

Successful and theoretically well-founded damping techniques include Tikhonov damping (aka Tikhonov regularization, which is closely connected to the trust-region method) with Levenberg-Marquardt style adaptation \citep{levenberg_marquardt}, line-searches, and trust regions, truncation, etc., all of which tend to be much more effective in practice than merely applying a learning rate to the update, or adding a \emph{fixed} multiple of the identity to the curvature matrix.  Indeed, a subset of these techniques was exploited in the work of \citet{HF}, and primitive versions of them have appeared implicitly in older works such as \citet{diag_lecun}, and also in many recent diagonal methods like that of \citet{ADADELTA}, although often without a good understanding of what they are doing and why they help. 

Crucially, more powerful 2nd-order optimizers like HF and \acronym{}, which have the capability of taking \emph{much larger steps} than 1st-order methods (or methods which use diagonal curvature matrices), \emph{require} more sophisticated damping solutions to work well, and will usually \emph{completely fail} without them, which is consistent with predictions made in various theoretical analyses \citep[e.g.][]{nocedal_book}.  As an analogy one can think of such powerful 2nd-order optimizers as extremely fast racing cars that need more sophisticated control systems than standard cars to prevent them from flying off the road.  Arguably one of the reasons why high-powered 2nd-order optimization methods have historically tended to under-perform in machine learning applications, and in neural network training in particular, is that their designers did not understand or take seriously the issue of quadratic model approximation quality, and did not employ the more sophisticated and effective damping techniques that are available to deal with this issue. 

For a detailed review and discussion of various damping techniques and their crucial role in practical 2nd-order optimization methods, we refer the reader to \citet{HF_chapter}.

%In this section we give a concrete proposal for how to use these damping techniques to give a practical optimization method which we have found works very well in practice.  

%Natural gradient methods, like any optimization methods that involve the use of quadratic approximations (see discussion in Section \ref{sec:natural_gradient} [[unless it got moved]]), must be used with some kind of ``update damping" or ``trust-region" scheme in order to work well.  Without such a schemes, these methods tend to trust their local quadratic approximations too much, and will often produce inappropriate large updates and unusable updates.

%Damping schemes can range from simple scaling/line-searching of the update, to so-called Tikhonov regularization methods that involve adding to the curvature matrix a multiple of the identity, to more exotic ones such as premature termination of conjugate gradient iterations (in the context of HF, see [[cite HF chapter]]), or ones that use the particular structure of the model being optimized [[cite HF RNN paper]].  

%Note that because the local properties of the objective function change as optimization proceeds, any success scheme must be adaptive in order to work well in practice.

\subsection{A highly effective damping scheme for \acronym{}}

Methods like HF which use the exact Fisher seem to work reasonably well with an adaptive Tikhonov regularization technique where $\lambda I$ is added to $F + \eta I$, and where $\lambda$ is adapted according to Levenberg-Marquardt style adjustment rule. %, as their quadratic models typically become more trustworthy as optimization progresses, thus allowing the adjustment heuristics to shrink the Tikhonov terms.  
This common and well-studied method can be shown to be equivalent to imposing an adaptive spherical region (known as a ``trust region") which constrains the optimization of the quadratic model \citep[e.g][]{nocedal_book}.  However, we found that this simple technique is insufficient when used with our approximate natural gradient update proposals.  In particular, we have found that there never seems to be a ``good" choice for $\lambda$ that gives rise to updates which are of a quality comparable to those produced by methods that use the exact Fisher, such as HF.  

One possible explanation for this finding is that, unlike quadratic models based on the exact Fisher (or equivalently, the GGN), the one underlying \acronym{} has no guarantee of being accurate up to 2nd-order.  Thus, $\lambda$ must remain large in order to compensate for this intrinsic 2nd-order inaccuracy of the model, which has the side effect of ``washing out" the small eigenvalues (which represent important low-curvature directions).

%By contrast, we found that this failed to happen with our method [[name??]], 
Fortunately, through trial and error, we were able to find a relatively simple and highly effective damping scheme, which combines several different techniques, and which works well within \acronym{}.  Our scheme works by computing an initial update proposal using a version of the above described adaptive Tikhonov damping/regularization method, and then re-scaling this according to quadratic model computed using the exact Fisher.  This second step is made practical by the fact that it only requires a single matrix-vector product with the exact Fisher, and this can be computed efficiently using standard methods.  We discuss the details of this scheme in the following subsections.

\subsection{A factored Tikhonov regularization technique}
\label{sec:factored_tik}

In the first stage of our damping scheme we generate a candidate update proposal $\Delta$ by applying a slightly modified form of Tikhononv damping to our approximate Fisher, before multiplying $-\nabla h$ by its inverse.

In the usual Tikhonov regularization/damping technique, one adds $(\lambda + \eta) I$ to the curvature matrix (where $\eta$ accounts for the $\ell_2$ regularization), which is equivalent to adding a term of the form $\displaystyle \frac{\lambda + \eta}{2} \| \delta \|_2^2$ to the corresponding quadratic model (given by $M(\delta)$ with $F$ replaced by our approximation).  For the block-diagonal approximation $\breve{F}$ of $\tilde{F}$ (from Section \ref{sec:blockdiag_approx}) this amounts to adding $(\lambda + \eta)I$ (for a lower dimensional $I$) to each of the individual diagonal blocks, which gives modified diagonal blocks of the form
\begin{align}
\label{eqn:exact_tik}
\bar{A}_{i-1,i-1} \otimes G_{i,i} \:+\: (\lambda + \eta) I \:=\: \bar{A}_{i-1,i-1} \otimes G_{i,i} \:+\: (\lambda + \eta) I \otimes I
\end{align}
Because this is the sum of two Kronecker products we cannot use the simple identity $(A \otimes B)^{-1} = A^{-1} \otimes B^{-1}$ anymore.  Fortunately however, there are efficient techniques for inverting such matrices, which we discuss in detail in Appendix \ref{sec:stein_solve}.

If we try to apply this same Tikhonov technique to our more sophisticated approximation $\hat{F}$ of $\tilde{F}$ (from Section \ref{sec:blocktridiag_approx}) by adding $(\lambda + \eta) I$ to each of the diagonal blocks of $\hat{F}$, it is no longer clear how to efficiently invert $\hat{F}$.  Instead, a solution which we have found works very well in practice (and which we also use for the block-diagonal approximation $\breve{F}$), is to add $\pi_i (\sqrt{\lambda + \eta}) I$ and $\displaystyle \frac{1}{\pi_i} (\sqrt{\lambda + \eta}) I$ for a scalar constant $\pi_i$ to the individual Kronecker factors $\bar{A}_{i-1,i-1}$ and $G_{i,i}$ (resp.) of each diagonal block, giving
\begin{align}
\label{eqn:approx_tik}
\left (\bar{A}_{i-1,i-1} + \pi_i (\sqrt{\lambda + \eta}) I \right) \otimes \left (G_{i,i} + \frac{1}{\pi_i} (\sqrt{\lambda + \eta}) I \right)
\end{align}
As this is a single Kronecker product, all of the computations described in Sections \ref{sec:blockdiag_approx} and \ref{sec:blocktridiag_approx} can still be used here too, simply by replacing each $\bar{A}_{i-1,i-1}$ and $G_{i,i}$ with their modified versions $\bar{A}_{i-1,i-1} + \pi_i (\sqrt{\lambda + \eta}) I$ and $\displaystyle G_{i,i} + \frac{1}{\pi_i} (\sqrt{\lambda + \eta}) I$.

To see why the expression in eqn.~\ref{eqn:approx_tik} is a reasonable approximation to eqn.~\ref{eqn:exact_tik}, note that expanding it gives
\begin{align*}
\bar{A}_{i-1,i-1} \otimes G_{i,i} + \pi_i (\sqrt{\lambda + \eta}) I \otimes G_{i,i} +  \frac{1}{\pi_i} (\sqrt{\lambda + \eta}) \bar{A}_{i-1,i-1} \otimes I + (\lambda + \eta) I \otimes I
\end{align*}
which differs from eqn.~\ref{eqn:exact_tik} by the residual error expression
\begin{align*}
\pi_i (\sqrt{\lambda + \eta}) I \otimes G_{i,i} +  \frac{1}{\pi_i} (\sqrt{\lambda + \eta}) \bar{A}_{i-1,i-1} \otimes I
\end{align*}

While the choice of $\pi_i = 1$ is simple and can sometimes work well in practice, a slightly more principled choice can be found by minimizing the obvious upper bound (following from the triangle inequality) on the matrix norm of this residual expression, for some matrix norm $\| \cdot \|_\upsilon$.  This gives
\begin{align*}
\pi_i = \sqrt{\frac{\|\bar{A}_{i-1,i-1} \otimes I\|_\upsilon}{\| I \otimes G_{i,i} \|_\upsilon}}
\end{align*}

Evaluating this expression can be done efficiently for various common choices of the matrix norm $\| \cdot \|_\upsilon$.  For example, for a general $B$ we have $\| I \otimes B \|_F = \| B \otimes I \|_F = \sqrt{d} \|B\|_F$ where $d$ is the height/dimension of $I$, and also $\| I \otimes B \|_2 = \| B \otimes I \|_2 = \| B \|_2$.  

In our experience, one of the best and must robust choices for the norm $\| \cdot \|_\upsilon$ is the trace-norm, which for PSD matrices is given by the trace.  With this choice, the formula for $\pi_i$ has the following simple form:
\begin{align*}
\pi_i = \sqrt{ \frac{\trace(\bar{A}_{i-1,i-1})/(d_{i-1}+1)}{\trace(G_{i,i})/d_i} }
\end{align*}
where $d_i$ is the dimension (number of units) in layer $i$.  Intuitively, the inner fraction is just the average eigenvalue of $\bar{A}_{i-1,i-1}$ divided by the average eigenvalue of $G_{i,i}$.

Interestingly, we have found that this factored approximate Tikhonov approach, which was originally motivated by computational concerns, often works better than the exact version (eqn.~\ref{eqn:exact_tik}) in practice.  The reasons for this are still somewhat mysterious to us, but it may have to do with the fact that the inverse of the product of two quantities is often most robustly estimated as the inverse of the product of their individually regularized estimates.

%The $\gamma$ parameter can be adapted throughout optimization using a simple heuristic which we will describe later in Subsection \ref{sec:adapt_gamma}.

\subsection{Re-scaling according to the exact $F$}
\label{sec:update_rescaling}

Given an update proposal $\Delta$ produced by multiplying the negative gradient $-\nabla h$ by our approximate Fisher inverse (subject to the Tikhonov technique described in the previous subsection), the second stage of our proposed damping scheme re-scales $\Delta$ according to the quadratic model $M$ as computed with the exact $F$, to produce a final update $\delta = \alpha \Delta$.

More precisely, we optimize $\alpha$ according to the value of the quadratic model
\begin{align*}
M(\delta) = M(\alpha \Delta) = \frac{\alpha^2}{2} \Delta^\top (F + (\lambda + \eta) I) \Delta + \alpha \nabla h^\top \Delta + h(\theta)
\end{align*}
as computed using an estimate of the exact Fisher $F$ (to which we add the $\ell_2$ regularization + Tikhonov term $(\lambda + \eta) I$). Because this is a 1-dimensional quadratic minimization problem, the formula for the optimal $\alpha$ can be computed very efficiently as
\begin{align*}
\alpha^* = \frac{-\nabla h^\top \Delta}{\Delta^\top (F + (\lambda + \eta)I ) \Delta} = \frac{-\nabla h^\top \Delta}{\Delta^\top F \Delta + (\lambda + \eta) \|\Delta\|_2^2}
\end{align*}

To evaluate this formula we use the current stochastic gradient $\nabla h$ (i.e. the same one used to produce $\Delta$), and compute matrix-vector products with $F$ using the input data from the same mini-batch.  While using a mini-batch to compute $F$ gets away from the idea of basing our estimate of the curvature on a long history of data (as we do with our \emph{approximate} Fisher), it is made slightly less objectionable by the fact that we are only using it to estimate a single scalar quantity ($\Delta^\top F \Delta$).  This is to be contrasted with methods like HF which perform a long and careful optimization of $M(\delta)$ using such an estimate of $F$.

Because the matrix-vector products with $F$ are only used to compute scalar quantities in \acronym{}, we can reduce their computational cost by roughly one half (versus standard matrix-vector products with $F$) using a simple trick which is discussed in Appendix \ref{sec:cheap_uFv}.

Intuitively, this second stage of our damping scheme effectively compensates for the intrinsic inaccuracy of the approximate quadratic model (based on our approximate Fisher) used to generate the initial update proposal $\Delta$, by essentially falling back on a more accurate quadratic model based on the exact Fisher.

Interestingly, by re-scaling $\Delta$ according to $M(\delta)$, \acronym{} can be viewed as a version of HF which uses our approximate Fisher as a preconditioning matrix (instead of the traditional diagonal preconditioner), and runs CG for only 1 step, initializing it from 0.  This observation suggests running CG for longer, thus obtaining an algorithm which is even closer to HF (although using a much better preconditioner for CG).  Indeed, this approach works reasonably well in our experience, but suffers from some of the same problems that HF has in the stochastic setting, due its much stronger use of the mini-batch--estimated exact $F$.

Figure \ref{fig:damping_rescaling} demonstrates the effectiveness of this re-scaling technique versus the simpler method of just using the raw $\Delta$ as an update proposal.  We can see that $\Delta$, without being re-scaled, is a very poor update to $\theta$, and won't even give \emph{any} improvement in the objective function unless the strength of the factored Tikhonov damping terms is made very large.  On the other hand, when the update is re-scaled, we can afford to compute $\Delta$ using a much smaller strength for the factored Tikhonov damping terms, and overall this yields a much larger and more effective update to $\theta$.

\begin{figure}%[H]
\begin{centering}
\vspace{-0.35in}
\includegraphics[width=0.8\columnwidth]{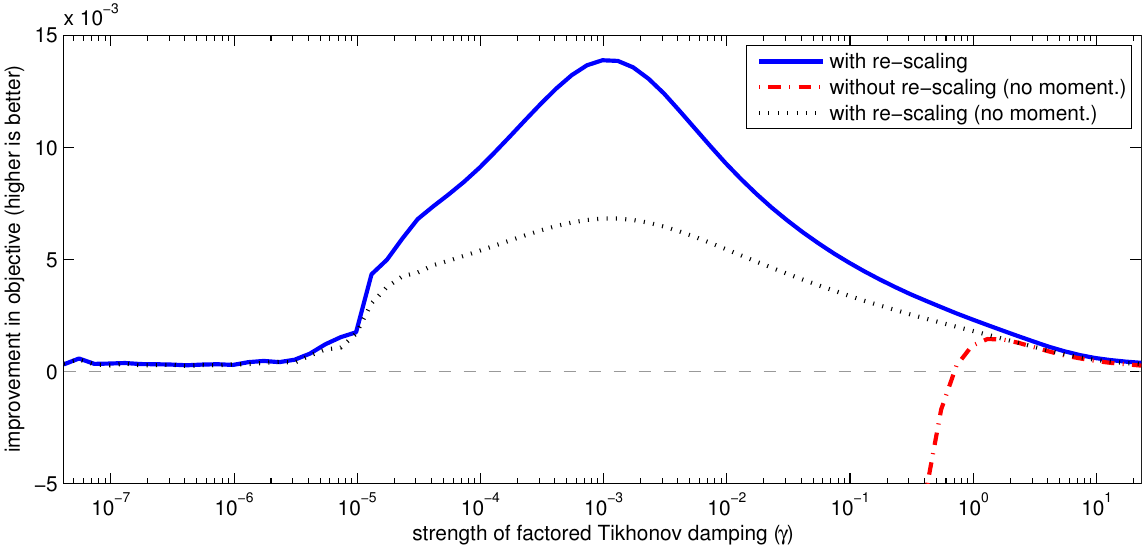}
\caption{\small A comparison of the effectiveness of the proposed damping scheme, with and without the re-scaling techniques described in Section \ref{sec:update_rescaling}.  The network used for this comparison is the one produced at iteration 500 by \acronym{} (with the block-tridiagonal inverse approximation) on the MNIST autoencoder problem described in Section \ref{sec:kfac_experiments}.  The y-axis is the improvement in the objective function $h$ (i.e. $h(\theta) - h(\theta + \delta)$) produced by the update $\delta$, while the x-axis is the strength constant used in the factored Tikhonov damping technique (which is denoted by ``$\gamma$" as described in Section \ref{sec:gamma}).  In the legend, ``no moment." indicates that the momentum technique developed for \acronym{} in Section \ref{sec:momentum} (which relies on the use of re-scaling) was not used.\label{fig:damping_rescaling} }
\end{centering}
\end{figure}

%\vspace{0.4in}

\subsection{Adapting $\lambda$}
\label{sec:adapt_lambda}

%A simple measure of the quality of our choice of $\gamma$ is the (negative) value of the quadratic model $M(\delta) = M(\alpha \Delta)$ for the optimally chosen $\alpha$.  To adjust $\gamma$ based on this measure (or others like it) we use a simple greedy adjustment heuristic.  In particular, every so often during the optimization we try 3 different values of $\gamma$, such as $\gamma_0$, $(3/2) \gamma_0$ and $(2/3)\gamma_0$ where $\gamma_0$ is the current value, and choose the new $\gamma$ to be the best of these (as measured by our quality metric).  In our experiments, ``every so often" usually means every 10-20 iterations of the optimizer.

%We have found that the goodness measure $M(\delta)$ works well in practice, and has the added bonus that it can be computed at essentially no additional cost from the incidental quantities already computed when solving for the optimal $\alpha$.  In our experiments we found it behaves similarly to other obvious ways of measuring the ``goodness" of $\gamma$, such as $h(\theta + \delta)$.

Tikhonov damping can be interpreted as implementing a trust-region constraint on the update $\delta$, so that in particular the constraint $\|\delta\| \leq r$ is imposed for some $r$, where $r$ depends on $\lambda$ and the curvature matrix \citep[e.g.][]{nocedal_book}.  While some approaches adjust $r$ and then seek to find the matching $\lambda$, it is often simpler just to adjust $\lambda$ directly, as the precise relationship between $\lambda$ and $r$ is complicated, and the curvature matrix is constantly evolving as optimization takes place.

The theoretically well-founded Levenberg-Marquardt style rule used by HF for doing this, which we will adopt for \acronym{}, is given by
\begin{itemize}[label={}]
\item $\quad$ if $\rho > 3/4$ then $\lambda \leftarrow \omega_1 \lambda$
\item $\quad$ if $\rho < 1/4$ then $\displaystyle \lambda \leftarrow \frac{1}{\omega_1} \lambda$
\end{itemize}
where $\displaystyle \rho \equiv \frac{h(\theta + \delta) - h(\theta)}  {M(\delta) - M(0)}$ is the ``reduction ratio" and $0 < \omega_1 < 1$ is some decay constant, and all quantities are computed on the current mini-batch (and $M$ uses the exact $F$). 

Intuitively, this rule tries to make $\lambda$ as small as possible (and hence the implicit trust-region as large as possible) while maintaining the property that the quadratic model $M(\delta)$ remains a good \emph{local} approximation to $h$ (in the sense that it accurately predicts the value of $h(\theta + \delta)$ for the $\delta$ which gets chosen at each iteration).  It has the desirable property that as the optimization enters the final convergence stage where $M$ becomes an almost exact approximation in a sufficiently large neighborhood of the local minimum, the value of $\lambda$ will go rapidly enough towards $0$ that it doesn't interfere with the asymptotic local convergence theory enjoyed by 2nd-order methods \citep{levenberg_marquardt}.

In our experiments we applied this rule every $T_1$ iterations of \acronym{}, with $\omega_1 = (19/20)^{T_1}$ and $T_1 = 5$, from a starting value of $\lambda = 150$.  Note that the optimal value of $\omega_1$ and the starting value of $\lambda$ may be application dependent, and setting them inappropriately could significantly slow down \acronym{} in practice. 

Computing $\rho$ can be done quite efficiently.  Note that for the optimal $\delta$, $M(\delta) - M(0) = \frac{1}{2} \nabla h^\top\delta$, and $h(\theta)$ is available from the usual forward pass.  The only remaining quantity which is needed to evaluate $\rho$ is thus $h(\theta + \delta)$, which will require an additional forward pass.  But fortunately, we only need to perform this once every $T_1$ iterations.

\subsection{Maintaining a separate damping strength for the approximate Fisher}
\label{sec:gamma}

While the scheme described in the previous sections works reasonably well in most situations, we have found that in order to avoid certain failure cases and to be truly robust in a large variety of situations, the Tikhonov damping strength parameter for the factored Tikhonov technique described in Section \ref{sec:factored_tik} should be maintained and adjusted independently of $\lambda$.  To this end we replace the expression $\sqrt{\lambda + \eta}$ in Section \ref{sec:factored_tik} with a separate constant $\gamma$, which we initialize to $\sqrt{\lambda + \eta}$ but which is then adjusted using a different rule, which is described at the end of this section.

The reasoning behind this modification is as follows.  The role of $\lambda$, according to the Levenberg Marquardt theory \citep{levenberg_marquardt}, is to be as small as possible while maintaining the property that the quadratic model $M$ remains a trust-worthy approximation of the true objective.  Meanwhile, $\gamma$'s role is to ensure that the initial update proposal $\Delta$ is as good an approximation as possible to the true optimum of $M$ (as computed using a mini-batch estimate of the exact $F$), so that in particular the re-scaling performed in Section \ref{sec:update_rescaling} is as benign as possible.  While one might hope that adding the same multiple of the identity to our approximate Fisher as we do to the exact $F$ (as it appears in $M$) would produce the best $\Delta$ in this regard, this isn't obviously the case.  In particular, using a larger multiple may help compensate for the approximation we are making to the Fisher when computing $\Delta$, and thus help produce a more ``conservative" but ultimately more useful initial update proposal $\Delta$, which is what we observe happens in practice.

A simple measure of the quality of our choice of $\gamma$ is the (negative) value of the quadratic model $M(\delta) = M(\alpha \Delta)$ for the optimally chosen $\alpha$.  To adjust $\gamma$ based on this measure (or others like it) we use a simple greedy adjustment rule.  In particular, every $T_2$ iterations during the optimization we try 3 different values of $\gamma$ ($\gamma_0$, $\omega_2 \gamma_0$, and $(1/\omega_2) \gamma_0$, where $\gamma_0$ is the current value) and choose the new $\gamma$ to be the best of these, as measured by our quality metric.  In our experiments we used $T_2 = 20$ (which must be a multiple of the constant $T_3$ as defined in Section \ref{sec:efficiency}), and $\omega_2 = (\sqrt{19/20})^{T_2}$.

We have found that $M(\delta)$ works well in practice as a measure of the quality of $\gamma$, and has the added bonus that it can be computed at essentially no additional cost from the incidental quantities already computed when solving for the optimal $\alpha$.  In our initial experiments we found that using it gave similar results to those obtained by using other obvious measures for the quality of $\gamma$, such as $h(\theta + \delta)$.

\section{Momentum}
\label{sec:momentum}

\citet{momentum_ilya} found that momentum \citep{momentum_orig, early_nnet_momentum} was very helpful in the context of stochastic gradient descent optimization of deep neural networks.  A version of momentum is also present in the original HF method, and it plays an arguably even more important role in more ``stochastic" versions of HF \citep{HF_chapter,ryan_hf}. 

A natural way of adding momentum to \acronym{}, and one which we have found works well in practice, is to take the update to be $\delta = \alpha \Delta + \mu \delta_0$, where $\delta_0$ is the final update computed at the previous iteration, and where $\alpha$ and $\mu$ are chosen to minimize $M(\delta)$.  This allows \acronym{} to effectively build up a better solution to the local quadratic optimization problem $\min_\delta M(\delta)$ (where $M$ uses the \emph{exact} $F$) over many iterations, somewhat similarly to how Matrix Momentum \citep{MM_natural} and HF do this \citep[see][]{momentum_ilya}.  %As in Matrix Momentum, we can expect that if the definition of $M(\delta)$ stabilizes (which will happen if both the optimizer convergences and the noise in the gradient estimate shrinks), $\delta$ will tend to the exact minimizer of $M$, or in other words the exact natural gradient update (modulo appropriate damping).

The optimal solution for $\alpha$ and $\mu$ can be computed as
\begin{align*}
\begin{bmatrix}
\alpha^* \\ \mu^*
\end{bmatrix}
%= -
%\begin{bmatrix}
%\Delta^\top (F + (\lambda + \eta)I) \Delta  & \Delta^\top (F + (\lambda + \eta)I) \delta_0 \\
%\delta_0^\top (F + (\lambda + \eta)I) \Delta  & \delta_0^\top (F + (\lambda + \eta)I) \delta_0
%\end{bmatrix}^{-1}
%\begin{bmatrix}
%\nabla h^\top \Delta \\ \nabla h^\top \delta_0
%\end{bmatrix}
= -
\begin{bmatrix}
\Delta^\top F \Delta + (\lambda + \eta)\|\Delta\|_2^2  & \Delta^\top F \delta_0 + (\lambda + \eta) \Delta^\top \delta_0\\
\Delta^\top F \delta_0 + (\lambda + \eta) \Delta^\top \delta_0 & \delta_0^\top F \delta_0 + (\lambda + \eta)\|\delta_0\|_2^2
\end{bmatrix}^{-1}
\begin{bmatrix}
\nabla h^\top \Delta \\ \nabla h^\top \delta_0
\end{bmatrix}
\end{align*}

The main cost in evaluating this formula is computing the two matrix-vector products $F \Delta$ and $F \delta_0$.  Fortunately, the technique discussed in Appendix \ref{sec:cheap_uFv} can be applied here to compute the 4 required scalars at the cost of only two forwards passes (equivalent to the cost of only one matrix-vector product with $F$).

%This method of introducing momentum into the optimization is remenisent of the type of momentum present naturally in HF approach.  Indeed, just like 
%As discussed in Section \ref{sec:update_rescaling}, the rescaling of the update $\delta$ according to the quadratic model 
%turns our proposed approach into a version of HF where CG is preconditioned with the approximate Fisher and run for only one step.  The addition of $\Delta_0$ to the search space is very similar to how in HF one uses the previous update is used to initialize CG.  However, there is one subtle difference.  Here we are searching over the subspace $\ospan\{ \delta, \Delta_0 \}$, whereas one step of CG initialized from $\Delta_0$ searches the space $???$.
%With the addition of the direction $\Delta_0$ to the searchable subspace, this method becomes similar to 
% and the aforementioned connection that our update scaling damping technique has with that %method.  It 

Empirically we have found that this type of momentum provides substantial acceleration in regimes where the gradient signal has a low noise to signal ratio, which is usually the case in the early to mid stages of stochastic optimization, but can also be the case in later stages if the mini-batch size is made sufficiently large. % where the optimization problem at hand more closely resembles classical deterministic optimization instead of stochastic estimation.  
These findings are consistent with predictions made by convex optimization theory, and with older empirical work done on neural network optimization \citep{lecun_tricks}. 

Notably, because the implicit ``momentum decay constant" $\mu$ in our method is being computed on the fly, one doesn't have to worry about setting schedules for it, or adjusting it via heuristics, as one often does in the context of SGD.

Interestingly, if $h$ is a quadratic function (so the definition of $M(\delta)$ remains fixed at each iteration) and all quantities are computed deterministically (i.e. without noise), then using this type of momentum makes \acronym{} equivalent to performing preconditioned linear CG on $M(\delta)$, with the preconditioner given by our approximate Fisher.  This follows from the fact that linear CG can be interpreted as a momentum method where the learning rate $\alpha$ and momentum decay coefficient $\mu$ are chosen to jointly minimize $M(\delta)$ at the current iteration.

\section{Computational Costs and Efficiency Improvements}
\label{sec:efficiency}

Let $d$ be the typical number of units in each layer and $m$ the mini-batch size. The significant computational tasks required to compute a single update/iteration of \acronym{}, and rough estimates of their associated computational costs, are as follows:
\begin{enumerate}
\item standard forwards and backwards pass: $2C_1 \ell d^2 m$ \label{task:standard_fb}
\item computation of the gradient $\nabla h$ on the current mini-batch using quantities computed in backwards pass: $C_2 \ell d^2 m$ \label{task:gradient}
\item additional backwards pass with random targets (as described in Section \ref{sec:estimating_A_and_G}): $C_1 \ell d^2 m$ \label{task:extra_backwards}
\item updating the estimates of the required $\bar{A}_{i,j}$'s and $G_{i,j}$'s from quantities computed in the forwards pass and the additional randomized backwards pass:  $2C_2 \ell d^2 m$ \label{task:A_G_est}
\item matrix inverses (or SVDs for the block-tridiagonal inverse, as described in Appendix \ref{sec:stein_solve}) required to compute the inverse of the approximate Fisher: $C_3 \ell d^3$ for the block-diagonal inverse, $C_4 \ell d^3$ for the block-tridiagonal inverse \label{task:inverses_and_svds}
\item various matrix-matrix products required to compute the matrix-vector product of the approximate inverse with the stochastic gradient: $C_5 \ell d^3$ for the block-diagonal inverse, $C_6 \ell d^3$ for the block-tridiagonal inverse \label{task:approxF_product}
\item matrix-vector products with the exact $F$ on the current mini-batch using the approach in Appendix \ref{sec:cheap_uFv}: $4C_1 \ell d^2 m$ with momentum, $2C_1 \ell d^2 m$ without momentum \label{task:exactF_product}
\item additional forward pass required to evaluate the reduction ratio $\rho$ needed to apply the $\lambda$ adjustment rule described in Section \ref{sec:adapt_lambda}: $C_1 \ell d^2 m$ every $T_1$ iterations \label{task:rho_comp}
\end{enumerate}
Here the $C_i$ are various constants that account for implementation details, and we are assuming the use of the naive cubic matrix-matrix multiplication and inversion algorithms when producing the cost estimates.  Note that it it is hard to assign precise values to the constants, as they very much depend on how these various tasks are implemented.

Note that most of the computations required for these tasks will be sped up greatly by performing them in parallel across units, layers, training cases, or all of these.  The above cost estimates however measure sequential operations, and thus may not accurately reflect the true computation times enjoyed by a parallel implementation.  In our experiments we used a vectorized implementation that performed the computations in parallel over units and training cases, although not over layers (which is possible for computations that don't involve a sequential forwards or backwards ``pass" over the layers).

Tasks \ref{task:standard_fb} and \ref{task:gradient} represent the standard stochastic gradient computation.

The costs of tasks \ref{task:extra_backwards} and \ref{task:A_G_est} are similar and slightly smaller than those of tasks \ref{task:standard_fb} and \ref{task:gradient}.  One way to significantly reduce them is to use a random subset of the current mini-batch of size $\tau_1 m$ to update the estimates of the required $\bar{A}_{i,j}$'s and $G_{i,j}$'s.  One can similarly reduce the cost of task \ref{task:exactF_product} by computing the (factored) matrix-vector product with $F$ using such a subset of size $\tau_2 m$, although we recommend proceeding with caution when doing this, as using inconsistent sets of data for the quadratic and linear terms in $M(\delta)$ can hypothetically cause instability problems which are avoided by using consistent data (see \citet{HF_chapter}, Section 13.1).  In our experiments in Section \ref{sec:kfac_experiments} we used $\tau_1 = 1/8$ and $\tau_2 = 1/4$, which seemed to have a negligible effect on the quality of the resultant updates, while significantly reducing per-iteration computation time.  In a separate set of unreported experiments we found that in certain situations, such as when $\ell_2$ regularization isn't used and the network starts heavily overfitting the data, or when smaller mini-batches were used, we had to revert to using $\tau_2 = 1$ to prevent significant deterioration in the quality of the updates.

The cost of task \ref{task:rho_comp} can be made relatively insignificant by making the adjustment period $T_1$ for $\lambda$ large enough.  We used $T_1 = 5$ in our experiments.

%The two tasks whose costs are hard to compare directly with the cost of computing the gradient \ref{task:inverses_and_svds} and \ref{task:approxF_product}. 
The costs of tasks \ref{task:inverses_and_svds} and \ref{task:approxF_product} are hard to compare directly with the costs associated with computing the gradient, as their relative sizes will depend on factors such as the architecture of the neural network being trained, as well as the particulars of the implementation.  However, one quick observation we can make is that both tasks \ref{task:inverses_and_svds} and \ref{task:approxF_product} involve computations that be performed in parallel across the different layers, which is to be contrasted with many of the other tasks which require \emph{sequential} passes over the layers of the network.  

Clearly, if $m \gg d$, then the cost of tasks \ref{task:inverses_and_svds} and \ref{task:approxF_product} becomes negligible in comparison to the others.  However, it is more often the case that $m$ is comparable or perhaps smaller than $d$.  Moreover, while algorithms for inverses and SVDs tend to have the same asymptotic cost as matrix-matrix multiplication, they are at least several times more expensive in practice, in addition to being harder to parallelize on modern GPU architectures (indeed, CPU implementations are often faster in our experience).   Thus, $C_3$ and $C_4$ will typically be (much) larger than $C_5$ and $C_6$, and so in a basic/naive implementation of \acronym{}, task \ref{task:inverses_and_svds} can dominate the overall per-iteration cost. 

Fortunately, there are several possible ways to mitigate the cost of task \ref{task:inverses_and_svds}.  As mentioned above, one way is to perform the computations for each layer in parallel, and even simultaneously with the gradient computation and other tasks.  In the case of our block-tridiagonal approximation to the inverse, one can avoid computing any SVDs or matrix square roots by using an iterative Stein-equation solver (see Appendix \ref{sec:stein_solve}).  And there are also ways of reducing matrix-inversion (and even matrix square-root) to a short sequence of matrix-matrix multiplications using iterative methods \citep{newton_inversion}.  Furthermore, because the matrices in question only change slowly over time, one can consider hot-starting these iterative inversion methods from previous solutions.  In the extreme case where $d$ is very large, one can also consider using low-rank + diagonal approximations of the $\bar{A}_{i,j}$ and $G_{i,j}$ matrices maintained online (e.g. using a similar strategy as \citet{TONGA}) from which inverses and/or SVDs can be more easily computed. Although based on our experience such approximations can, in some cases, lead to a substantial degradation in the quality of the updates.

While these ideas work reasonably well in practice, perhaps the simplest method, and the one we ended up settling on for our experiments, is to simply recompute the approximate Fisher inverse only every $T_3$ iterations (we used $T_3 = 20$ in our experiments).  As it turns out, the curvature of the objective stays relatively stable during optimization, especially in the later stages, and so in our experience this strategy results in only a modest decrease in the quality of the updates.

If $m$ is much smaller than $d$, the costs associated with task \ref{task:approxF_product} can begin to dominate (provided $T_3$ is sufficiently large so that the cost of task \ref{task:inverses_and_svds} is relatively small).  And unlike task \ref{task:inverses_and_svds}, task \ref{task:approxF_product} must be performed at every iteration.  While the simplest solution is to increase $m$ (while reaping the benefits of a less noisy gradient), in the case of the block-diagonal inverse it turns out that we can change the cost of task \ref{task:approxF_product} from $C_5 \ell d^3$ to $C_5 \ell d^2 m$ by taking advantage of the low-rank structure of the stochastic gradient.  The method for doing this is described below.

Let $\bar{\mathcal{A}}_i$ and $\mathcal{G}_i$ be matrices whose columns are the $m$ $\bar{a}_i$'s and $g_i$'s (resp.)~associated with the current mini-batch.  Let $\nabla_{W_i} h$ denote the gradient of $h$ with respect to $W_i$, shaped as a matrix (instead of a vector).  The estimate of $\nabla_{W_i} h$ over the mini-batch is given by $\frac{1}{m} \mathcal{G}_i \bar{\mathcal{A}}_{i-1}^\top$, which is of rank-$m$.  From Section \ref{sec:blockdiag_approx}, computing the $\breve{F}^{-1} \nabla h$ amounts to computing $U_i = G_{i,i}^{-1} (\nabla_{W_i} h) \bar{A}_{i-1,i-1}^{-1}$.  Substituting in our mini-batch estimate of $\nabla_{W_i} h$ gives
\begin{align*}
U_i &= G_{i,i}^{-1} \left( \frac{1}{m} \mathcal{G}_i \bar{\mathcal{A}}_{i-1}^\top \right) \bar{A}_{i-1,i-1}^{-1} = \frac{1}{m} \left( G_{i,i}^{-1} \mathcal{G}_i \right) \left( \bar{\mathcal{A}}_{i-1}^\top \bar{A}_{i-1,i-1}^{-1} \right)
\end{align*}
Direct evaluation of the expression on the right-hand side involves only matrix-matrix multiplications between matrices of size $m \times d$ and $d \times m$ (or between those of size $d \times d$ and $d \times m$), and thus we can reduce the cost of task \ref{task:approxF_product} to $C_5 \ell d^2 m$.

Note that the use of standard $\ell_2$ weight-decay is not compatible with this trick.  This is because the contribution of the weight-decay term to each $\nabla_{W_i} h$ is $\eta W_i$, which will typically not be low-rank.  Some possible ways around this issue include computing the weight-decay contribution $\eta \breve{F}^{-1} \theta$ separately and refreshing it only occasionally, or using a different regularization method, such as drop-out \citep{dropout} or weight-magnitude constraints.

Note that the adjustment technique for $\gamma$ described in Section \ref{sec:gamma} requires that, at every $T_2$ iterations, we compute 3 different versions of the update for each of 3 candidate values of $\gamma$.  In an ideal implementation these could be computed in parallel with each other, although in the summary analysis below we will assume they are computed serially.

Summarizing, we have that with all of the various efficiency improvements discussed in this section, the average per-iteration computational cost of \acronym{}, in terms of \emph{serial} arithmetic operations, is
\begin{align*}
[ (2 &+ \tau_1 + 2(1+\chi_{mom})(1+2/T_2)\tau_2 + 1/T_1)C_1 + (1 + 2 \tau_1) C_2 ] \ell d^2 m \\
&+ (1+2/T_2) [(C_4/T_3 + C_6)\chi_{tri} + C_3/T_3 (1-\chi_{tri})] \ell d^3 + (1+2/T_2) C_5  (1-\chi_{tri}) \ell d^2\min\{d,m\}
\end{align*}
where $\chi_{mom}, \chi_{tri} \in \{0,1\}$ are flag variables indicating whether momentum and the block-tridiagonal inverse approximation (resp.) are used.  

Plugging in the values of these various constants that we used in our experiments, for the block-diagonal inverse approximation ($\chi_{tri} = 0$) this becomes
\begin{align*}
( 3.425 C_1 + 1.25 C_2 ) \ell d^2 m +  0.055 C_3 \ell d^3 + 1.1 C_5 \ell d^2\min\{d,m\}
\end{align*}
and for the block-tridiagonal inverse approximation ($\chi_{tri} = 1$)
\begin{align*}
( 3.425 C_1 + 1.25 C_2 ) \ell d^2 m + (0.055 C_4 + 1.1 C_6) \ell d^3
\end{align*}
which is to be compared to the per-iteration cost of SGD, as given by
\begin{align*}
(2C_1 + C_2) \ell d^2 m
\end{align*}

\section{Pseudocode for \acronym{}}
\label{sec:pseudocode}
\vspace{-0.2in}
Algorithm \ref{alg:K-FAC} gives high-level pseudocode for the \acronym{} method, with the details of how to perform the computations required for each major step left to their respective sections.
%\vspace{-0.1in}
\begin{algorithm}[H]
\small
\begin{spacing}{1.02}
\begin{algorithmic}
\STATE $\bullet$ Initialize $\theta_1$ (e.g. using a good method such as the ones described in \citet{HF} or \citet{bengio_glorot})
\STATE $\bullet$ Choose initial values of $\lambda$ (err on the side of making it too large)
\STATE $\bullet$ $\gamma \gets \sqrt{\lambda + \eta}$ 
\STATE $\bullet$ $k \gets 1$
\WHILE{ $\theta_k$ is not satisfactory }
\STATE $\bullet$ Choose a mini-batch size $m$ (e.g. using a fixed value, an adaptive rule, or some predefined schedule)
\STATE $\bullet$ Select a random mini-batch $S' \subset S$ of training cases of size $m$
\STATE $\bullet$ Select a random subset $S_1 \subset S'$ of size $\tau_1 |S'|$
\STATE $\bullet$ Select a random subset $S_2 \subset S'$ of size $\tau_2 |S'|$
\STATE $\bullet$ Perform a forward and backward pass on $S'$ to estimate the gradient $\nabla h(\theta_k)$ (see Algorithm \ref{alg:nnet_gradient})
\STATE $\bullet$ Perform an additional backwards pass on $S_1$ using random targets generated from the model's predictive distribution (as described in Section \ref{sec:estimating_A_and_G})
\STATE $\bullet$ Update the estimates of the required $\bar{A}_{i,j}$'s and $G_{i,j}$'s using the $a_i$'s computed in forward pass for $S_1$, and the $g_i$'s computed in the additional backwards pass for $S_1$ (as described Section \ref{sec:estimating_A_and_G})
\STATE $\bullet$ Choose a set $\Gamma$ of new candidate $\gamma$'s as described in Section \ref{sec:gamma} (setting $\Gamma = \{\gamma\}$ if not adjusting $\gamma$ at this iteration, i.e. if $k \not\equiv 0 \pmod{T_2}$ )
\FOR {each $\gamma \in \Gamma$}
\IF{ recomputing the approximate Fisher inverse this iteration (i.e. if $k \equiv 0 \pmod{T_3}$ or $k \leq 3$)}
\STATE $\bullet$ Compute the approximate Fisher inverse (using the formulas derived in Section \ref{sec:blockdiag_approx} or Section \ref{sec:blocktridiag_approx}) from versions of the current $\bar{A}_{i,j}$'s and $G_{i,j}$'s which are modified as per the factored Tikhonov damping technique described in Section \ref{sec:factored_tik} (but using $\gamma$ as described in Section \ref{sec:gamma})
\ENDIF
\STATE $\bullet$ Compute the update proposal $\Delta$ by multiplying current estimate of approximate Fisher inverse by the estimate of $\nabla h$ (using the formulas derived in Section \ref{sec:blockdiag_approx} or Section \ref{sec:blocktridiag_approx}).  For layers with size $d < m$ consider using trick described at the end of Section \ref{sec:efficiency} for increased efficiency.
\STATE $\bullet$ Compute the final update $\delta$ from $\Delta$ as described in Section \ref{sec:update_rescaling} (or Section \ref{sec:momentum} if using momentum) where the matrix-vector products with $F$ are estimated on $S_2$ using the $a_i$'s computed in the forward pass
\ENDFOR
\STATE $\bullet$ Select the $\delta$ and the new $\gamma$ computing in the above loop that correspond to the lowest value of $M(\delta)$ (see Section \ref{sec:gamma})
\IF{ updating $\lambda$ this iteration (i.e. if $k \equiv 0 \pmod {T_1}$) }
	\STATE $\bullet$ Update $\lambda$ with the Levenberg-Marquardt style rule described in Section \ref{sec:adapt_lambda}
\ENDIF
\STATE $\bullet$ $\theta_{k+1} \gets \theta_k + \delta$
\STATE $\bullet$ $k \gets k+1$
\ENDWHILE
%\STATE {\bf output:} $\theta_k$
\end{algorithmic}
\end{spacing}
\caption{ High-level pseudocode for \acronym{} \label{alg:K-FAC} }
\end{algorithm}

\section{Invariance Properties and the Relationship to Whitening and Centering}
\label{sec:invariance}

When computed with the exact Fisher, the natural gradient specifies a direction in the space of predictive distributions which is invariant to the specific way that the model is parameterized.  This invariance means that the smooth path through distribution space produced by following the natural gradient with infinitesimally small steps will be similarly invariant.

For a practical natural gradient based optimization method which takes large discrete steps in the direction of the natural gradient, this invariance of the optimization path will only hold approximately.  As shown by \citet{ng_martens}, the approximation error will go to zero as the effects of damping diminish and the reparameterizing function $\zeta$ tends to a locally linear function.  Note that the latter will happen as $\zeta$ becomes smoother, or the local region containing the update shrinks to zero.  

Because \acronym{} uses an approximation of the natural gradient, these invariance results are not applicable in our case.  Fortunately, as was shown by \citet{ng_martens}, one can establish invariance of an update direction with respect to a given reparameterization of the model by verifying certain simple properties of the curvature matrix $C$ used to compute the update.  We will use this result to show that, under the assumption that damping is absent (or negligible in its affect), \acronym{} is invariant to a broad and natural class of transformations of the network.%, which end up being equivalent to a certain type of linear reparameterization.

This class of transformations is given by the following modified network definition (c.f. the definition in Section \ref{sec:neural_networks}):
\begin{align*}
s^\dagger_i &= W_i^\dagger \bar{a}^\dagger_{i-1} \\
\bar{a}^\dagger_i &= \Omega_i \bar{\nonlin}_i( \Phi_i s^\dagger_i)
\end{align*}
where $\bar{\nonlin}_i$ is the function that computes $\nonlin_i$ and then appends a homogeneous coordinate (with value 1), $\Omega_i$ and $\Phi_i$ are arbitrary invertible matrices of the appropriate sizes (except that we assume $\Omega_\ell = I$), $\bar{a}^\dagger_0 = \Omega_0 \bar{a}_0$, and where the network's output is given by $f^\dagger(x,\theta) = a^\dagger_\ell$.  Note that because $\Omega_i$ multiplies $\bar{\nonlin}_i(\Phi_i s^\dagger_i)$, it can implement arbitrary translations of the unit activities $\nonlin_i(\Phi_i s^\dagger_i)$ in addition to arbitrary linear transformations. %Figure \ref{fig:nnet_trans} depicts the transformed version of the network for $\ell = 2$.
Figure \ref{fig:nnet_trans} illustrates our modified network definition for $\ell = 2$ (c.f. Figure \ref{fig:nnet}).

Here, and going forward, we will add a ``$\dagger$" superscript to any network-dependent quantity in order to denote the analogous version of it computed by the transformed network.  Note that under this identification, the loss derivative formulas for the transformed network are analogous to those of the original network, and so our various Fisher approximations are still well defined.

\begin{figure}%[H]
\begin{centering}
\includegraphics[width=.4\columnwidth]{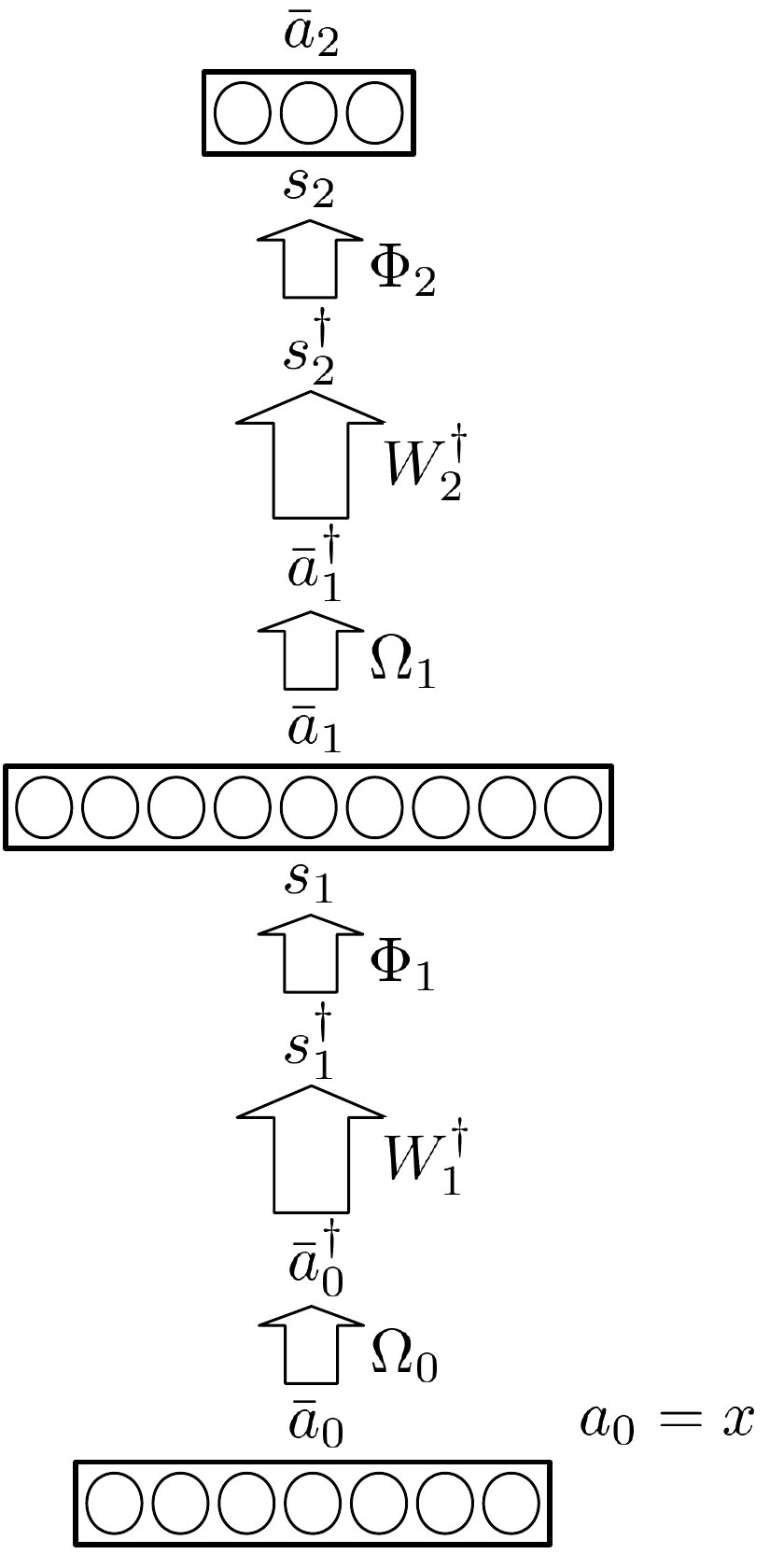}
\caption{\small A depiction of a transformed network for $\ell = 2$.  Note that the quantities labeled with $\bar{a_i}$ and $s_i$ (without ``$\dagger$") will be equal to the analogous quantities from the default network, provided that $\theta = \zeta(\theta^\dagger)$ as in Theorem \ref{thm:invariance}.  \label{fig:nnet_trans} }
\end{centering}
\end{figure}
\newpage

The following theorem describes the main technical result of this section.

\begin{theorem}
\label{thm:invariance}
There exists an invertible linear function $\theta = \zeta(\theta^\dagger)$ so that $f^\dagger(x,\theta^\dagger) = f(x,\theta) = f(x,\zeta(\theta^\dagger))$, and thus the transformed network can be viewed as a reparameterization of the original network by $\theta^\dagger$.  Moreover, additively updating $\theta$ by $\delta = -\alpha \breve{F}^{-1} \nabla h$ or $\delta = -\alpha \hat{F}^{-1} \nabla h$ in the original network is equivalent to additively updating $\theta^\dagger$ by $\delta^\dagger = -\alpha \breve{F}^{\dagger-1} \nabla h^\dagger$ or $\delta^\dagger = -\alpha \hat{F}^{\dagger-1} \nabla h^\dagger$ (resp.) in the transformed network, in the sense that $\zeta(\theta^\dagger + \delta^\dagger) = \theta + \delta$.
\end{theorem}

This immediately implies the following corollary which characterizes the invariance of a basic version of \acronym{} to the given class of network transformations.

\begin{corollary}
\label{cor:invariance}
The optimization path taken by \acronym{} (using either of our Fisher approximations $\breve{F}$ or $\hat{F}$) through the space of predictive distributions is the same for the default network as it is for the transformed network (where the $\Omega_i$'s and $\Phi_i$'s remain fixed).  This assumes the use of an equivalent initialization ($\theta_0 = \zeta(\theta_0^\dagger)$), and a basic version of \acronym{} where damping is absent or negligible in effect, momentum is not used, and where the learning rates are chosen in a way that is independent of the network's parameterization.
\end{corollary}

While this corollary assumes that the $\Omega_i$'s and $\Phi_i$'s are fixed, if we relax this assumption so that they are allowed to vary smoothly with $\theta$, then $\zeta$ will be a smooth function of $\theta$, and so as discussed in \citet{ng_martens}, invariance of the optimization path will hold approximately in a way that depends on the smoothness of $\zeta$ (which measures how quickly the $\Omega_i$'s and $\Phi_i$'s change) and the size of the update.  Moreover, invariance will hold exactly in the limit as the learning rate goes to 0.

Note that the network transformations can be interpreted as replacing the network's nonlinearity $\bar{\nonlin}_i (s_i)$ at each layer $i$ with a ``transformed" version $\Omega_i \bar{\nonlin}_i( \Phi_i s_i)$.  So since the well-known logistic sigmoid and tanh functions are related to each other by such a transformation, an immediate consequence of Corollary \ref{cor:invariance} is that \acronym{} is invariant to the choice of logistic sigmoid vs.~tanh activation functions (provided that equivalent initializations are used and that the effect of damping is negligible, etc.).  

Also note that because the network inputs are also transformed by $\Omega_0$, \acronym{} is thus invariant to arbitrary affine transformations of the input, which includes many popular training data preprocessing techniques.  %And as we will see below, it can be interpreted as performing these kinds of transformations automatically.

Many other natural network transformations, such as ones which ``center" and normalize unit activities so that they have mean 0 and variance 1 can be described using diagonal choices for the $\Omega_i$'s and $\Phi_i$'s which vary smoothly with $\theta$.  In addition to being approximately invariant to such transformations (or exactly, in the limit as the step size goes to 0), \acronym{} is similarly invariant to a more general class of such transformations, such as those which transform the units so that they have a mean of 0, so they are ``centered", and a \emph{covariance matrix} of $I$, so they are ``whitened", which is a much stronger condition than the variances of the individual units each being 1.  

%In the case where $\Delta$ is computed simply as $\Delta = \breve{F}^{-1} \nabla h$ from our block-diagonal approximation $\breve{F}$ of the Fisher (i.e. without the use of damping techniques or other complications), it has a simple and intuitive interpretation, which is stated in the follow corollary.

In the case where we use the block-diagonal approximation $\breve{F}$ and compute updates without damping, Theorem \ref{thm:invariance} affords us an additional elegant interpretation of what \acronym{} is doing.  In particular, the updates produced by \acronym{} end up being equivalent to those produced by \emph{standard gradient descent} using a network which is transformed so that the unit activities and the unit-gradients are both centered and whitened (with respect to the model's distribution).  This is stated formally in the following corollary.
\begin{corollary}
\label{cor:whitened_interpretation}
Additively updating $\theta$ by $-\alpha \breve{F}^{-1} \nabla h$ in the original network is equivalent to additively updating $\theta^\dagger$ by the gradient descent update $-\alpha \nabla h^\dagger$ (where $\theta = \zeta(\theta^\dagger)$ as in Theorem \ref{thm:invariance}) in a transformed version of the network where the unit activities $a^\dagger_i$ and the unit-gradients $g^\dagger_i$ are both centered and whitened with respect to the model's distribution.
\end{corollary}

\section{Related Work}

\label{sec:related}

The Hessian-free optimization method of \citet{HF} uses linear conjugate gradient (CG) to optimize local quadratic models of the form of eqn.~\ref{eqn:quad_model} (subject to an adaptive Tikhonov damping technique) in lieu of directly solving it using matrix inverses.  As discussed in the introduction, the main advantages of \acronym{} over HF are twofold.  Firstly, \acronym{} uses an efficiently computable direct solution for the inverse of the curvature matrix and thus avoids the costly matrix-vector products associated with running CG within HF.  Secondly, it can estimate the curvature matrix from a lot of data by using an online exponentially-decayed average, as opposed to relatively small-sized fixed mini-batches used by HF.  The cost of doing this is of course the use of an inexact approximation to the curvature matrix.

%[[\textbf{Briefly discuss FANG here}]]

\citet{TONGA} proposed a neural network optimization method known as TONGA based on a block-diagonal approximation of the \emph{empirical} Fisher where each block corresponds to the weights associated with a particular unit.  By contrast, \acronym{} uses \emph{much} larger blocks, each of which corresponds to all the weights within a particular layer.  The matrices which are inverted in \acronym{} are roughly the same size as those which are inverted in TONGA, but rather than there being one per unit as in TONGA, there are only two per layer. Therefore, \acronym{} is significantly less computationally intensive than TONGA, despite using what is arguably a much more accurate approximation to the Fisher.  Note that to help mitigate the cost of the many matrix inversions it requires, TONGA approximates the blocks as being low-rank plus a diagonal term, although this introduces further approximation error.

Centering methods work by either modifying the gradient \citep{Schraudolph_centering} or dynamically reparameterizing the network itself \citep{Tapani_centering, Vatanen_centering, Wiesler_centering}, so that various unit-wise scalar quantities like the activities (the $a_i$'s) and local derivatives (the $\phi_i'(s_i)$'s) %(and possibly also backpropagated loss derivatives) 
are 0 on average (i.e. ``centered"), as they appear in the formula for the gradient.  Typically, these methods require the introduction of additional ``skip" connections (which bypass the nonlinearities of a given layer) in order to preserve the expressive power/efficiency of the network after these transformations are applied.

It is argued by \citet{Tapani_centering} that the application of the centering transformation makes the Fisher of the resulting network closer to a diagonal matrix, and thus makes its gradient more closely resemble its natural gradient.  However, this argument uses the strong approximating assumption that the correlations between various network-dependent quantities, such as the activities of different units within a given layer, are zero.  In our notation, this would be like assuming that the $G_{i,i}$'s are diagonal, and that the $\bar{A}_{i,i}$'s are rank-1 plus a diagonal term.  Indeed, using such an approximation within the block-diagonal version of \acronym{} would yield an algorithm similar to standard centering, although without the need for skip connections (and hence similar to the version of centering proposed by \citet{Wiesler_centering}).

%Sketch proof of equivalence to centering (assuming G = I, A = rank 1 + diag as needed):
%--
%Both start at default param.  After one update, both methods are at the same model, but K-FAC restores default param, while centering uses new one.
%
%We know from the theorem that if K-FAC was applied to a model with this new param, it would generate the same update as it would for the default param.  So the question is now:
%
%Given a non-default param (corresponding to a network transform), do both K-FAC and centering produce the same update?  Well, probably yes, since the non-linearity can just absorb the reparam and we can treat the whole thing the same as a different "default param" for both centering and K-FAC

As shown in Corollary \ref{cor:whitened_interpretation}, \acronym{} can also be interpreted as using the gradient of a transformed network as its update direction, although one in which the $g_i$'s and $a_i$'s are both centered and \emph{whitened} (with respect to the model's distribution). Intuitively, it is this whitening which accounts for the correlations between activities (or back-propagated gradients) within a given layer. %, which would otherwise be ignored by only doing centering.

%\acronym{} also has the additional advantage over centering methods that it doesn't require the network to be augmented with skip connections.

\citet{Ollivier} proposed a neural network optimization method which uses a block-diagonal approximation of the Fisher, with the blocks corresponding to the incoming weights (and bias) of each unit.  This method is similar to TONGA, except that it approximates the Fisher instead of the empirical Fisher (see \citet{ng_martens} for a discussion of the difference between these).  Because computing blocks of the Fisher is expensive (it requires $k$ backpropagations, where $k$ is the number of output units), this method uses a biased deterministic approximation which can be computed more efficiently, and is similar in spirit to the deterministic approximation used by \citet{lecun_tricks}.  %The author shows that this approximate Fisher has the interpretation of being the Hessian of a parameterization invariant metric (and thus grants analogous invariance to optimization methods which use this approximate Fisher).
Note that while such an approximation could hypothetically be used within \acronym{} to compute the $G_{i,j}$'s, we have found that our basic unbiased stochastic approximation works nearly as well as the exact values in practice.

%As is the case with TONGA, the inverting all of these blocks can become expensive in large networks, and so the author proposes additional approximation which further simplifies the computations [[and I have no idea how to interpret this ``quasi-diagonal" approximation]]. This is to be constrasted with TONGA which uses a dynamically adjusted low-rank + diagonal approximation.

The work most closely related to ours is that of \citet{heskes}, who proposed an approximation of the Fisher of feed-forward neural networks similar to our Kronecker-factored block-diagonal approximation $\breve{F}$ from Section \ref{sec:blockdiag_approx}, and used it to derive an efficient approximate natural-gradient based optimization method by exploiting the identity $(A \otimes B)^{-1} = A^{-1} \otimes B^{-1}$.  \acronym{} differs from Heskes' method in several important ways which turn out to be crucial to it working well in practice.  

In Heskes' method, update damping is accomplished using a basic factored Tikhonov technique where $\gamma I$ is added to each $G_{i,i}$ and $\bar{A}_{i,i}$ for a fixed parameter $\gamma > 0$ which is set by hand.  By contrast, \acronym{} uses a factored Tikhonov technique where $\gamma$ adapted dynamically as described in Section \ref{sec:gamma}, combined with a re-scaling technique based on a local quadratic model computed using the exact Fisher (see Section \ref{sec:update_rescaling}).  Note that the adaptation of $\gamma$ is important since what constitutes a good or even merely acceptable value of $\gamma$ will change significantly over the course of optimization.  And the use of our re-scaling technique, or something similar to it, is also crucial as we have observed empirically that basic Tikhonov damping is incapable of producing high quality updates by itself, even when $\gamma$ is chosen optimally at each iteration (see Figure \ref{fig:damping_rescaling} of Section \ref{sec:update_rescaling}).   

Also, while Heskes' method computes the $G_{i,i}$'s exactly, \acronym{} uses a stochastic approximation which scales efficiently to neural networks with much higher-dimensional outputs (see Section \ref{sec:estimating_A_and_G}).

Other advances we have introduced include the more accurate block-tridiagonal approximation to the inverse Fisher, a parameter-free type of momentum (see Section \ref{sec:momentum}), online estimation of the $G_{i,i}$ and $\bar{A}_{i,i}$ matrices, and various improvements in computational efficiency (see Section~\ref{sec:efficiency}).  We have found that each of these additional elements is important in order for \acronym{} to work as well as it does in various settings.

Concurrently with this work \citet{povey_ng} has developed a neural network optimization method which uses a block-diagonal Kronecker-factored approximation similar to the one from \citet{heskes}.  This approach differs from \acronym{} in numerous ways, including its use of the empirical Fisher (which doesn't work as well as the standard Fisher in our experience -- see Section \ref{sec:estimating_A_and_G}), and its use of only a basic factored Tikhonov damping technique without adaptive re-scaling or any form of momentum.  One interesting idea introduced by \citet{povey_ng} is a particular method for maintaining an online low-rank plus diagonal approximation of the factor matrices for each block, which allows their inverses to be computed more efficiently (although subject to an approximation).  While our experiments with similar kinds of methods for maintaining such online estimates found that they performed poorly in practice compared to the solution of refreshing the inverses only occasionally (see Section \ref{sec:efficiency}), the particular one developed by \citet{povey_ng} could potentially still work well, and may be especially useful for networks with very wide layers.

%[[\textbf{the remainder of this section is a good candidate for removal in the ICML version}]]

%Our derivation of the Kronecker-factored Fisher approximation is also arguably stronger and better justified than Heskes', although he gives at least one useful interpretation of it which we haven't considered. 

\vspace{-0.1in}
\section{Heskes' interpretation of the block-diagonal approximation}
\vspace{-0.075in}

\citet{heskes} discussed an alternative interpretation of the block-diagonal approximation which yields some useful insight to complement our own theoretical analysis. In particular, he observed that the block-diagonal Fisher approximation $\breve{F}$ is the curvature matrix corresponding to the following quadratic function which measures the difference between the new parameter value $\theta'$ and the current value $\theta$:
\begin{align*}
D(\theta', \theta) = \frac{1}{2} \sum_{i=1}^\ell \expected\left[ (s_i - s'_i)^\top G_{i,i} (s_i - s'_i) \right] %= \frac{1}{2} \sum_{i=1}^\ell \Tr( G_{i,i} \expected\left[ (s_i - W'_i \bar{a}_{i-1}) (s_i - W'_i \bar{a}_{i-1})^\top \right] )
\end{align*}
Here, $s'_i = W'_i \bar{a}_{i-1}$, and the $s_i$'s and $\bar{a}_i$'s are determined by $\theta$ and are independent of $\theta'$ (which determines the $W'_i$'s).  

$D(\theta',\theta)$ can be interpreted as a reweighted sum of squared changes of each of the $s_i$'s.  The reweighing matrix $G_{i,i}$ is given by
\begin{align*}
G_{i,i} = \expected \left[ g_i g_i^\top \right] = \expected\left[F_{P^{(i)}_{y|s_i}}\right]
\end{align*}
where $P^{(i)}_{y|s_i}$ is the network's predictive distribution as parameterized by $s_i$, and $F_{P^{(i)}_{y|s_i}}$ is its Fisher information matrix, and where the expectation is taken w.r.t.~the distribution on $s_i$ (as induced by the distribution on the network's input $x$).  Thus, the effect of reweighing by the $G_{i,i}$'s is to (approximately) translate changes in $s_i$ into changes in the predictive distribution over $y$, although using the expected/average Fisher $G_{i,i} = \expected[F_{P^{(i)}_{y|s_i}}]$ instead of the more specific Fisher $F_{P^{(i)}_{y|s_i}}$.  

Interestingly, if one used $F_{P^{(i)}_{y|s_i}}$ instead of $G_{i,i}$ in the expression for $D(\theta',\theta)$, then $D(\theta',\theta)$ would correspond to a basic layer-wise block-diagonal approximation of $F$ where the blocks are computed exactly (i.e. without the Kronecker-factorizing approximation introduced in Section \ref{sec:kron_approx}).  Such an approximate Fisher would have the interpretation of being the Hessian w.r.t.~$\theta'$ of either of the measures
\begin{align*}
\sum_{i=1}^\ell \expected \left[ \KL \left ( P^{(i)}_{y|s_i} \: \| \: P^{(i)}_{y| s'_i} \right )\right] \quad\quad \mbox{or} \quad\quad \sum_{i=1}^\ell \expected \left[ \KL \left ( P^{(i)}_{y| s'_i}  \: \| \: P^{(i)}_{y|s_i} \right ) \right]
\end{align*}
Note that each term in either of these sums is a function measuring an intrinsic quantity (i.e. changes in the output distribution), and so overall these are intrinsic measures except insofar as they assume that $\theta$ is divided into $\ell$ independent groups that each parameterize one of the $\ell$ different predictive distributions (which are each conditioned on their respective $a_{i-1}$'s).

It is not clear whether $\breve{F}$, with its Kronecker-factorizing structure can similarly be interpreted as the Hessian of such a self-evidently intrinsic measure.  If it could be, then this would considerably simplify the proof of our Theorem \ref{thm:invariance} (e.g. using the techniques of \citet{IGO}).  Note that $D(\theta',\theta)$ itself doesn't work, as it isn't obviously intrinsic.  Despite this, as shown in Section \ref{sec:invariance}, both $\breve{F}$ and our more advanced approximation $\hat{F}$ produce updates which have strong invariance properties.

%\vspace{-0.1in}

\section{Experiments}
\label{sec:kfac_experiments}

%\vspace{-0.075in}

To investigate the practical performance of \acronym{} we applied it to the 3 deep autoencoder optimization problems from \citet{HintonScience}, which use the ``MNIST", ``CURVES", and ``FACES" datasets respectively (see \citet{HintonScience} for a complete description of the network architectures and datasets).  Due to their high difficulty, performance on these problems has become a standard benchmark for neural network optimization methods \citep[e.g.][]{HF,KSD,momentum_ilya}.  We included $\ell_2$ regularization with a coefficient of $\eta = 10^{-5}$ in each of these three optimization problems (i.e. so that $\frac{\eta}{2} \|\theta\|_2^2$ was added to the objective), which is lower than what was used by \citet{HF}, but higher than what was used by \citet{momentum_ilya}.

%Mention where they come from, and papers they have been considered in.  Should we specific the architectures?  Maybe stuff like that should go the appendix, if anywhere

As our baseline we used the version of SGD with momentum based on Nesterov's Accelerated Gradient \citep{NAG} described in \citet{momentum_ilya}, which was calibrated to work well on these particular deep autoencoder problems.  For each problem we followed the prescription given by \citet{momentum_ilya} for determining the learning rate, and the increasing schedule for the decay parameter $\mu$. %However, in lieu of using a two-phase optimization approach, we used an averaging approach
We did not compare to methods based on diagonal approximations of the curvature matrix, as in our experience such methods tend not perform as well on these kinds of optimization problems as the baseline does (an observation which is consistent with the findings of \citet{schraudolph, ADADELTA}).

Our implementation of \acronym{} used most of the efficiency improvements described in Section \ref{sec:efficiency}, except that all ``tasks" were computed serially (and thus with better engineering and more hardware, a faster implementation could likely be obtained). Because the mini-batch size $m$ tended to be comparable to or larger than the typical/average layer size $d$, we did not use the technique described at the end of Section \ref{sec:efficiency} for accelerating the computation of the approximate inverse, as this only improves efficiency in the case where $m < d$, and will otherwise decrease efficiency.

Both \acronym{} and the baseline were implemented using vectorized MATLAB code accelerated with the GPU package Jacket.  The code for \acronym{} is available for download\footnote{\url{http://www.cs.toronto.edu/~jmartens/docs/KFAC3-MATLAB.zip}}.  All tests were performed on a single computer with a 4.4 Ghz 6 core Intel CPU and an NVidia GTX 580 GPU with 3GB of memory.  Each method used the same initial parameter setting, which was generated using the ``sparse initialization" technique from \citet{HF} (which was also used by \citet{momentum_ilya}).

To help mitigate the detrimental effect that the noise in the stochastic gradient has on the convergence of the baseline (and to a lesser extent \acronym{} as well) we used a exponentially decayed iterate averaging approach based loosely on Polyak averaging \citep[e.g.][]{kevin}.  In particular, at each iteration we took the ``averaged" parameter estimate to be the previous such estimate, multiplied by $\xi$, plus the new iterate produced by the optimizer, multiplied by $1-\xi$, for $\xi = 0.99$.  Since the training error associated with the optimizer's current iterate may sometimes be lower than the training error associated with the averaged estimate (which will often be the case when the mini-batch size $m$ is very large), we report the minimum of these two quantities.

To be consistent with the numbers given in previous papers we report the reconstruction error instead of the actual objective function value (although these are almost perfectly correlated in our experience).  And we report the error on the training set as opposed to the test set, as we are chiefly interested in optimization speed and not the generalization capabilities of the networks themselves.

In our first experiment we examined the relationship between the mini-batch size $m$ and the per-iteration rate of progress made by \acronym{} and the baseline on the MNIST problem.  The results from this experiment are plotted in Figure \ref{fig:experiments_minibatch}.  They strongly suggest that the per-iteration rate of progress of \acronym{} tends to a superlinear function of $m$ (which can be most clearly seen by examining the plots of training error vs training cases processed), which is to be contrasted with the baseline, where increasing $m$ has a much smaller effect on the per-iteration rate of progress, and with \acronym{} without momentum, where the per-iteration rate of progress seems to be a linear or slightly sublinear function of $m$.  It thus appears that the main limiting factor in the convergence of \acronym{} (with momentum applied) is the noise in the gradient, at least in later stages of optimization, and that this is not true of the baseline to nearly the same extent.  This would seem to suggest that \acronym{}, much more than SGD, would benefit from a massively parallel distributed implementation which makes use of more computational resources than a single GPU.

\begin{figure}[H]
\begin{centering}

\includegraphics[width=0.46\columnwidth]{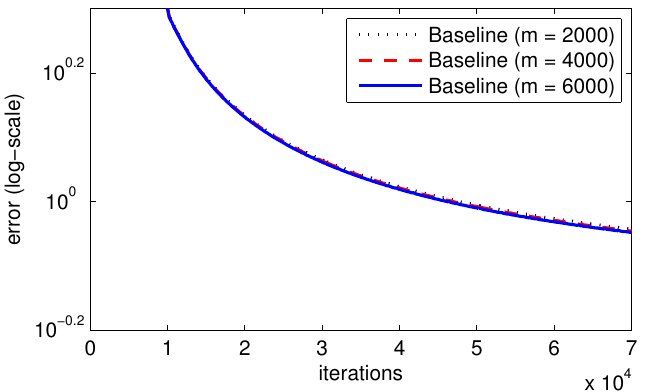} \hspace{0.025\columnwidth}
\includegraphics[width=0.46\columnwidth]{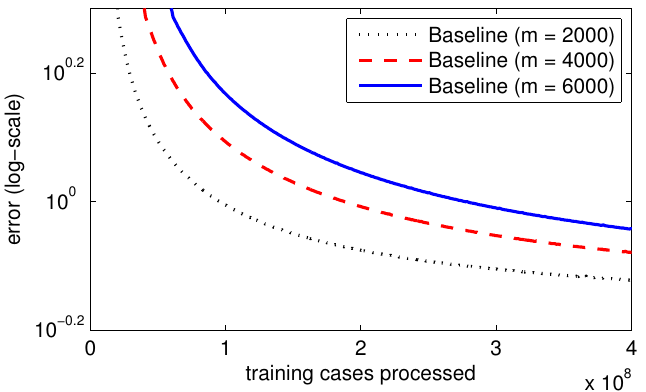}

\vspace{0.005\columnwidth}

\includegraphics[width=0.46\columnwidth]{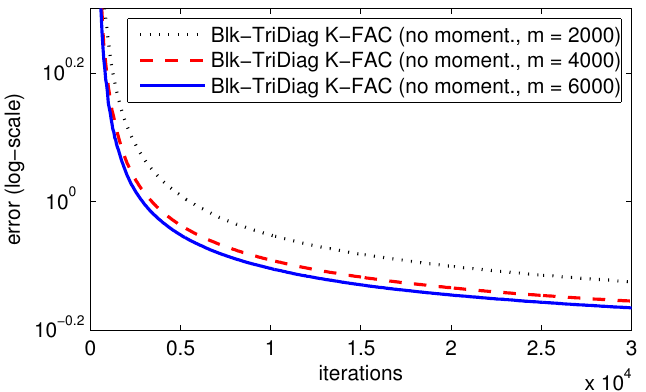} \hspace{0.025\columnwidth}
\includegraphics[width=0.46\columnwidth]{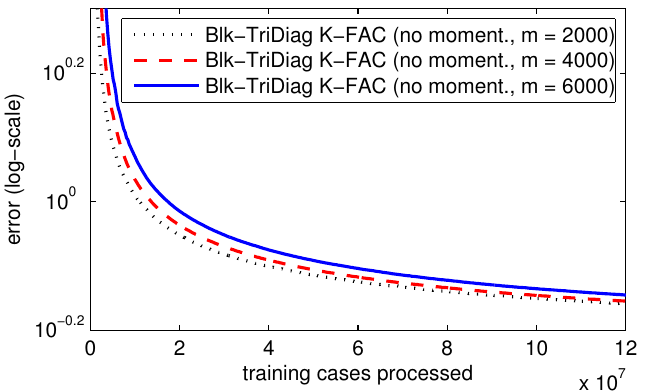}

\vspace{0.005\columnwidth}

\includegraphics[width=0.46\columnwidth]{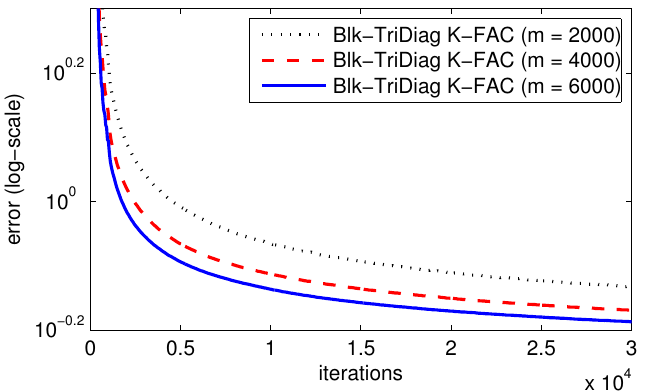} \hspace{0.025\columnwidth}
\includegraphics[width=0.46\columnwidth]{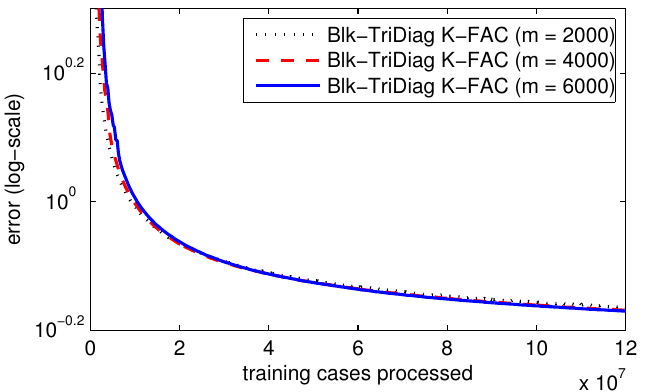}

\vspace{0.005\columnwidth}

\includegraphics[width=0.46\columnwidth]{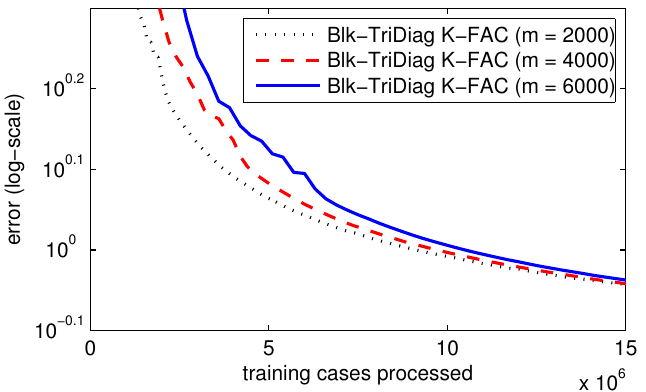} \hspace{0.025\columnwidth}
\includegraphics[width=0.46\columnwidth]{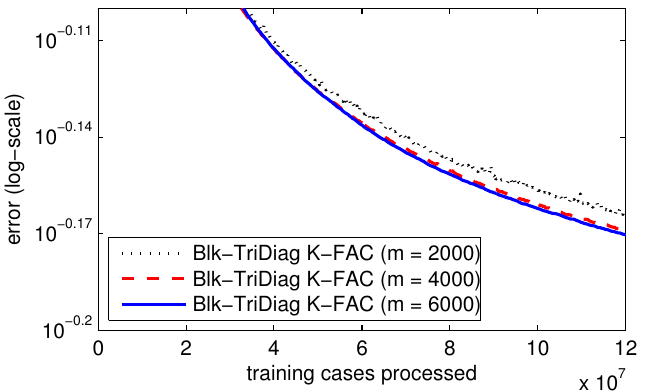}

\vspace{-0.1in}
\caption{\small Results from our first experiment examining the relationship between the mini-batch size $m$ and the per-iteration progress (\textbf{left column}) or per-training case progress (\textbf{right column}) progress made by \acronym{} on the MNIST deep autoencoder problem.  Here, ``Blk-TriDiag \acronym{}" is the block-tridiagonal version of \acronym{}, and ``Blk-Diag \acronym{}" is the block-diagonal version, and ``no moment." indicates that momentum was not used.  The \textbf{bottom row} consists of zoomed-in versions of the right plot from the row above it, with the left plot concentrating on the beginning stage of optimization, and the right plot concentrating on the later stage. Note that the x-axes of these two last plots are at significantly different scales ($10^6$ vs $10^7$). \label{fig:experiments_minibatch} }
\end{centering}
\end{figure}

But even in the single CPU/GPU setting, the fact that the per-iteration rate of progress tends to a \emph{superlinear} function of $m$, while the per-iteration computational cost of \acronym{} is a roughly linear function of $m$, suggests that in order to obtain the best per-second rate of progress with \acronym{}, we should use a rapidly increasing schedule for $m$.   To this end we designed an exponentially increasing schedule for $m$, given by $m_k = \min( m_1 \exp( (k-1)/b ), |S| )$, where $k$ is the current iteration, $m_1 = 1000$, and where $b$ is chosen so that $m_{500} = |S|$.  The approach of increasing the mini-batch size in this way is analyzed by \citet{FriedlanderSchmidt}.  Note that for other neural network optimization problems, such as ones involving larger training datasets than these autoencoder problems, a more slowly increasing schedule, or one that stops increasing well before $m$ reaches $|S|$, may be more appropriate.  One may also consider using an approach similar to that of \citet{byrd2012sample} for adaptively determining a suitable mini-batch size.

In our second experiment we evaluated the performance of our implementation of \acronym{} versus the baseline on all 3 deep autoencoder problems, where we used the above described exponentially increasing schedule for $m$ for \acronym{}, and a fixed setting of $m$ for the baseline and momentum-less \acronym{} (which was chosen from a small range of candidates to give the best overall per-second rate of progress).  The relatively high values of $m$ chosen for the baseline ($m = 250$ for CURVES, and $m = 500$ for MNIST and FACES, compared to the $m = 200$ which was used by \citet{momentum_ilya}) reflect the fact that our implementation of the baseline uses a high-performance GPU and a highly optimized linear algebra package, which allows for many training cases to be efficiently processed in parallel.  Indeed, after a certain point, making $m$ much smaller didn't result in a significant reduction in the baseline's per-iteration computation time.

Note that in order to process the very large mini-batches required for the exponentially increasing schedule without overwhelming the memory of the GPU, we partitioned the mini-batches into smaller ``chunks" and performed all computations involving the mini-batches, or subsets thereof, one chunk at a time.

The results from this second experiment are plotted in Figures \ref{fig:experiments_time} and \ref{fig:experiments_iterations}.  For each problem \acronym{} had a \emph{per-iteration} rate of progress which was orders of magnitude higher than that of the baseline's (Figure \ref{fig:experiments_iterations}), provided that momentum was used, which translated into an overall much higher \emph{per-second} rate of progress (Figure \ref{fig:experiments_time}), despite the higher cost of \acronym{}'s iterations (due mostly to the much larger mini-batch sizes used).  Note that Polyak averaging didn't produce a significant increase in convergence rate of \acronym{} in this second experiment (actually, it hurt a bit) as the increasing schedule for $m$ provided a much more effective (although expensive) solution to the problem of noise in the gradient.

The importance of using some form of momentum on these problems is emphasized in these experiments by the fact that without the momentum technique developed in Section \ref{sec:momentum}, \acronym{} wasn't significantly faster than the baseline (which itself used a strong form of momentum).  These results echo those of \citet{momentum_ilya}, who found that without momentum, SGD was orders of magnitude slower on these particular problems.  Indeed, if we had included results for the baseline without momentum they wouldn't even have appeared in the axes boundaries of the plots in Figure \ref{fig:experiments_time}.

\begin{figure}[H]
\begin{centering}
\includegraphics[width=0.7\columnwidth]{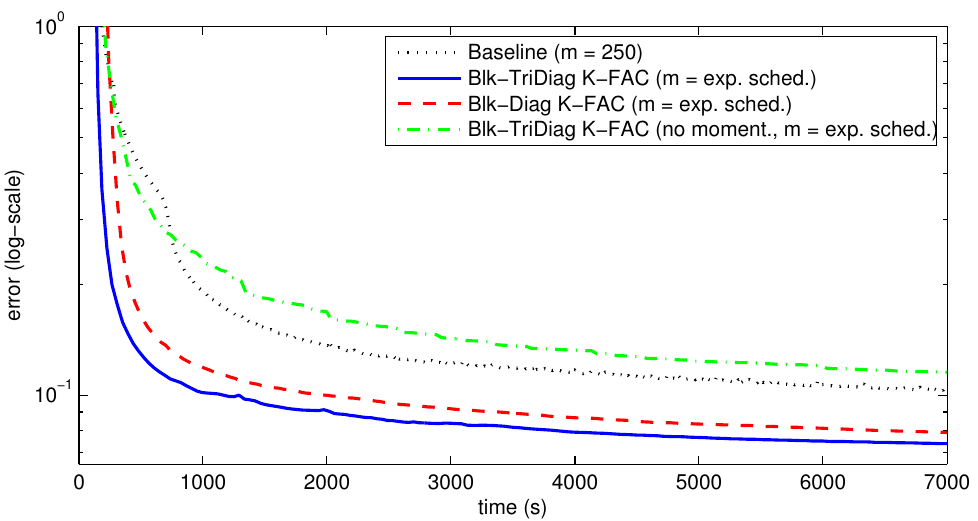}

\vspace{0.025\columnwidth}

\includegraphics[width=0.7\columnwidth]{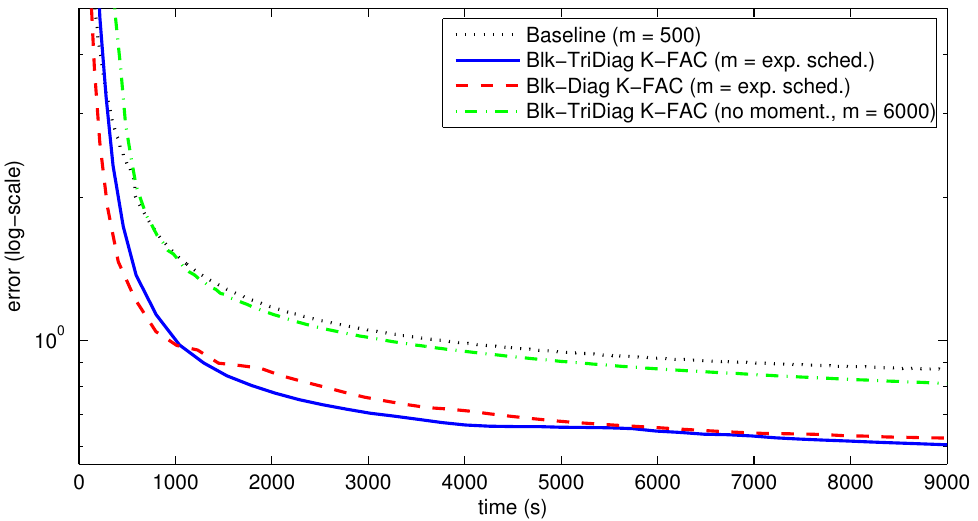}

\vspace{0.025\columnwidth}
\includegraphics[width=0.7\columnwidth]{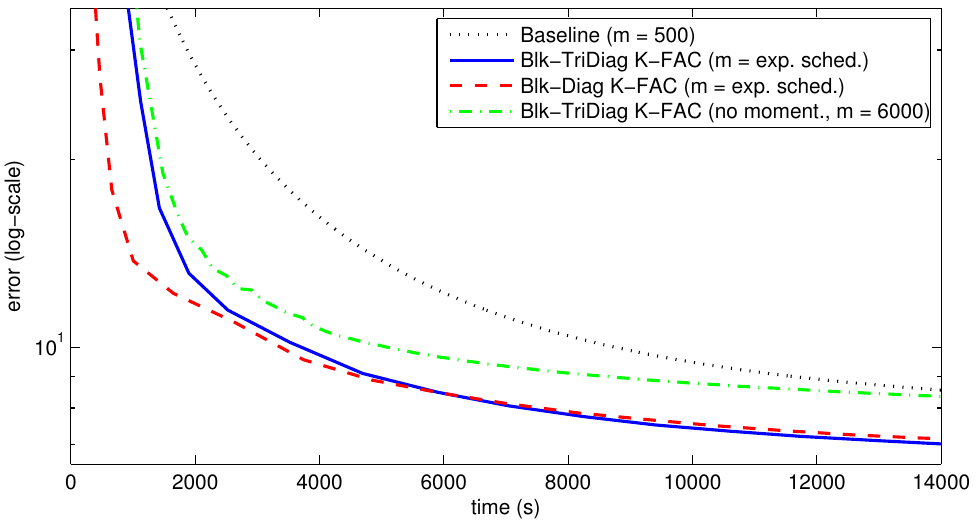}

\caption{\small Results from our second experiment showing training error versus computation time on the CURVES (\textbf{top}), MNIST (\textbf{middle}), and FACES (\textbf{bottom}) deep autoencoder problems.  %Here, ``Blk-TriDiag \acronym{}" is the block-tridiagonal version of \acronym{}, and ``Blk-Diag \acronym{}" is the block-diagonal version.  ``No moment." indicates that momentum was not used. 
\label{fig:experiments_time} }
\end{centering}
\end{figure}

\begin{figure}[H]
\begin{centering}

\includegraphics[width=0.45\columnwidth]{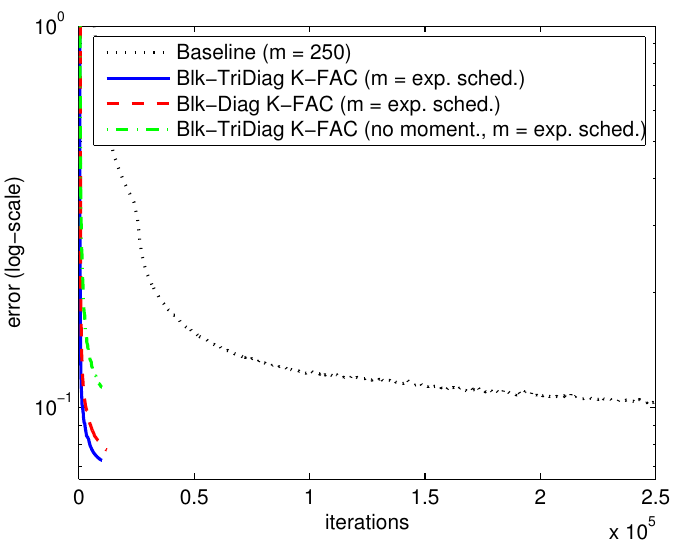}
\hspace{0.025\columnwidth}
\includegraphics[width=0.45\columnwidth]{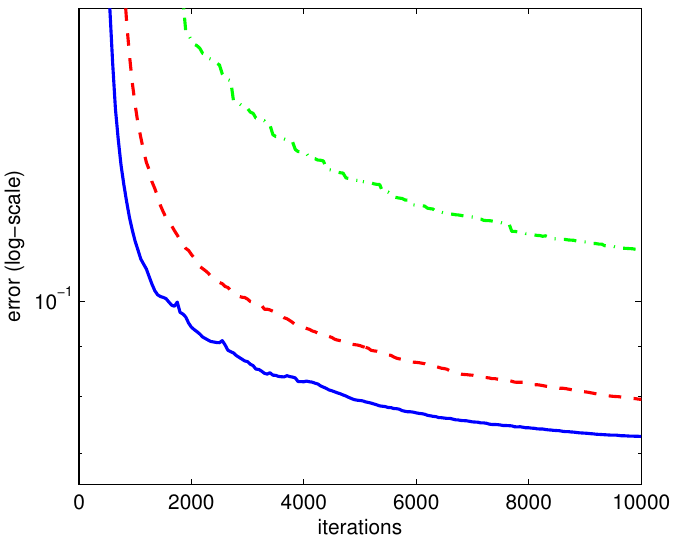}

\vspace{0.025\columnwidth}

\includegraphics[width=0.45\columnwidth]{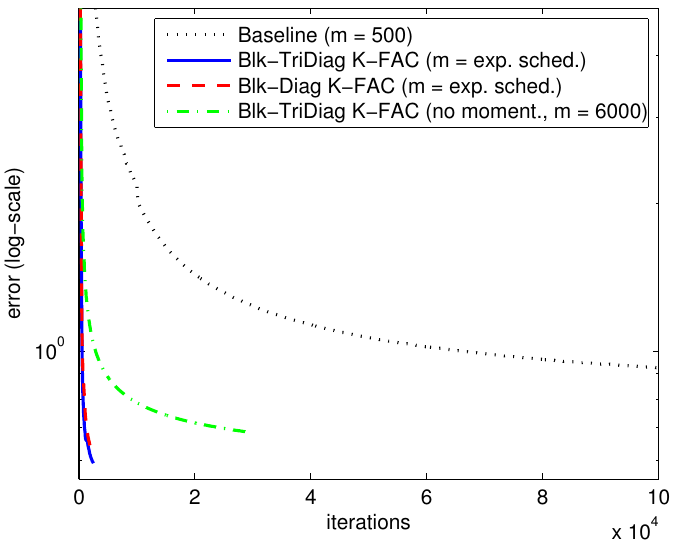}
\hspace{0.025\columnwidth}
\includegraphics[width=0.45\columnwidth]{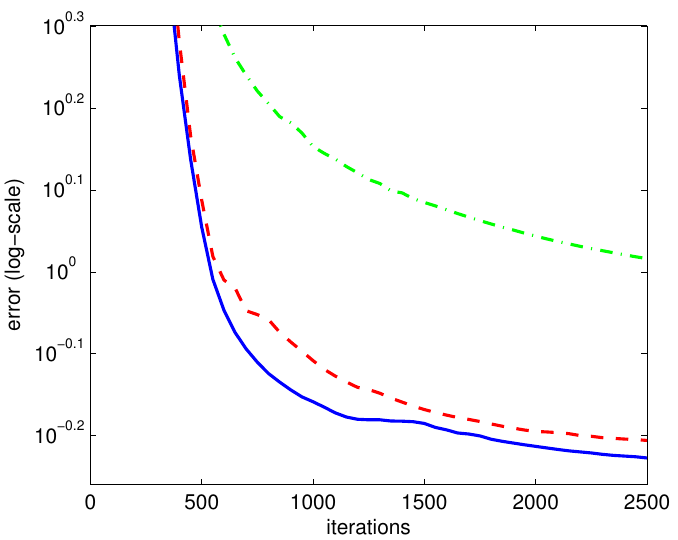}

\vspace{0.025\columnwidth}

\includegraphics[width=0.45\columnwidth]{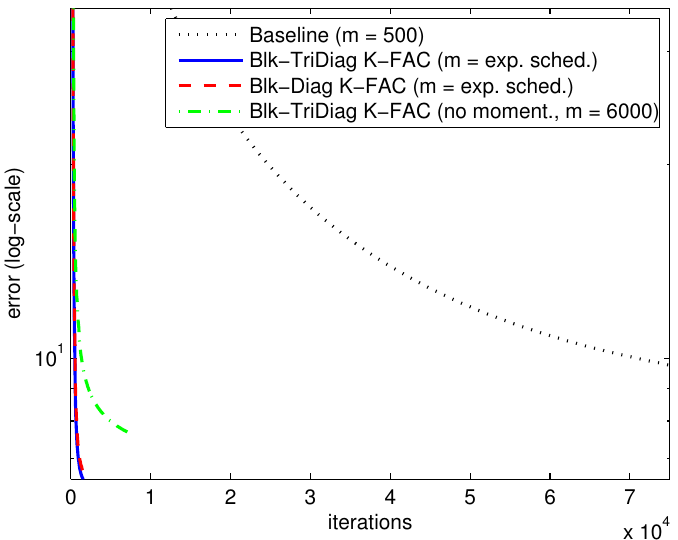}
\hspace{0.025\columnwidth}
\includegraphics[width=0.45\columnwidth]{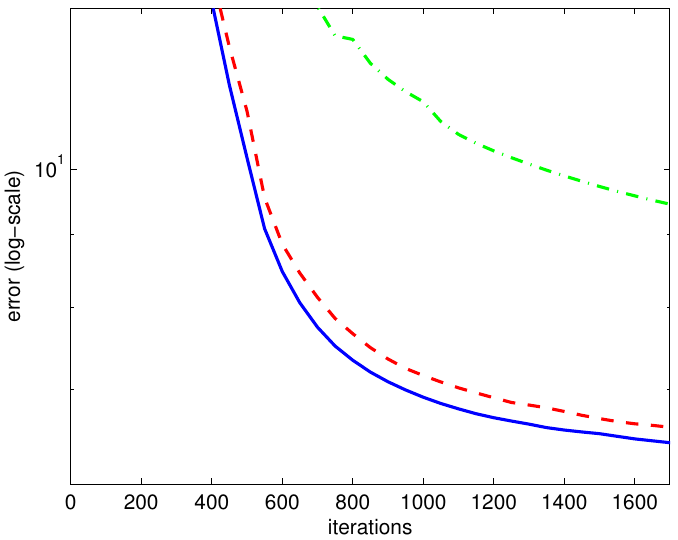}

\caption{\small More results from our second experiment showing training error versus iteration on the CURVES (\textbf{top} row), MNIST (\textbf{middle} row), and FACES (\textbf{bottom} row) deep autoencoder problems.  The plots on the right are zoomed in versions of those on the left which highlight the difference in per-iteration progress made by the different versions of \acronym{}.  \label{fig:experiments_iterations} }
\end{centering}
\end{figure}

Recall that the type of momentum used by \acronym{} compensates for the inexactness of our approximation to the Fisher by allowing \acronym{} to build up a better solution to the exact quadratic model minimization problem (defined using the exact Fisher) across many iterations.  Thus, if we were to use a much stronger approximation to the Fisher when computing our update proposals $\Delta$, the benefit of using this type of momentum would have likely been much smaller than what we observed.  One might hypothesize that it is the particular type of momentum used by \acronym{} that is mostly responsible for its advantages over the SGD baseline.  However in our testing we found that for SGD the more conventional type of momentum used by \citet{momentum_ilya} performs significantly better.

From Figure \ref{fig:experiments_iterations} we can see that the block-tridiagonal version of \acronym{} has a per-iteration rate of progress which is typically 25\% to 40\% larger than the simpler block-diagonal version.  This observation provides empirical support for the idea that the block-tridiagonal approximate inverse Fisher $\hat{F}^{-1}$ is a more accurate approximation of $F^{-1}$ than the block-diagonal approximation $\breve{F}^{-1}$.  However, due to the higher cost of the iterations in the block-tridiagonal version, its overall per-second rate of progress seems to be only moderately higher than the block-diagonal version's, depending on the problem.  

%\vspace{0.3in}

Note that while matrix-matrix multiplication, matrix inverse, and SVD computation all have the same computational complexity, in practice their costs differ significantly (in increasing order as listed).  Computation of the approximate Fisher inverse, which is performed in our experiments once every 20 iterations (and for the first 3 iterations), requires matrix inverses for the block-diagonal version, and SVDs for the block-tridiagonal version. For the FACES problem, where the layers can have as many as 2000 units, this accounted for a significant portion of the difference in the average per-iteration computational cost of the two versions (as these operations must be performed on $2000 \times 2000$ sized matrices).

While our results suggest that the block-diagonal version is probably the better option overall due to its greater simplicity (and comparable per-second progress rate), the situation may be different given a more efficient implementation of \acronym{} where the more expensive SVDs required by the tri-diagonal version are computed approximately and/or in parallel with the other tasks, or perhaps even while the network is being optimized.

Our results also suggest that \acronym{} may be much better suited than the SGD baseline for a massively distributed implementation, since it would require far fewer synchronization steps (by virtue of the fact that it requires far fewer iterations).

%\section{Discussion of results}

\section{Conclusions and future directions}

In this paper we developed \acronym{}, an approximate natural gradient-based optimization method.  We started by developing an efficiently invertible approximation to a neural network's Fisher information matrix, which we justified via a theoretical and empirical examination of the statistics of the gradient of a neural network.  Then, by exploiting the interpretation of the Fisher as an approximation of the Hessian, we designed a developed a complete optimization algorithm using quadratic model-based damping/regularization techniques, which yielded a highly effective and robust method virtually free from the need for hyper-parameter tuning.  We showed the \acronym{} preserves many of natural gradient descent's appealing theoretical properties, such as invariance to certain reparameterizations of the network.  Finally, we showed that \acronym{}, when combined with a form of momentum and an increasing schedule for the mini-batch size $m$, far surpasses the performance of a well-tuned version of SGD with momentum on difficult deep auto-encoder optimization benchmarks (in the setting of a single GPU machine).  Moreover, our results demonstrated that \acronym{} requires orders of magnitude fewer total updates/iterations than SGD with momentum, making it ideally suited for a massively distributed implementation where synchronization is the main bottleneck.

Some potential directions for future development of \acronym{} include:
\begin{itemize}
\item a better/more-principled handling of the issue of gradient stochasticity than a pre-determined increasing schedule for $m$
\item extensions of \acronym{} to recurrent or convolutional architectures, which may require specialized approximations of their associated Fisher matrices
\item an implementation that better exploits opportunities for parallelism described in Section \ref{sec:efficiency}
\item exploitation of massively distributed computation in order to compute high-quality estimates of the gradient%, which would enable massively data-parallel neural network optimization with the need for very few synchonization steps
\end{itemize}

%\subsubsection*{Acknowledgments}
\section*{Acknowledgments}
We gratefully acknowledge support from Google, NSERC, and the University of Toronto.   We would like to thank Ilya Sutskever for his constructive comments on an early draft of this paper.

%\subsubsection*{References}

%[[\textbf{REMEMBER TO ANONYMIZE ANY ``IN PREPARATION" CITATIONS LIKE FANG}]]

%\bibliographystyle{plainnat}
%\bibliographystyle{plainnat-eprints}
\bibliographystyle{plainnat-eprints-authabrv}
 
\bibliography{kron_bibliography}

\newpage
\appendix

\section{Derivation of the expression for the approximation from Section \ref{sec:justifying_approx}}

In this section we will show that
\begin{align*}
\expected \left[ \bar{a}^{(1)} \bar{a}^{(2)} \: g^{(1)} g^{(2)} \right] - &\expected \left[ \bar{a}^{(1)} \bar{a}^{(2)} \right] \expected \left[ g^{(1)} g^{(2)} \right] \\
&= \kappa(\bar{a}^{(1)}, \bar{a}^{(2)}, g^{(1)}, g^{(2)}) + \kappa(\bar{a}^{(1)}) \kappa(\bar{a}^{(2)}, g^{(1)}, g^{(2)}) + \kappa(\bar{a}^{(2)})\kappa(\bar{a}^{(1)}, g^{(1)}, g^{(2)})
\end{align*}

The only specific property of the distribution over $\bar{a}^{(1)}$, $\bar{a}^{(2)}$, $g^{(1)}$, and $g^{(2)}$ which we will require to do this is captured by the following lemma.
\begin{lemma}
\label{lemma:udv_expectation}
Suppose $u$ is a scalar variable which is independent of $y$ when conditioned on the network's output $f(x,\theta)$, and $v$ is some intermediate quantity computed during the evaluation of $f(x,\theta)$ (such as the activities of the units in some layer).  Then we have
\begin{align*}
\expected \left[ u \: \deriv v \right] = 0
\end{align*}
\end{lemma}
Our proof of this lemma (which is at the end of this section) makes use of the fact that the expectations are taken with respect to the network's predictive distribution $P_{y|x}$ as opposed to the training distribution $\hat{Q}_{y|x}$.

Intuitively, this lemma says that the intermediate quantities computed in the forward pass of Algorithm \ref{alg:nnet_gradient} (or various functions of these) are statistically uncorrelated with various derivative quantities computed in the backwards pass, provided that the targets $y$ are sampled according to the network's predictive distribution $P_{y|x}$ (instead of coming from the training set).  Valid choices for $u$ include $\bar{a}^{(k)}$, $\bar{a}^{(k)} - \expected \left[ \bar{a}^{(k)} \right]$ for $k \in \{1,2\}$, and products of these.  Examples of invalid choices for $u$ include expressions involving $g^{(k)}$, since these will depend on the derivative of the loss, which is not independent of $y$ given $f(x,\theta)$.

According to a well-known general formula relating moments to cumulants we may write $\expected \left[ \bar{a}^{(1)} \bar{a}^{(2)} \: g^{(1)} g^{(2)} \right]$ as a sum of 15 terms, each of which is a product of various cumulants corresponding to one of the 15 possible ways to partition the elements of $\{\bar{a}^{(1)}, \bar{a}^{(2)}, g^{(1)}, g^{(2)}\}$ into non-overlapping sets.  For example, the term corresponding to the partition $\{\{\bar{a}^{(1)}\}, \{\bar{a}^{(2)}, g^{(1)}, g^{(2)}\}\}$ is $\kappa( \bar{a}^{(1)} ) \kappa( \bar{a}^{(2)}, g^{(1)}, g^{(2)} )$.

Observing that 1st-order cumulants correspond to means and 2nd-order cumulants correspond to covariances, for $k \in \{1,2\}$ Lemma \ref{lemma:udv_expectation} gives 
\begin{align*}
\kappa(g^{(k)}) = \expected \left[ g^{(k)} \right] = \expected \left[ \deriv x^{(k)} \right] = 0
\end{align*}
where $x^{(1)} = [x_i]_{k_2}$, and $x^{(2)} = [x_j]_{k_4}$ (so that $g^{(k)} = \deriv x^{(k)}$).  And similarly for $k,m \in \{1,2\}$ it gives
\begin{align*}
\kappa(\bar{a}^{(k)},g^{(m)}) = \expected \left[ \left( \bar{a}^{(m)} - \expected\left[ \bar{a}^{(m)} \right]\right) \left (g^{(k)} - \expected \left[ g^{(k)} \right] \right) \right] = \expected \left[ \left( \bar{a}^{(m)} - \expected\left[ \bar{a}^{(m)} \right]\right) g^{(k)} \right] = 0
\end{align*}
Using these identities we can eliminate 10 of the terms.

The remaining expression for $\expected \left[ \bar{a}^{(1)} \bar{a}^{(2)} \: g^{(1)} g^{(2)} \right]$ is thus
\begin{align*}
\kappa(\bar{a}^{(1)}, \bar{a}^{(2)}, g^{(1)}, g^{(2)}) + \kappa(\bar{a}^{(1)}) & \kappa(\bar{a}^{(2)}, g^{(1)}, g^{(2)}) + \kappa(\bar{a}^{(2)})\kappa(\bar{a}^{(1)}, g^{(1)}, g^{(2)}) \\
&+ \kappa( \bar{a}^{(1)}, \bar{a}^{(2)} ) \kappa( g^{(1)}, g^{(2)} ) + \kappa( \bar{a}^{(1)} ) \kappa( \bar{a}^{(2)} ) \kappa( g^{(1)}, g^{(2)} )
\end{align*}

Noting that
\begin{align*}
\kappa( \bar{a}^{(1)}, \bar{a}^{(2)} ) &\kappa( g^{(1)}, g^{(2)} ) + \kappa( \bar{a}^{(1)} ) \kappa( \bar{a}^{(2)} ) \kappa( g^{(1)}, g^{(2)} ) \\
&= \cov( \bar{a}^{(1)}, \bar{a}^{(2)} ) \expected \left[ g^{(1)} g^{(2)} \right] + \expected \left[ \bar{a}^{(1)} \right] \expected \left[ \bar{a}^{(2)} \right] \expected \left[ g^{(1)} g^{(2)} \right] = \expected \left[ \bar{a}^{(1)} \bar{a}^{(2)} \right] \expected \left[ g^{(1)} g^{(2)} \right]
\end{align*}
it thus follows that
\begin{align*}
\expected \left[ \bar{a}^{(1)} \bar{a}^{(2)} \: g^{(1)} g^{(2)} \right] - &\expected \left[ \bar{a}^{(1)} \bar{a}^{(2)} \right] \expected \left[ g^{(1)} g^{(2)} \right] \\
&= \kappa(\bar{a}^{(1)}, \bar{a}^{(2)}, g^{(1)}, g^{(2)}) + \kappa(\bar{a}^{(1)}) \kappa(\bar{a}^{(2)}, g^{(1)}, g^{(2)}) + \kappa(\bar{a}^{(2)})\kappa(\bar{a}^{(1)}, g^{(1)}, g^{(2)})
\end{align*}
as required.

It remains to prove Lemma \ref{lemma:udv_expectation}.

\begin{proof}[Proof of Lemma \ref{lemma:udv_expectation}]

The chain rule gives
\begin{align*}
\deriv v = -\derivfrac{ \log p( y | x, \theta) }{v} = \left. -\derivfrac{ \log r( y | z) }{z} \right|_{z = f(x,\theta)}^\top \derivfrac{f(x,\theta)}{v}
\end{align*}

From which it follows that
\begin{align*}
\expected \left[ u \: \deriv v \right] &= \expected_{\hat{Q}_x} \left[ \expected_{P_{y | x}} \left[ u \: \deriv v \right]  \right] = \expected_{\hat{Q}_x} \left[ \expected_{R_{y | f(x,\theta)}} \left[ u \: \deriv v \right]  \right] \\
&= \expected_{\hat{Q}_x} \left[ \expected_{R_{y | f(x,\theta)}} \left[ -u \left. \derivfrac{ \log r( y | z) }{z} \right|_{z = f(x,\theta)}^\top \derivfrac{f(x,\theta)}{v} \right]  \right] \\
&= \expected_{\hat{Q}_x} \left[ -u \expected_{R_{y | f(x,\theta)}} \left[ \left. \derivfrac{ \log r( y | z) }{z} \right|_{z = f(x,\theta)} \right]^\top \derivfrac{f(x,\theta)}{v} \right] = \expected_{\hat{Q}_x} \left[ -u \: \vec{0}^\top \: \derivfrac{f(x,\theta)}{v} \right] = 0
\end{align*}

That the inner expectation above is $\vec{0}$ follows from the fact that the expected score of a distribution, when taken with respect to that distribution, is $\vec{0}$.

\end{proof}

\section{Efficient techniques for inverting $A \otimes B \pm C \otimes D$}
\label{sec:stein_solve}

It is well known that $(A \otimes B)^{-1} = A^{-1} \otimes B^{-1}$, and that matrix-vector products with this matrix can thus be computed as $(A^{-1} \otimes B^{-1}) v = \ovec( B^{-1} V A^{-\top} )$, where $V$ is the matrix representation of $v$ (so that $v = \ovec(V)$).

Somewhat less well known is that there are also formulas for $(A \otimes B \pm C \otimes D)^{-1}$ which can be efficiently computed and likewise give rise to efficient methods for computing matrix-vector products.

First, note that $(A \otimes B \pm C \otimes D)^{-1} v = u$ is equivalent to $(A \otimes B \pm C \otimes D) u = v$, which is equivalent to the linear matrix equation $B U A^\top \pm D U C^\top = V$, where $u = \ovec(U)$ and $v = \ovec(V)$.  This is known as a generalized Stein equation, and different examples of it have been studied in the control theory literature, where they have numerous applications.  For a recent survey of this topic, see \citet{iter_survey}.

One well-known class of methods called Smith-type iterations \citep{smith} involve rewriting this matrix equation as a fixed point iteration and then carrying out this iteration to convergence.  Interestingly, through the use of a special squaring trick, one can simulate $2^j$ of these iterations with only $\bigO(j)$ matrix-matrix multiplications. 

Another class of methods for solving Stein equations involves the use of matrix decompositions \citep[e.g.][]{stein_chu, stein_gardiner}.  Here we will present such a method particularly well suited for our application, as it produces a formula for $(A \otimes B + C \otimes D)^{-1} v$,  which after a fixed overhead cost (involving the computation of SVDs and matrix square roots), can be repeatedly evaluated for different choices of $v$ using only a few matrix-matrix multiplications.

We will assume that $A$, $B$, $C$, and $D$ are symmetric positive semi-definite, as they always are in our applications. We have
\begin{align*}
A \otimes B \pm C \otimes D = (A^{1/2} \otimes B^{1/2})(I \otimes I \pm A^{-1/2}CA^{-1/2} \otimes B^{-1/2}DB^{-1/2})(A^{1/2} \otimes B^{1/2})
\end{align*}
Inverting both sides of the above equation gives
\begin{align*}
(A \otimes B \pm C \otimes D)^{-1} = (A^{-1/2} \otimes B^{-1/2})(I \otimes I \pm A^{-1/2}CA^{-1/2} \otimes B^{-1/2}DB^{-1/2})^{-1}(A^{-1/2} \otimes B^{-1/2})
\end{align*}

Using the symmetric eigen/SVD-decomposition, we can write $A^{-1/2}CA^{-1/2} = E_1 S_1 E_1^\top$ and $B^{-1/2}DB^{-1/2} = E_2 S_2 E_2^\top$, where for $i \in \{1,2\}$ the $S_i$ are diagonal matrices and the $E_i$ are unitary matrices.

This gives
\begin{align*}
I \otimes I \pm A^{-1/2}CA^{-1/2} \otimes B^{-1/2}DB^{-1/2} &= I \otimes I \pm E_1 S_1 E_1^\top \otimes E_2 S_2 E_2^\top \\
&= E_1 E_1^\top \otimes E_2 E_2^\top \pm E_1 S_1 E_1^\top \otimes E_2 S_2 E_2^\top \\
&= (E_1 \otimes E_2)(I \otimes I \pm S_1  \otimes S_2)(E_1^\top \otimes E_2^\top)
\end{align*}
so that
\begin{align*}
(I \otimes I \pm A^{-1/2}CA^{-1/2} \otimes B^{-1/2}DB^{-1/2})^{-1} = (E_1 \otimes E_2)(I \otimes I \pm S_1 \otimes S_2)^{-1}(E_1^\top \otimes E_2^\top)
\end{align*}
Note that both $I \otimes I$ and $S_1 \otimes S_2$ are diagonal matrices, and thus the middle matrix $(I \otimes I \pm S_1 \otimes S_2)^{-1}$ is just the inverse of a diagonal matrix, and so can be computed efficiently.

Thus we have
\begin{align*}
(A \otimes B \pm C \otimes D)^{-1} &= (A^{-1/2} \otimes B^{-1/2})(E_1 \otimes E_2)(I \otimes I \pm S_1 \otimes S_2)^{-1}(E_1^\top \otimes E_2^\top)(A^{-1/2} \otimes B^{-1/2}) \\
&= (K_1 \otimes K_2)(I \otimes I \pm S_1 \otimes S_2)^{-1}(K_1^\top \otimes K_2^\top)
\end{align*}
where $K_1 = A^{-1/2} E_1$ and $K_2 = B^{-1/2} E_2$.

And so matrix-vector products with $(A \otimes B \pm C \otimes D)^{-1}$ can be computed as
\begin{align*}
(A \otimes B \pm C \otimes D)^{-1} v = \ovec\left (  K_2 \left[ (K_2^\top V K_1) \oslash \left (\mathbf{1}\mathbf{1}^\top \pm s_2 s_1^\top \right) \right] K_1^\top  \right)
\end{align*}
where $E \oslash F$ denotes element-wise division of $E$ by $F$, $s_i = \diag(S_i)$, and $\mathbf{1}$ is the vector of ones (sized as appropriate).  Note that if we wish to compute multiple matrix-vector products with $(A \otimes B \pm C \otimes D)^{-1}$ (as we will in our application), the quantities $K_1$, $K_2$, $s_1$ and $s_2$ only need to be computed the first time, thus reducing the cost of any future such matrix-vector products, and in particular avoiding any additional SVD computations.

In the considerably simpler case where $A$ and $B$ are both scalar multiples of the identity, and $\xi$ is the product of these multiples, we have
\begin{align*}
(\xi I \otimes I \pm C \otimes D)^{-1} = (E_1 \otimes E_2)(\xi I \otimes I \pm S_1 \otimes S_2)^{-1}(E_1^\top \otimes E_2^\top)
\end{align*}
where $E_1 S_1 E_1^\top$ and $E_2 S_2 E_2^\top$ are the symmetric eigen/SVD-decompositions of $C$ and $D$, respectively.  And so matrix-vector products with $(\xi I \otimes I \pm C \otimes D)^{-1}$ can be computed as
\begin{align*}
(\xi I \otimes I \pm C \otimes D)^{-1} v = \ovec\left (  E_2 \left[ (E_2^\top V E_1) \oslash \left (\xi \mathbf{1}\mathbf{1}^\top \pm s_2 s_1^\top \right) \right] E_1^\top  \right)
\end{align*}

\section{Computing $v^\top F v$ and $u^\top F v$ more efficiently}
\label{sec:cheap_uFv}

Note that the Fisher is given by
\begin{align*}
F = \expected_{\hat{Q}_x} \left[ J^\top F_R J \right]
\end{align*}
where $J$ is the Jacobian of $f(x,\theta)$ and $F_R$ is the Fisher information matrix of the network's predictive distribution $R_{y|z}$, evaluated at $z = f(x,\theta)$ (where we treat $z$ as the ``parameters"). 

To compute the matrix-vector product $Fv$ as estimated from a mini-batch we simply compute $J^\top F_R J v$ for each $x$ in the mini-batch, and average the results.  This latter operation can be computed in 3 stages \citep[e.g.][]{ng_martens}, which correspond to multiplication of the vector $v$ first by $J$, then by $F_R$, and then by $J^\top$.  

Multiplication by $J$ can be performed by a forward pass which is like a linearized version of the standard forward pass of Algorithm \ref{alg:nnet_gradient}.  As $F_R$ is usually diagonal or diagonal plus rank-1, matrix-vector multiplications with it are cheap and easy.  Finally, multiplication by $J^\top$ can be performed by a backwards pass which is essentially the same as that of Algorithm \ref{alg:nnet_gradient}. See \citet{schraudolph, ng_martens} for further details.

The naive way of computing $v^\top F v$ is to compute $Fv$ as above, and then compute the inner product of $Fv$ with $v$.   Additionally computing $u^\top F v$ and $u^\top F u$ would require another such matrix-vector product $F u$. 

However, if we instead just compute the matrix-vector products $J v$ (which requires only half the work of computing $Fv$), then computing $v^\top F v$ as $(J v)^\top F_R (J v)$ is essentially free.  And with $J u$ computed, we can similarly obtain $u^\top F v$ as $(J u)^\top F_R (J v)$ and $u^\top F u$ as $(J u)^\top F_R (J u)$. 

This trick thus reduces the computational cost associated with computing these various scalars by roughly half.

\section{Proofs for Section \ref{sec:invariance}}

\begin{proof}[Proof of Theorem \ref{thm:invariance}]

First we will show that the given network transformation can be viewed as reparameterization of the network according to an invertible linear function $\zeta$.  

Define $\theta^\dagger = [ \ovec(W^\dagger_1)^\top \ovec(W^\dagger_2)^\top \dots \ovec(W^\dagger_\ell)^\top ]^\top$, where $W^\dagger_i = \Phi_i^{-1} W_i \Omega_{i-1}^{-1}$ (so that $W_i = \Phi_i W^\dagger_i \Omega_{i-1}$) and let $\zeta$ be the function which maps $\theta^\dagger$ to $\theta$.  Clearly $\zeta$ is an invertible linear transformation.  

If the transformed network uses $\theta^\dagger$ in place of $\theta$ then we have 
\begin{align*}
\bar{a}^\dagger_i = \Omega_i \bar{a}_i \quad \quad \mbox{and} \quad \quad s^\dagger_i = \Phi_i^{-1} s_i
\end{align*}
which we can prove by a simple induction.  First note that $\bar{a}^\dagger_0 = \Omega_0 \bar{a}_0$ by definition.  Then, assuming by induction that $\bar{a}^\dagger_{i-1} = \Omega_{i-1} \bar{a}_{i-1}$, we have
\begin{align*}
s^\dagger_i = W^\dagger_i \bar{a}^\dagger_{i-1} = \Phi_i^{-1} W_i \Omega_{i-1}^{-1} \Omega_{i-1} \bar{a}_{i-1} = \Phi_i^{-1} W_i \bar{a}_{i-1} = \Phi_i^{-1} s_i
\end{align*}
and therefore also
\begin{align*}
\bar{a}^\dagger_i = \Omega_i \bar{\nonlin}_i( \Phi_i s^\dagger_i ) = \Omega_i \bar{\nonlin}_i( \Phi_i \Phi_i^{-1} s_i ) = \Omega_i \bar{\nonlin}_i( s_i ) = \Omega_i \bar{a}_i
\end{align*}

And because $\Omega_\ell = I$, we have $\bar{a}^\dagger_\ell = \bar{a}_\ell$, or more simply that $a^\dagger_\ell = a_\ell$, and thus both the original network and the transformed one have the same output (i.e. $f(x,\theta) = f^\dagger(x,\theta^\dagger)$). % and have the same objective function value (as evaluated at $\theta$ and $\theta^\dagger$ respectively).
From this it follows that $f^\dagger(x,\theta^\dagger) = f(x,\theta) = f(x,\zeta(\theta^\dagger))$, and thus the transformed network can be viewed as a reparameterization of the original network by $\theta^\dagger$.  Similarly we have that $h^\dagger(\theta^\dagger) = h(\theta) = h(\zeta(\theta^\dagger))$.

The following lemma is adapted from \citet{ng_martens} (see the section titled ``A critical analysis of parameterization invariance").

\begin{lemma}
\label{lemma:invar}
Let $\zeta$ be some invertible affine function mapping $\theta^\dagger$ to $\theta$, which reparameterizes the objective $h(\theta)$ as $h(\zeta(\theta^\dagger))$.  Suppose that $\B$ and $\B^\dagger$ are invertible matrices satisfying
\begin{align*}
J_\zeta^\top \B J_\zeta = \B^\dagger
\end{align*}
Then, additively updating $\theta$ by $\delta = -\alpha \B^{-1} \nabla h$ is equivalent to additively updating $\theta^\dagger$ by $\delta^\dagger = -\alpha \B^{\dagger-1} \nabla_{\theta^\dagger} h(\zeta(\theta^\dagger))$, in the sense that $\zeta(\theta^\dagger + \delta^\dagger ) = \theta + \delta$.
\end{lemma}

%\begin{lemma}
%\label{lemma:invar}
%Let $\zeta$ be some invertible affine function mapping $\theta^\dagger$ to $\theta$, and suppose we are given two optimizers, one working in the $\theta$ parameterization and one working in the $\theta^\dagger$ parameterization, which start at equivalent initializations (i.e. $\theta_0 = \zeta(\theta_0^\dagger)$).  Suppose further that they use additive updates of the form $d_\theta = \alpha C^{-1} \nabla h$ and $d_{\theta^\dagger} = \alpha {C^\dagger}^{-1} \nabla h^\dagger$ (respectively), where $C$ and $C^\dagger$ are some curvature matrices (which may depend on the current value of $\theta$ or $\theta^\dagger$), and $\nabla h^\dagger$ is the gradient of $h(\zeta(\theta^\dagger))$ w.r.t. $\theta^\dagger$, and $\alpha$ is some learning rate which is parameterization invariant (but can otherwise depend on the iteration number or the network's current predictive distribution).  Then if the curvatures matrices satisfy the identity
%\begin{align*}
%J_\zeta^\top C J_\zeta = C^\dagger
%\end{align*}
%where $J_\zeta$ is the Jacobian of $\zeta$, we have that the paths taken by the optimizers, when viewed in the same space, are identical.
%\end{lemma}

%Going forward, we will add a ``$\dagger$" superscript to any network-dependent quantity in order to denote the version of that quantity computed using the transformed model.

Because $h^\dagger(\theta^\dagger) = h(\theta) = h(\zeta(\theta^\dagger))$ we have that $\nabla h^\dagger = \nabla_{\theta^\dagger} h(\zeta(\theta^\dagger))$.   So, by the above lemma, to prove the theorem it suffices to show that $J_\zeta^\top \breve{F} \ J_\zeta = \breve{F}^\dagger$ and $J_\zeta^\top \tilde{F} \ J_\zeta = \tilde{F}^\dagger$.%, where $\breve{F}^\dagger$ and $\tilde{F}^\dagger$ are the versions of $\breve{F}$ and $\tilde{F}^\dagger$ computed in the transformed model.

Using $W_i = \Phi_i W^\dagger_i \Omega_{i-1}$ it is straightforward to verify that
\begin{align*}
J_\zeta = \diag( \Omega_0^\top \otimes \Phi_1, \Omega_1^\top \otimes \Phi_2, \: \ldots, \: \Omega_{\ell-1}^\top \otimes \Phi_\ell )
\end{align*}

Because $s_i = \Phi_i s^\dagger_i$ and the fact that the networks compute the same outputs (so the loss derivatives are identical), we have by the chain rule that, $g_i^\dagger = \deriv s^\dagger_i = \Phi_i^\top \deriv s_i = \Phi_i^\top g_i$, and therefore
\begin{align*}
G_{i,j}^\dagger = \expected \left[ g_i^\dagger g_j^{\dagger\top} \right] = \expected \left[ \Phi_i^\top g_i (\Phi_i^\top g_i)^\top \right] = \Phi_i^\top \expected \left[ g_i g_i^\top \right] \Phi_j = \Phi_i^\top G_{i,j} \Phi_j
\end{align*}

Furthermore,
\begin{align*}
\bar{A}_{i,j}^\dagger = \expected \left[ \bar{a}_i^\dagger \bar{a}_j^{\dagger\top} \right] = \expected \left[ (\Omega_i \bar{a}_i) (\Omega_j \bar{a}_j)^\top \right] = \Omega_i \expected \left[ \bar{a}_i \bar{a}_j^\top \right] \Omega_j^\top = \Omega_i \bar{A}_{i,j} \Omega_j^\top
\end{align*}

Using these results we may express the Kronecker-factored blocks of the approximate Fisher $\tilde{F}^\dagger$, as computed using the transformed network, as follows:
\begin{align*}
\tilde{F}^\dagger_{i,j} = \bar{A}^\dagger_{i-1,j-1} \otimes G^\dagger_{i,j} = \Omega_{i-1} \bar{A}_{i-1,j-1} \Omega_{j-1}^\top
 \otimes \Phi_i^\top G_{i,j} \Phi_j &= (\Omega_{i-1} \otimes \Phi_i^\top)(\bar{A}_{i-1,j-1} \otimes G_{i,j})(\Omega_{j-1}^\top \otimes \Phi_j) \\
&= (\Omega_{i-1} \otimes \Phi_i^\top) \tilde{F}_{i,j} (\Omega_{j-1}^\top \otimes \Phi_j)
\end{align*}

Given this identity we thus have
\begin{align*}
\breve{F}^\dagger &= \diag\left( \tilde{F}^\dagger_{1,1}, \tilde{F}^\dagger_{2,2}, \: \ldots, \: \tilde{F}^\dagger_{\ell,\ell} \right) \\
&= \diag\left( (\Omega_0 \otimes \Phi_1^\top) \tilde{F}_{1,1} (\Omega_0^\top \otimes \Phi_1), (\Omega_1 \otimes \Phi_2^\top) \tilde{F}_{2,2} (\Omega_1^\top \otimes \Phi_2), \: \ldots, \: (\Omega_{\ell-1} \otimes \Phi_\ell^\top) \tilde{F}_{\ell,\ell} (\Omega_{\ell-1}^\top \otimes \Phi_\ell) \right) \\
&= \diag( \Omega_0 \otimes \Phi_1^\top, \Omega_1 \otimes \Phi_2^\top, \: \ldots, \: \Omega_{\ell-1} \otimes \Phi_\ell^\top ) \diag\left( \tilde{F}_{1,1}, \tilde{F}_{2,2}, \: \ldots, \:, \tilde{F}_{\ell,\ell} \right) \\
&\hspace{3.5in} \cdot \diag( \Omega_0^\top \otimes \Phi_1, \Omega_1^\top \otimes \Phi_2, \: \ldots, \: \Omega_{\ell-1}^\top \otimes \Phi_\ell ) \\
&= J_\zeta^\top \breve{F} J_\zeta
\end{align*}
%Thus, by Lemma \ref{lemma:invar} we have that the theorem holds for updates computed with $\breve{F}$.

We now turn our attention to the $\hat{F}$ (see Section \ref{sec:blocktridiag_approx} for the relevant notation). 

First note that
\begin{align*}
\Psi^\dagger_{i,i+1} &= \tilde{F}^\dagger_{i,i+1} {\tilde{F}_{i+1,i+1}}^{\dagger-1} = (\Omega_{i-1} \otimes \Phi_i^\top) \tilde{F}_{i,i+1} (\Omega_{i}^\top \otimes \Phi_{i+1}) \left(  (\Omega_i \otimes \Phi_{i+1}^\top) \tilde{F}_{i+1,i+1} (\Omega_i^\top \otimes \Phi_{i+1}) \right)^{-1} \\
&= (\Omega_{i-1} \otimes \Phi_i^\top) \tilde{F}_{i,i+1} (\Omega_{i}^\top \otimes \Phi_{i+1}) (\Omega_i^\top \otimes \Phi_{i+1})^{-1} \tilde{F}_{i+1,i+1}^{-1} (\Omega_i \otimes \Phi_{i+1}^\top)^{-1} \\
&= (\Omega_{i-1} \otimes \Phi_i^\top) \tilde{F}_{i,i+1} \tilde{F}_{i+1,i+1}^{-1} (\Omega_i \otimes \Phi_{i+1}^\top)^{-1} \\
&= (\Omega_{i-1} \otimes \Phi_i^\top) \Psi_{i,i+1} (\Omega_i \otimes \Phi_{i+1}^\top)^{-1}
\end{align*}
and so
\begin{align*}
\Sigma^\dagger_{i|i+1} &= \tilde{F}^\dagger_{i,i} - \Psi^\dagger_{i,i+1} \tilde{F}^\dagger_{i+1,i+1} \Psi_{i,i+1}^{\dagger\top} \\
&= (\Omega_{i-1} \otimes \Phi_i^\top) \tilde{F}_{i,i} (\Omega_{i-1}^\top \otimes \Phi_i) \\
& \quad\quad - (\Omega_{i-1} \otimes \Phi_i^\top) \Psi_{i,i+1} (\Omega_i \otimes \Phi_{i+1}^\top)^{-1}  (\Omega_i \otimes \Phi_{i+1}^\top) \tilde{F}_{i+1,i+1} (\Omega_i^\top \otimes \Phi_{i+1})   (\Omega_i \otimes \Phi_{i+1}^\top)^{-\top} \\
&\hspace{4.75in} \cdot \Psi_{i,i+1}^\top (\Omega_{i-1} \otimes \Phi_i^\top)^\top \\
&= (\Omega_{i-1} \otimes \Phi_i^\top) (\tilde{F}_{i,i} - \Psi_{i,i+1} \tilde{F}_{i+1,i+1} \Psi_{i,i+1}^\top) (\Omega_{i-1}^\top \otimes \Phi_i) \\
& = (\Omega_{i-1} \otimes \Phi_i^\top) \Sigma_{i|i+1} (\Omega_{i-1}^\top \otimes \Phi_i)
\end{align*}
Also, $\Sigma^\dagger_\ell = \tilde{F}^\dagger_{\ell,\ell} = (\Omega_{\ell-1} \otimes \Phi_\ell^\top) \tilde{F}_{\ell,\ell} (\Omega_{\ell-1}^\top \otimes \Phi_\ell) = (\Omega_{\ell-1} \otimes \Phi_\ell^\top) \Sigma_\ell (\Omega_{\ell-1}^\top \otimes \Phi_\ell)$.

From these facts it follows that
\begin{align*}
\Lambda^{\dagger-1} &= \diag\left( \Sigma_{1|2}^\dagger, \Sigma_{2|3}^\dagger, \: \ldots, \: \Sigma_{\ell-1|\ell}^\dagger, \Sigma_\ell^\dagger \right ) \\
&= \diag\left( (\Omega_0 \otimes \Phi_1^\top) \Sigma_{1|2} (\Omega_0 \otimes \Phi_1^\top), (\Omega_1 \otimes \Phi_2^\top) \Sigma_{2|3} (\Omega_1 \otimes \Phi_2^\top), \: \ldots, \right. \\
&\hspace{2in}\left. (\Omega_{\ell-2} \otimes \Phi_{\ell-1}^\top) \Sigma_{\ell-1|\ell} (\Omega_{\ell-2} \otimes \Phi_{\ell-1}^\top), (\Omega_{\ell-1} \otimes \Phi_\ell^\top) \Sigma_\ell (\Omega_{\ell-1}^\top \otimes \Phi_\ell) \right ) \\
&= \diag( \Omega_0 \otimes \Phi_1^\top, \Omega_1 \otimes \Phi_2^\top, \: \ldots, \: \Omega_{\ell-2} \otimes \Phi_{\ell-1}^\top,  \Omega_{\ell-1} \otimes \Phi_\ell^\top )  \diag\left( \Sigma_{1|2}, \Sigma_{2|3}, \: \ldots, \: \Sigma_{\ell-1|\ell}, \Sigma_\ell \right ) \\
&\hspace{2.75in}\diag( \Omega_0^\top \otimes \Phi_1, \Omega_1^\top \otimes \Phi_2, \: \ldots, \: \Omega_{\ell-2}^\top \otimes \Phi_{\ell-1},  \Omega_{\ell-1}^\top \otimes \Phi_\ell ) \\
&= J_\zeta^\top \Lambda^{-1} J_\zeta
\end{align*}
Inverting both sides gives $\Lambda^\dagger = J_\zeta^{-1} \Lambda J_\zeta^{-\top}$.

Next, observe that
\begin{align*}
\Psi^{\dagger\top}_{i,i+1} (\Omega_{i-1}^\top \otimes \Phi_i)^{-1} &= (\Omega_i \otimes \Phi_{i+1}^\top)^{-\top} \Psi_{i,i+1}^\top (\Omega_{i-1} \otimes \Phi_i^\top)^\top (\Omega_{i-1}^\top \otimes \Phi_i)^{-1}
&= (\Omega_i^\top \otimes \Phi_{i+1})^{-1} \Psi_{i,i+1}^\top
\end{align*}
from which it follows that
\begin{align*}
&\Xi^{\dagger\top} J_\zeta^{-1} = \begin{bmatrix}
I  &   &   &    \\
-\Psi^{\dagger\top}_{1,2}  &  I &   &  \\
 &  -\Psi^{\dagger\top}_{2,3} &  I &  \\
 &   &  \ddots  & \ddots  \\
&   &   &    -\Psi^{\dagger\top}_{\ell-1,\ell} &  I 
\end{bmatrix}
\diag( (\Omega_0^\top \otimes \Phi_1)^{-1}, (\Omega_1^\top \otimes \Phi_2)^{-1}, \: \ldots, \: (\Omega_{\ell-1}^\top \otimes \Phi_\ell)^{-1} ) \\
&= \begin{bmatrix}
(\Omega_0^\top \otimes \Phi_1)^{-1}  &   &   &    \\
-\Psi^{\dagger\top}_{1,2} (\Omega_0^\top \otimes \Phi_1)^{-1}  &  (\Omega_1^\top \otimes \Phi_2)^{-1} &   &  \\
 &  -\Psi^{\dagger\top}_{2,3} (\Omega_1^\top \otimes \Phi_2)^{-1} &  (\Omega_2^\top \otimes \Phi_3)^{-1} &  \\
 &   &  \ddots  & \ddots  \\
&   &   &    -\Psi^{\dagger\top}_{\ell-1,\ell}(\Omega_{\ell-2}^\top \otimes \Phi_{\ell-1})^{-1} &  (\Omega_{\ell-1}^\top \otimes \Phi_\ell)^{-1} 
\end{bmatrix} \\
&= \begin{bmatrix}
(\Omega_0^\top \otimes \Phi_1)^{-1}  &   &   &    \\
-(\Omega_0^\top \otimes \Phi_1)^{-1} \Psi^{\top}_{1,2}  &  (\Omega_1^\top \otimes \Phi_2)^{-1} &   &  \\
 &  -(\Omega_1^\top \otimes \Phi_2)^{-1} \Psi^{\top}_{2,3} &  (\Omega_2^\top \otimes \Phi_3)^{-1} &  \\
 &   &  \ddots  & \ddots  \\
&   &   &    -(\Omega_{\ell-2}^\top \otimes \Phi_{\ell-1})^{-1} \Psi^{\top}_{\ell-1,\ell} &  (\Omega_{\ell-1}^\top \otimes \Phi_\ell)^{-1} 
\end{bmatrix} \\
& = \diag( (\Omega_0^\top \otimes \Phi_1)^{-1}, (\Omega_1^\top \otimes \Phi_2)^{-1}, \: \ldots, \: (\Omega_{\ell-1}^\top \otimes \Phi_\ell)^{-1} ) 
\begin{bmatrix}
I  &   &   &    \\
-\Psi^{\top}_{1,2}  &  I &   &  \\
 &  -\Psi^{\top}_{2,3} &  I &  \\
 &   &  \ddots  & \ddots  \\
&   &   &    -\Psi^{\top}_{\ell-1,\ell} &  I 
\end{bmatrix} \\
&= J_\zeta^{-1} \Xi^{\top}
\end{align*}

Combining $\Lambda^\dagger = J_\zeta^{-1} \Lambda J_\zeta^{-\top}$ and $\Xi^{\dagger\top} J_\zeta^{-1} = J_\zeta^{-1} \Xi^{\top}$ we have
\begin{align*}
\hat{F}^{\dagger-1} = \Xi^{\dagger\top} \Lambda^\dagger \Xi^\dagger = \Xi^{\dagger\top} J_\zeta^{-1} \Lambda J_\zeta^{-\top} \Xi^\dagger = (\Xi^{\dagger\top} J_\zeta^{-1}) \Lambda (\Xi^{\dagger\top} J_\zeta^{-1})^\top &= (J_\zeta^{-1} \Xi^{\top}) \Lambda (J_\zeta^{-1} \Xi^{\top})^\top \\
&= J_\zeta^{-1} \Xi^{\top} \Lambda \Xi J_\zeta^{-\top}\\
&= J_\zeta^{-1} \hat{F}^{-1} J_\zeta^{-\top}
\end{align*}

Inverting both sides gives $\hat{F}^\dagger = J_\zeta^\top \hat{F} J_\zeta$ as required. %, and thus by Lemma \ref{lemma:invar} we have that the theorem holds for updates computed with $\hat{F}$.

\end{proof}

\begin{proof}[Proof of Corollary \ref{cor:whitened_interpretation}]

First note that a network which is transformed so that $G^\dagger_{i,i} = I$ and $\bar{A}^\dagger_{i,i} = I$ will satisfy the required properties.  To see this, note that $\expected[g^\dagger_i  g^{\dagger\top}_i] = G^\dagger_{i,i} = I$ means that $g^\dagger_i$ is whitened with respect to the model's distribution by definition (since the expectation is taken with respect to the model's distribution), and furthermore we have that $\expected[g^\dagger_i] = 0$ by default (e.g. using Lemma \ref{lemma:udv_expectation}), so $g^\dagger_i$ is centered.  And since $\expected[a^\dagger_i  a^{\dagger\top}_i]$ is the square submatrix of $\bar{A}^\dagger_{i,i} = I$ which leaves out the last row and column, we also have that $\expected[a^\dagger_i  a^{\dagger\top}_i] = I$ and so $a^\dagger_i$ is whitened.  Finally, observe that $\expected[a^\dagger_i]$ is given by the final column (or row) of $\bar{A}_{i,i}$, excluding the last entry, and is thus equal to $0$, and so $a^\dagger_i$ is centered.

Next, we note that if $G^\dagger_{i,i} = I$ and $\bar{A}^\dagger_{i,i} = I$ then
\begin{align*}
\breve{F}^\dagger = \diag\left( \bar{A}^\dagger_{0,0} \otimes G^\dagger_{1,1}, \bar{A}^\dagger_{1,1} \otimes G^\dagger_{2,2}, \: \ldots, \: \bar{A}^\dagger_{\ell-1,\ell-1} \otimes G^\dagger_{\ell,\ell} \right) = \diag\left( I \otimes I, I \otimes I, \: \ldots, \: I \otimes I \right) = I
\end{align*}
and so $-\alpha \breve{F}^{-1} \nabla h^\dagger = -\alpha \nabla h^\dagger$ is indeed a standard gradient descent update.

Finally, we observe that there are choices of $\Omega_i$ and $\Phi_i$ which will make the transformed model satisfy $G^\dagger_{i,i} = I$ and $\bar{A}^\dagger_{i,i} = I$.  In particular, from the proof of Theorem \ref{thm:invariance} we have that $G_{i,j}^\dagger = \Phi_i^\top G_{i,j} \Phi_j$ and $\bar{A}_{i,j}^\dagger = \Omega_i \bar{A}_{i,j} \Omega_j^\top$, and so taking $\Phi_i = G_{i,i}^{-1/2}$ and $\Omega_i = \bar{A}_{i,i}^{-1/2}$ works.

The result now follows from Theorem \ref{thm:invariance}.

\end{proof}

\end{document}